\documentclass[twoside,11pt]{article}

\usepackage{jmlr2e}

\usepackage{bbm,bm,floatrow}
\usepackage[long]{optidef}
\usepackage[mathscr]{eucal}
\usepackage{subfig,cleveref}

\newcommand{\NN}{\mathbb{N}}
\newcommand{\Psc}{\mathscr{P}}
\newcommand{\Csc}{\mathscr{C}}
\newcommand{\iv}[2]{[\![#1,#2]\!]} 
\renewcommand{\d}{\mathrm{d}}    
\newcommand{\Kcons}{\mathscr{K}} 
\newcommand{\Dsdp}{\mathbf{D}}   
\newcommand{\Bsc}{\mathscr{B}}   
\newcommand{\RR}{\mathbb{R}}     
\newcommand{\Lcal}{\mathcal{L}}  
\newcommand{\LcalS}{\mathcal{L}_S}  
\newcommand{\Ltilde}{\tilde{L}}  
\newcommand{\BB}{\mathbb{B}}     
\newcommand{\dBB}{\mathring{\mathbb{B}}} 
\newcommand{\Ktens}{\bm{\mathcal{K}}} 
\newcommand{\dH}{\mathring{H}}    
\newcommand{\Id}{\text{Id}}

\DeclareMathOperator*{\diag}{diag} 
\DeclareMathOperator*{\diam}{diam} 
\DeclareMathOperator*{\Sp}{span}   

\DeclareMathOperator*{\dom}{dom}
\DeclareMathOperator*{\cob}{\overline{co}}

\DeclareMathOperator*{\tr}{Tr}

\usepackage{enumitem}
\usepackage{algorithm}
\usepackage{bm}
\usepackage{algorithmic}
\usepackage{booktabs}
\newcommand{\Y}{\mathscr{Y}}      
\newcommand{\X}{\mathscr{X}}      
\newcommand{\C}{\mathcal{C}}      
\newcommand{\K}{K}                
\newcommand{\F}{\mathscr{F}}      
\newcommand{\FK}{\F_{\K}}         
\newcommand{\Fk}{\F_{k}}          
\newcommand{\hatFK}{\hat{\F}_K}   
\renewcommand{\b}{\mathbf}        
\newcommand{\tb}{\textbf}         
\newcommand{\p}{\partial}         
\newcommand{\R}{\mathbb{R}}       
\newcommand{\Rnn}{\R_{+}}         
\newcommand{\N}{\mathbb{N}}       
\newcommand{\Zp}{\N^*}            
\renewcommand{\P}{\mathbb{P}}     
\newcommand{\T}{\top}             
\newcommand{\app}{\text{app}}     
\newcommand{\rel}{\text{relax}}   
\newcommand{\proj}{\text{proj}}   
\renewcommand{\S}{\mathscr{S}}    
\newcommand{\Tiv}{\mathcal{T}}    
\newcommand{\Cp}{\mathscr{C}_{\text{no-pers}}^{(k)}} 
\newcommand{\D}{\mathscr{D}}     
\newcommand{\NS}{N_S}            

\DeclareMathOperator*{\argmin}{arg\,min}

\firstpageno{1}

\usepackage{lastpage}
\jmlrheading{23}{2022}{1-\pageref{LastPage}}{1/21; Revised
7/22}{10/22}{21-0007}{Pierre-Cyril Aubin-Frankowski and Zolt{\'a}n Szab{\'o}}
\ShortHeadings{Handling Hard Affine SDP Shape Constraints in RKHSs}{Aubin-Frankowski and Szab{\'o}}

\begin{document}
	
	\title{Handling Hard Affine SDP Shape Constraints in RKHSs}
	\author{\name Pierre-Cyril Aubin-Frankowski \email  pierre-cyril.aubin@inria.fr \\
		\addr INRIA - Département d’Informatique de l’École Normale Supérieure,\\
		PSL Research University, 2 rue Simone Iff, 75012, Paris, France
		\AND
		\name Zolt{\'a}n Szab{\'o} \email z.szabo@lse.ac.uk \\
		\addr Department of Statistics,	London School of Economics\\
		Houghton Street, London, WC2A 2AE, UK}
	
    \editor{Massimiliano Pontil}
	
	\maketitle
	
	\begin{abstract}
		Shape constraints, such as non-negativity, monotonicity, convexity  or supermodularity, play a key role in various applications of machine learning and statistics. However, incorporating this side information into predictive models in a hard way (for example at \emph{all} points of an interval) for rich function classes is a notoriously challenging problem.
		We propose a unified and modular convex optimization framework, relying on second-order cone (SOC) tightening, to encode hard affine SDP constraints on function derivatives, for models belonging to vector-valued reproducing kernel Hilbert spaces (vRKHSs). The modular nature of the proposed approach allows to simultaneously handle multiple shape constraints, and to tighten an infinite number of constraints into finitely many. We prove the convergence of the proposed scheme and that of its adaptive variant, leveraging geometric properties of vRKHSs. Due to the covering-based construction of the tightening, the method is particularly well-suited to tasks with small to moderate input dimensions. The efficiency of the approach is illustrated in the context of shape optimization, safety-critical control, robotics and econometrics.
	\end{abstract}
	\begin{keywords}
		vector-valued reproducing kernel Hilbert space, shape-constrained optimization, matrix-valued kernel, kernel derivatives
	\end{keywords}
	
	\section{Introduction} 
	The design of flexible predictive models is among the most fundamental problems of machine learning. However, in several applications one is faced with a limited number of samples due to the difficulty or the cost of data acquisition.
	A well-established way to tackle this serious bottleneck and to improve sample-efficiency corresponds to incorporating qualitative priors on the shape of the model, such as non-negativity, monotonicity, convexity or supermodularity, collectively known as shape constraints \citep{guntuboyina18nonparametric}. This side information can originate from both physical and theoretical constraints on the model such as ``stay within boundaries'' in path-planning or ``be nonnegative and integrate to one'' in density estimation.
	
	Various scientific fields, including econometrics, statistics, biology, game theory or finance, impose shape constraints on their hypothesis classes. For instance, economic theory dictates increasing and concave utility functions, decreasing demand functions, or monotone link functions \citep{johnson18shape,chetverikov18econometrics}.
	In statistics, applying a monotonicity assumption on the regression function (for instance in isotonic regression;   \citealt{han19isotonic}) dates back at least to \citet{brunk55maximum}; the non-negativity requirement also arises naturally when learning the intensity function of Poisson processes and the triggering function of Hawkes processes \citep{yang19learning}. Density estimation entails non-negativity which can be paired with other constraints \citep{royset15fusion}, whereas, in quantile regression, conditional quantile functions grow w.r.t.\ the quantile level \citep{koenker05quantile}.
	In biology, monotone regression is particularly  well-suited to dose-response studies \citep{hu05analysis} and to identification of genome interactions \citep{luss12efficient}. Inventory problems, game theory and  pricing models commonly rely on the assumption of supermodularity \citep{topkis98supermodularity,simchilevi14logic}. In financial applications, call option prices should be increasing in volatility, monotone and convex in the underlying stock price \citep{aitsahalia03nonparametric}. In control theory, shape constraints are known as state constraints, and rank among the most difficult topics of the field \citep{hartl95survey,aubin2020hard_control}.
	
	A large and important class of these shape requirements takes the form of an affine SDP (positive semidefinite) inequality over the derivatives of $\b f\in\F$  where $\F$ is a hypothesis class (a set of candidate predictive models). Particularly, these constraints are requested to hold \emph{pointwise} at \emph{all elements} of a set $\Kcons\subseteq \R^d$:
	\begin{align}
		\b 0_{P\times P} & \preccurlyeq \diag(\b b)+\b D\b f(\b x)   \quad \forall \,\b x \in \Kcons \label{eq:pointwise-constraints}
	\end{align}
	for some bias $\b b\in\R^P$ and differential operator $\b D$ (e.g., the Hessian). The fundamental challenge one faces when optimizing an objective $\Lcal(\b f)$ over $\F$ is that in most relevant cases the set $\Kcons$ has non-finite cardinality, and hence there is an infinite number of constraints to satisfy. For instance, in constrained path-planning, $\Kcons$ corresponds to a time interval and the goal is to avoid collisions at all times. 
	
	In the statistics community, the main emphasis has been on designing consistent estimators and on studying their rates  \citep{han16multivariate,chen16generalized,freyberger18inference,lim20limiting,deng20isotonic,kur20optimality}. While these asymptotic results are of significant theoretical interest, imposing shape priors is generally beneficial in the small-sample regime. Since optimization  with an infinite number of constraints \eqref{eq:pointwise-constraints} is \emph{computationally intractable}, one has to either relax or tighten the problem.\footnote{\label{footnote:relax-tighten}We say that problem $(\Psc_1)$ is a relaxation (resp.\ tightening) of problem $(\Psc_2)$ if they have the same objective function and the search space of $(\Psc_1)$ contains (resp.\ is contained in) that of $(\Psc_2)$.} \emph{Relaxing} corresponds to approaches for which the constraint \eqref{eq:pointwise-constraints} is not guaranteed to be satisfied. For instance, one can choose to enforce the constraint only at a finite number of points \citep{takeuchi06nonparametric,blundell12measuring,agrell19gaussian} by replacing $\Kcons$ with a discretization $\{\b x_m\}_{m=1}^{M} \subsetneq \Kcons$ in \eqref{eq:pointwise-constraints}. An alternative approach for relaxing is to add soft penalties to the objective $\Lcal(\b f)$ \citep{sangnier16joint,koppel19projected,brault19infinite}. \emph{Tightening} on the contrary restricts the search space of functions to a smaller and more amenable subset $\F_0\subsetneq \F$. This principle can be implemented by encoding the requirement \eqref{eq:pointwise-constraints} into $\F_0$ through algebraic techniques. The approach is feasible for restrictive finite-dimensional $\F_0$ such as subsets of polynomials \citep{hall18thesis,curmei21shape} or polynomial splines \citep{turlach05shape,papp14shape,pya15shape,wu18semiparametric,meyer18framework}.  These limitations motivate the design of novel shape-constrained optimization techniques which avoid (i) restricted function classes, (ii) limited out-of-sample guarantees and (iii) the lack of modularity in terms of the shape constraints imposed. 	

	In this work the class of functions $\F$ is assumed to be a reproducing kernel Hilbert space $\F:=\FK$ (RKHS; \citealt{steinwart08support,saitoh16theory}; also referred to as abstract splines; \citealt{wahba90spline,berlinet04reproducing,wang11splines}). There are multiple advantages in selecting this family of functions. First, kernel methods rely inherently on pointwise evaluation \citep{aronszajn50theory} which are well-suited to handle the pointwise constraints \eqref{eq:pointwise-constraints}. In particular, the associated reproducing property (which also holds for derivatives; \citealt{zhou08derivative}) allows one to rephrase the inequality constraints \eqref{eq:pointwise-constraints} in $\FK$ using a geometric perspective, as it will be elaborated in Section~\ref{sec:constr1}. Moreover, RKHSs can be rich enough to approximate various function classes (including the space of continuous bounded functions, a property known as universality; \citealt{steinwart01influence,micchelli06universal,sriperumbudur11universality,simon-gabriel18kernel}).
	In addition, the models $\b f\in \FK$ obtained through kernel regression share the regularity of the underlying kernel \citep{steinwart08support}, allowing one to incorporate additional prior knowledge through the choice of $K$. Furthermore, vector-valued RKHSs (vRKHS; \citealt{micchelli05learning,brouard11semisupervised,kadri16operator,bouche21nonlinear,huusari21entangled}) induced by operator-valued kernels can efficiently encode dependency between output coordinates \citep[see][for an exhaustive review]{alvarez12kernels}. These vRKHSs have similar spectral decomposition \citep{devito13extension} and universal approximation properties \citep{carmeli10vector} as their real-valued counterpart. Finally, despite the infinite-dimensional nature of most vRKHSs of interest, kernel methods often  remain computationally tractable thanks to representer theorems \citep{scholkopf01generalized,zhou08derivative}. However, classical representer theorems only hold for a finite number of evaluations both in the objective and in the constraints. This is one of the points we address through our approach based on finite compact coverings.
	
	Other kernel approaches to deal with pointwise constraints were recently investigated based on kernel sum-of-squares turning the inequalities \eqref{eq:pointwise-constraints} into equalities to nonparametric nonnegative functions, requiring SDP optimization. These approaches either tighten a single non-negativity constraint over the whole space ($\Kcons=\R^d$; \citealt{marteauferey20nonparametric}) or relax a convexity constraint on a compact set \citep[Section 5]{muzellec22learning}. Our flexible framework can be seen as complementary to the latter since we propose a feasible tightening of constraints on a compact set, and thus achieve a certificate of optimality. We also have significantly fewer assumptions on the constraint set and on the kernel. Indeed, in \citet{muzellec22learning} the compact set considered is a finite union of Euclidean balls with the same radius and the kernel is essentially of Sobolev type.
	
	With a vRKHS choice for $\F$, our \tb{contributions}\footnote{Our main theoretical results are also gathered in Table~\ref{tab:main-theoretical-results} for the readers' convenience.} can be summarized as follows.
	\begin{enumerate}[labelindent=0em,leftmargin=1em,topsep=0cm,partopsep=0cm,parsep=0cm,itemsep=0mm]
		\item We propose two principled ways to tighten the infinite number of SDP constraints \eqref{eq:pointwise-constraints} through compact coverings in vRKHSs and through an upper bound of the modulus of continuity of $\b D \b f$. Specifically, we show that \eqref{eq:pointwise-constraints} can be tightened into a finite number of SDP inequalities with second-order cone (SOC) terms
		\begin{align}
			\eta_m \|\b f\|_K \, \b I_P & \preccurlyeq \diag(\b b)+ \b D\b f(\tilde{\b x}_m),   \quad \forall \, m \in [M]:=\{1,\ldots,M\} \label{eq:pointwise-constraints_SOC}
		\end{align}
		for a suitable choice of $\eta_m>0$ and $\tilde{\b x}_m\in \Kcons$ ($m\in [M]$).
		\item When considering supervised learning over vRKHSs, we prove an existence result and a representer theorem for the strengthened problems; this approach allows handling several shape constraints in a modular way. In addition, we establish the convergence, when refining the covering, to the solution of the original problem with constraint \eqref{eq:pointwise-constraints}.
		\item We design adaptive variants of the previous schemes, in order to enforce the constraints only where it is necessary, and show the convergence of these variants.
		\item We illustrate the efficiency of our approach in the context of shape optimization, safety-critical control, robotics and	econometrics.
	\end{enumerate}
	\begin{table}
        \begin{minipage}{\textwidth}
        \centering
        \begin{tabular}{rll}
		  \toprule
               Result & Content & Page\\ \midrule
               Lemma~\ref{lemma:reproducing} & reproducing property for derivatives with matrix-valued kernels & page~\pageref{lemma:reproducing} \\
Theorem~\ref{thm:inclusion} & tightening based on set inclusion (balls and half-spaces, $P=1$) & page~\pageref{thm:inclusion}\\
Theorem~\ref{thm:SDP} & tightening based on modulus of continuity (balls, $P\geq 1$) & page~\pageref{thm:SDP} \\
Lemma~\ref{lemma:eta_SDP_4Dtensor} & finite-dimensional description of $\eta_{m,P}$ in Theorem~\ref{thm:SDP} & page~\pageref{lemma:eta_SDP_4Dtensor} \\
Theorem~\ref{thm:certificate} & tightenings: existence of solution, certificate of optimality  & page~\pageref{thm:certificate}\\
Corollary~\ref{thm:aposteriori_bound} & tightenings: a posteriori bound & page~\pageref{thm:aposteriori_bound}\\
Proposition~\ref{thm:apriori_bound} & tightenings: a priori bound, convergence  & page~\pageref{thm:apriori_bound}\\
Proposition~\ref{prop:reproducing} & tightenings: representer theorem & page~\pageref{prop:reproducing}\\
Theorem~\ref{thm:soap_bubble} & adaptive tightening (soap bubble algorithm): convergence & page~\pageref{thm:soap_bubble}\\ \bottomrule
       \end{tabular}
               \end{minipage}
        \caption{Main theoretical results.}
        \label{tab:main-theoretical-results}
    \end{table} 
	In this paper, we thus propose a unified and modular convex optimization framework for kernel machines relying on SOC tightening to encode hard affine SDP constraints on function derivatives. Our framework is suited for a large number of settings and applications owing to the ubiquity of shape constraints. To our best knowledge, this is the first approach with similar properties. Due to the covering-based construction, the method is particularly well-suited to the setting of small to moderate input dimensions, but can face curse-of-dimensionality issues in larger dimensions. 
	
	This article extends the results of \citet{aubin20hard} by (i) considering matrix-valued rather than real-valued kernels, (ii) generalizing the shape requirements studied from real-valued to affine SDP constraints, (iii) proposing an adaptive covering scheme and showing its convergence, and (iv) providing applications complementary to the previous focus on joint quantile regression. The present article also encompasses two prior domain-specific applications with $\X=[0,T]$: convoy trajectory reconstruction \citep{aubin20kernel} and linear quadratic optimal control \citep{aubin2020hard_control}.
	
	\noindent\tb{Structure of the paper.} Our problem is introduced in Section~\ref{sec:problem-formulation}. Section~\ref{sec:constraints} discusses the handling of hard affine SDP shape constraints. The constraints are then embedded into an optimization problem in Section~\ref{sec:optimization}. In Section~\ref{sec:covering_algorithms} we present the soap bubble algorithm which is an adaptive scheme  combining the results of Section~\ref{sec:constraints} and Section~\ref{sec:optimization}. Numerical illustrations are given in Section~\ref{sec:numerical-demos}. Conclusions  are drawn in Section~\ref{sec:conclusions}. Proofs are collected in Section~\ref{sec:proofs} in the Appendix.\\
	
	\noindent\tb{Notation:}  We introduce below the notation $\N$, $\Zp$, $\Rnn$, $\iv{n_1}{n_2}$, $[a,b]$, $[N]$, $\#S$, $A\backslash B$, $\prod_{i\in [I]}S_i$, $S^I$, $\chi_S$, $\max(S)$, $\diam(\Omega)$,  $\mathring{S}$, $\bar{S}$, $\left<\b a, \b b\right>$, $\left\|\b a \right\|_2$, $\S^{d-1}$, $\b a\ge \b b$, $\b a>\b b$, $\b u \otimes \b v$, $\diag(\b v)$, $\b M^\T$, $\left<\b A, \b B\right>_F$, $\left\|\b A\right\|_F$, $\b e_i$, $\b 0_{d_1 \times d_2}$, $\b I_d$, $S_d$, $S_d^+$, $[\b V_1;\ldots;\b V_N]$, $[\b H_1,\ldots,\b H_N]$,  $\Ktens \b A$, $|\b r|$, $\p^{\b r}$, $\p^{\b r,\b q}$, $\C^{s}\left(\X,\R^{d}\right)$, $\C^{s,s}\left(\X\times \X,\R^{d_1\times d_2}\right)$, $O_{1,s}$, $O_{Q,s}$,  $H^+_{\F}(f,\rho)$, $H^-_{\F}(f,\rho)$, $H_{\F}(f,\rho)$,  $\BB_{\F}(c,r)$,  $\BB_{\X}(\b c,r)$, $V^\perp$.  Depending on the reader's background, one may skip these definitions, and return to them if necessary.
	
	\tb{Sets:}
	Let $\N=\{0,1,\ldots\}$, $\Zp = \{1,2,\ldots\}$ and $\Rnn$ denote the set of natural numbers, positive integers and non-negative reals, respectively. We write $\iv{n_1}{n_2}=\{n_1,n_1+1,\ldots,n_2\}$ for the set of integers between $n_1, n_2\in \N$ (not to be confused with the closed interval $[a,b]$) and use the shorthand $[N]:=\iv{1}{N}$ with $N\in \N$, with the convention that $[0]$ is the empty set. The cardinality of a set $S$ is denoted by $\#S$, the difference of two sets $A$ and $B$ by $A\backslash B$. Given sets $(S_i)_{i\in [I]}$, let $\prod_{i\in [I]}S_i$ be their Cartesian product; we use the shorthand $S^I$ if $S = S_1 = \ldots = S_I$. For a set $S$, its indicator function is $\chi_S$: $\chi_S(x) = 0$ if $x \in S$, $\chi_S(x)=\infty$ otherwise. The maximum of a set $S\subset \R$ with finite cardinality is denoted by $\max(S)$. Let the diameter of a set $\Omega$ contained in a normed space $(\F,\|\cdot\|_{\F})$ be denoted by $\diam(\Omega)=\sup_{\b x,\b y\in\Omega}\left\|\b x - \b y\right\|_{\F}$; $\diam(\Omega)<\infty$ if $\Omega$ is bounded. The interior of a set $S\subseteq \F$ is denoted by $\mathring{S}$, its closure by $\bar{S}$.  Throughout the paper $\X\subseteq \R^d$ denotes a set  which is contained in the closure of its interior ($\X \subseteq \bar{\mathring{\X}}$).\footnote{Examples of such sets include for instance all open sets or half intervals $[a,b)$ where $a\in \R$, $b \in \R\cup\{\infty\}$. Counter-examples are sets with isolated points, which are unsuitable for differentiation of functions.}
	
	\tb{Linear algebra:} The inner product of  vectors $\b a, \b b \in \R^d$ is denoted by $\left<\b a, \b b\right>=\sum_{i\in [d]}a_ib_i$; the Euclidean norm is written as $\left\|\b a\right\|_2=\sqrt{\left<\b a, \b a\right>}$. The $d$-dimensional sphere is denoted by $\S^{d-1}=\left\{\b x\in\R^d\,:\, \left\|\b x\right\|_2 = 1\right\}$. For vectors $\b a$ and $\b b\in \R^d$, $\b a\ge \b b$ means that $a_i \ge b_i$ for all $i \in [d]$. Similarly, $\b a >\b b$ is defined as $a_i>b_i$ for all $i\in [d]$.  Let the tensor product of vector $\b u\in \R^{d_1}$ and $\b v\in \R^{d_2}$ be defined as $\b u \otimes \b v = [u_i v_j]_{i\in [d_1],\, j\in [d_2]}\in \R^{d_1 \times d_2}$. The $d\times d$-sized matrix with diagonal $\b v \in \R^d$ is $\diag(\b v)$. The transpose of a matrix $\b M$ is $\b M^\T$. The Frobenius product of the matrices $\b A, \b B \in \R^{d_1\times d_2}$ is $\left<\b A, \b B\right>_F = \sum_{i\in [d_1],\,j\in [d_2]}A_{ij}B_{ij}$; the associated Frobenius norm is $\left\|\b A\right\|_F = \sqrt{\left<\b A, \b A\right>_F }$. The $i^{th}$ canonical basis vector is $\b e_i$; the zero matrix is $\b 0_{d_1\times d_2} \in \R^{d_1\times d_2}$; the identity matrix is denoted by $\b I_d \in \R^{d\times d}$. The set of $d\times d$ symmetric (resp.\ positive semi-definite) matrices is denoted by $S_d$ (resp.\ $S_d^+$). 
	The vertical concatenation of matrices $\b V_1\in \R^{d_1\times d}, \ldots,\b V_N \in \R^{d_N \times d}$ is $[\b V_1;\ldots;\b V_N]\in \R^{\left(\sum_{n\in [N]}d_n\right) \times d}$; similarly the horizontal concatenation of $\b H_1 \in \R^{d \times d_1}, \ldots, \b H_N \in \R^{d \times d_N}$ is $[\b H_1,\ldots,\b H_N] \in \R^{d \times \left(\sum_{n\in[N]}d_n\right)}$. A tensor  $\Ktens\in \R^{d\times d \times d \times d}$ defines an $\R^{d\times d} \rightarrow \R^{d\times d}$ bounded linear operator by acting on a matrix $\b A\in \R^{d\times d}$ as $(\Ktens \b A)_{i,j}:=\sum_{n,m\in[d]}a_{n,m}\Ktens_{i,j,n,m}$ with $i,j\in [d]$.
	
	\tb{Analysis:}	Given a multi-index $\b r \in \N^d$ let $|\b r| = \sum_{j\in [d]} r_j$ be its length, and let the $\b r^{th}$ order partial derivative of a function $f$ be denoted by $\p^{\b r}f (\b x) = \frac{\p^{|\b r|}f(\b x)}{\p x_1^{r_1}\ldots \p x_d^{r_d}}$. Similarly for multi-indices $\b r,\b q \in \N^d$, let $\p^{\b r,\b q}f (\b x,\b y) = \frac{\p^{|\b r|,|\b q|}f(\b x,\b y)}{\p x_1^{r_1}\ldots \p x_d^{r_d}\p y_1^{q_1}\ldots \p y_d^{q_d}}$. For a fixed $s \in \N$, let the set of $\R^{d}$-valued functions on $\X$ with continuous derivatives up to order $s$  be denoted by $\C^s\left(\X,\R^{d}\right)$. 
	The set of $\R^{d_1\times d_2}$-valued functions on $\X\times \X$ for which  $\p^{\b r,\b r}f$ exists and is continuous up to order $|\b r|\le s\in \N$  is denoted by $\C^{s,s}\left(\X\times \X,\R^{d_1\times d_2}\right)$. Let the set of linear differential operators  of order at most $s\in \N$ on real-valued functions be denoted by $O_{1,s} = \left\{D\,:\, D(f)(\b x)= \sum_{j \in J}c_j \p^{\b r_j}f(\b x),\, \#J<\infty,\, |\b r_j| \le s,\, c_j \in \R\,\, (\forall j\in J)\right\}$. The set of linear differential operators  of order at most $s\in \N$ on $\R^Q$-valued functions is $O_{Q,s}=\left\{D: D(\b f)(\b x) = \sum_{q\in [Q]}\beta_q D_q(f_q)(\b x),\, \beta_q \in \R,\, D_q \in O_{1,s}\right\}$.
	
	\tb{Hilbert spaces:}
	Let $\F$ be a Hilbert space. For $f \in \F$ and  
	$\rho \in \R$, let the closed half-spaces and the affine hyperplane associated to the pair $( f,\rho)$ be defined as $H^+_{\F}(f,\rho) =\left\{ g\in \F\,:\, \left<f,g\right>_\F\geq \rho\right\}$,
	$H^-_{\F}(f,\rho)=\left\{g\in \F\,:\, \left< f,g \right>_\F\leq \rho\right\}$, $H_{\F}(f,\rho)=\left\{g\in \F\,:\, \left<f,g\right>_\F= \rho\right\}$. 
	The closed ball in $\F$ with center $c\in \F$ and radius $r>0$ is $\BB_{\F}(c,r)=\left\{f \in \F\,:\, \left\|c - f\right\|_{\F}\le r\right\}$. When $\F = \X \subseteq \R^d$ is equipped with a norm $\left\|\cdot\right\|_\X$, we write $\BB_{\X}(\b c,r)$ for balls. Let $V$ be a closed subspace of a Hilbert space $\F$, the orthogonal complement of $V$ in $\F$ is $V^\perp = \left\{f\in \F\,:\, \left<f, g\right>_{\F}=0\,\,\, \forall g \in V\right\}$.  
	
	\section{Problem Formulation} \label{sec:problem-formulation}
	In this section we formulate our problem after recalling the definition of vector-valued reproducing kernel Hilbert spaces (vRKHS).
	
	\tb{vRKHS:}	A function $\K: \X \times \X \rightarrow \R^{Q \times Q}$ is called a matrix-valued kernel on $\X$ if
	$K(\b x,\b x')=K(\b x',\b x)^\T$ for all $\b x,\b x' \in \X$ and $\sum_{i,j\in[N]} \b v_i^\T K(\b x_i,\b x_j)\b v_j\ge 0$ for all 
	$N\in \Zp$, $\left\{\b x_n\right\}_{n\in [N]} \subset \X$	and $\{\b v_n\}_{n\in [N]} \subset \R^Q$.  For $\b x\in \X$, let $K(\cdot,\b x)$ be the mapping $\b x' \in \X \mapsto K\left(\b x',\b x\right) \in \R^{Q\times Q}$. Let $\FK$ denote the vRKHS associated to the kernel $K$; we use the shorthand $\left\|\cdot\right\|_K :=
	\left\|\cdot\right\|_{\FK}$ and $\left<\cdot,\cdot\right>_K := \left<\cdot,\cdot\right>_{\FK}$ for the norm and the inner product on $\FK$. The Hilbert space $\FK$ consists of $\X \rightarrow \R^Q$ functions for which (i) $K(\cdot,\b x)\b c \in \FK$ for all $\b x\in\X$ and $\b c \in \R^Q$, and (ii) $\left<\b f,K(\cdot,\b x)\b c\right>_{\K}=\left<\b f(\b x),\b c\right>$ for all $\b f\in \FK$, $\b x\in\X$ and $\b c\in \R^Q$. The first property of vRKHSs describes the basic elements of $\FK$, the second one is called the reproducing property; this property can be extended to function derivatives, see Lemma~\ref{lemma:reproducing} below. Constructively, $\FK = \overline{\Sp}\left\{K(\cdot,\b x)\b c\,:\, \b x \in \X, \b c \in \R^Q\right\}$ where $\Sp$ denotes the linear hull of its argument and the bar stands for closure w.r.t.\ $\|\cdot\|_K$.  Given a vRKHS $\FK$, we use the shorthands $H^+_{K}(\b f,\rho)$, $H^-_{\K}(\b f,\rho)$, $H_{\K}(\b f,\rho)$ and $\BB_{\K}(\b c,r)$ for $H^+_{\FK}(\b f,\rho)$, $H^-_{\FK}(\b f,\rho)$, $H_{\FK}(\b f,\rho)$ and $\BB_{\FK}(\b c,r)$. For differential operators $D, \tilde{D} \in O_{Q,s}$ defined as 
	$D(\b f)(\b x) = \sum_{q\in [Q]}\beta_q D_{q} (f_q)(\b x)$ and $\tilde{D}(\b f)(\b x') = \sum_{q\in [Q]}\tilde{\beta}_q \tilde{D}_{q} (f_q)(\b x')$ and for a kernel $\K \in \C^{s,s}\left(\X\times \X,\R^{Q\times Q}\right)$, indicating by a subscript $\b x$ or $\b x'$ the variable w.r.t.\ which the derivation is taken, let
	\begin{align}\label{def:DtopD}
		\tilde{D}^\top D K(\b x', \b x)=\sum_{q,q'\in[Q]} \tilde{\beta}_{q'} \beta_{q}\b e_{q'}^\top \tilde{D}_{q',\b x'} D_{q,\b x}K(\b x', \b x) \b e_{q}\in\R.   
	\end{align}
	
	In this paper we \tb{focus} on optimization problems over vRKHSs with hard affine SDP shape constraints on derivatives. Typical examples can be formulated in the empirical risk minimization framework. Assume that we have access to samples $S=(\b x_n,\b y_n)_{n\in [N]} \in (\X \times \R^Q)^N$ which are supposed to be fixed and $\X \subseteq \R^d$ is assumed to be contained in the closure of its interior. We are given a kernel $K:\X \times \X \rightarrow \R^{Q\times Q}$ with associated vRKHS $\FK$; $K$ is assumed to belong to $\C^{s,s}\left(\X\times \X, \R^{Q\times Q}\right)$ with order $s\in \N$. The function family $\FK$ is used to capture the relation between the random variables $\b x$ and $\b y$ via the samples $S$, with the optional usage of a bias term $\b b \in \R^B$. The goodness of the estimated pair $(\b f, \b b)\in \FK \times \R^B$ is measured through a loss function $L$ (with the samples $S$ kept fixed) which can take into account both function values and function derivatives at the input points $\b x_n$; their number $\#J_n$ is allowed to differ for each $n$. The function values and derivatives of interest at each point $\b x_n$ are represented by the linear differential operators $(D^0_{n,j})_{j \in J_n} \subset O_{Q,s}$. With these notations, an objective function to minimize for given $S$ is 
	\begin{align}
		\LcalS(\b f,\b b) & = L\left(\b b,\left(\left(D^0_{n,j}(\b f)(\b x_n)\right)_{j\in J_n}\right)_{n\in [N]}\right) + R\left(\|\b f\|_K\right)+\chi_\Bsc(\b b), \label{eq:obj-function}
	\end{align}
	where $L:\R^B \times \R^{\sum_{n\in [N]}\#J_n} \rightarrow \R \cup \{\infty\}$, $R: \Rnn \rightarrow \R$ is a regularizer, and $\Bsc \subseteq \R^B$ is a closed convex set. 
	The pair $(\b f,\b b)$ is required to satisfy $I\in \Zp$ \emph{hard affine SDP shape constraints} on given sets $\Kcons_i \subseteq \X$ which are assumed to be compact\footnote{\label{footnote:non-compact}While in general we assume the $\Kcons_i$-s to be compact in $\X$, this requirement can be relaxed to boundedness of their image in $\FK$ under additional assumptions; see remark 'Non-compact $\Kcons$' in  Section~\ref{sec:constr1}.}:
	\begin{align}
		C & =\left\{(\b f,\b b)\,:\, \b 0_{P_i\times P_i} \preccurlyeq  \Dsdp_i  ( \b f - \b  f_{0,i})(\b x)+ \diag(\bm{\Gamma}_i\b b- \b b_{0,i} ), \forall \, \b x \in \Kcons_i, \forall\, i\in [I]\right\}. \label{def_mixded_constraint}\tag{$\Csc$}
	\end{align}
	\noindent In \eqref{def_mixded_constraint} the operator $\Dsdp_i$ aggregates $s^{th}$ order derivatives to the SDP constraints, i.e.\  
	\begin{align}
		\Dsdp_i (\b f)(\b x) = \left[D^i_{p_1,p_2}(\b f)(\b x)\right]_{p_1,p_2 \in [P_i]} \in S_{P_i} 
	\end{align}
	is a symmetric matrix with elements $D^i_{p_1,p_2} \in O_{Q,s}$. For instance, when $Q=1$, $s=2$, $I=1$, $P_1=d$, $\b 0_{d\times d} \preccurlyeq \Dsdp_1:=[\p^{\b e_i + \b e_j}]_{i,j\in [d]}$ requires the estimated function to be convex when its domain is restricted to a (convex) compact set $\Kcons_1$; requiring the function to be convex only in a subset of its arguments can be achieved by setting $P_i < d$. Possible shifts in \eqref{def_mixded_constraint} are expressed by the terms $\b b_{0,i}\in \R^{P_i}$ and $\b f_{0,i} \in \FK$. The matrices $\bm\Gamma_i\in \R^{P_i\times B}$ allow linear interaction between the bias coordinates. The bias $\b b \in \R^B$ can be both variable (e.g.\ $f_q + b_q$) and constraint-related (such as $b_1 \le f(x)$, $b_2 \le f'(x)$); hence $B$ can differ from $Q$. The geometric intuition of the $(\b f,\b b)$ pair follows that of the classical support vector machines where $\b f$ controls the direction, whereas $\b b$  determines the bias of the optimal hyperplane. Thus our \tb{problem of interest} combining objective functions more general\footnote{Throughout the manuscript, objective functions are denoted by $\Lcal$; specifically, when they depend on samples $S$, we write $\Lcal_S$.} than \eqref{eq:obj-function} and the hard affine SDP constraints \eqref{def_mixded_constraint} can be written as
	\begin{align}
		\left(\bar{\b f},\bar{\b b}\right) &\in \argmin_{\substack{\b f\,\in\,\FK,\, \b b\,\in\,\R^B,\\\,(\b f, \b b)\,\in \, C
		}} \Lcal(\b f,\b b). \label{opt-cons}\tag{$\Psc$}
	\end{align}
	
	\noindent\tb{Remarks}:
	\begin{itemize}[labelindent=0em,leftmargin=1em,topsep=0cm,partopsep=0cm,parsep=0cm,itemsep=2mm]
		\item Rewriting SDP constraints as \eqref{def_mixded_constraint}: Using \eqref{def_mixded_constraint} one can incorporate affine SDP constraints of the form
		\begin{align*}
			\{(\b f,\b b)\,|\, \b 0_{P\times P} \preccurlyeq  \tilde{\Dsdp} ( \b f - \b  f_{0})(\b x)+ \b M, \forall \, \b x \in \Kcons\},
		\end{align*}
		where $\b M\in S_{P}$. Indeed, by setting $\bm{\Gamma}=\b 0$ and $\b b_{0}=-\text{eig}(\b M)$ to be the negative of the eigenvalues of $\b M$, and using the spectral theorem
		\begin{align*}
			\b 0_{P\times P} &\preccurlyeq  \tilde{\Dsdp} ( \b f - \b  f_{0})(\b x)+ \b M = \tilde{\Dsdp} ( \b f - \b  f_{0})(\b x)+ \b U^\T \text{eig}(\b M) \b U\\
			& = \tilde{\Dsdp} ( \b f - \b  f_{0})(\b x)+ \b U^\T \diag(\bm{\Gamma}\b b- \b b_{0}) \b U \Leftrightarrow\\
			\b 0_{P\times P} &\preccurlyeq  \underbrace{\b U \tilde{\Dsdp} ( \b f - \b  f_{0})(\b x) \b U^\T}_{=:\Dsdp  ( \b f - \b  f_{0})(\b x)} + \diag(\bm{\Gamma}\b b- \b b_{0}).
		\end{align*}
		\item Further specific cases of \eqref{def_mixded_constraint}: Examples of \eqref{def_mixded_constraint} beyond the more classical cases of non-negativity, monotonicity or convexity include for instance $n$-monotonicity, monotonicity w.r.t.\ various partial orderings, $n$-alternating monotonicity, or supermodularity \citep[Section~C]{aubin20hard}.
		\item Cases not covered in \eqref{def_mixded_constraint}: Examples \emph{not} covered directly by \eqref{def_mixded_constraint} include for instance the Slutzky shape constraint and the quasi-convexity formula which are alternative assumptions on demand or utility functions. These non-affine requirements write as $\frac{\p f(x_1,x_2)}{\p x_1} + f(x_1,x_2)\frac{\p f(x_1,x_2)}{\p x_2} \le 0,\, \forall x_1, x_2$ and $f(\alpha \b x + (1-\alpha)\b x')\le \max\left(f(\b x),f\left(\b x'\right)\right)\,\forall \alpha \in [0,1], \b x, \b x'$, respectively.
		\item Equality constraints in \eqref{opt-cons}:  In this article, our primary focus is on convex \emph{inequality} constraints, handled through an interior approximation. When considering equality constraints, since convex equalities are affine, they would effectively restrict the hypothesis class to a closed affine subspace of $\FK$. A closed subspace of a vRKHS is also a vRKHS, possibly with a different kernel. Finitely many equality constraints can be handled in our framework without difficulty and without changing kernel; see our example on shape optimization in Section~\ref{sec:app:catenary}. On the other hand, an infinite number of equality requirements may require to determine explicitly the kernel of the subspace, which can be difficult. Nevertheless this is possible for instance in the case of a linear control problem (see Section~\ref{sec:app:safety-crit-control} and footnote~\ref{footnote:control:FK}). 
	\end{itemize}\vspace{0.1cm}
	
	\noindent\tb{Examples}: It is instructive to consider a few examples for the problem family \eqref{opt-cons}.
	\begin{itemize}[labelindent=0em,leftmargin=1em,topsep=0cm,partopsep=0cm,parsep=0cm,itemsep=2mm]
		\item \underline{Joint quantile regression} (JQR; as for instance defined by \citealt{sangnier16joint}): Assume that we are given samples $S=(\b x_n,y_n)_{n\in [N]}$ from the random variable $(X, Y)$ with values in $\X \times  \R \subseteq \R^{d+1}$, as well as $Q$ levels $0<\tau_1<\ldots<\tau_Q<1$. Our goal is to estimate \emph{jointly} the $\tau_q$-quantiles of the conditional distributions $\P(Y|X=\b x)$ for $q\in [Q]$. In the JQR problem, the estimated $\tau_q$-quantile functions $(f_q + b_q)_{q\in [Q]}$ (modulo the biases $b_q\in\R$) belong to a real-valued RKHS $\F_k$ associated to a kernel $k:\X \times \X \rightarrow \R$, and they have to satisfy jointly a monotonically increasing property w.r.t.\ the quantile level $\tau$. It is natural to require this non-crossing property on the smallest rectangle containing the input points $(\b x_n)_{n\in [N]}$, in other words on $\Kcons=\prod_{j\in [d]} \left[\min\left\{ (\b x_n)_j \right\}_{n\in [N]}, \max\left\{ (\b x_n)_j \right\}_{n\in [N]}\right]$. Hence, the optimization problem in JQR takes the form 
		\begin{mini*}|s|
			{\substack{\b f \in (\F_k)^Q,\\ \b  b \in \R^Q}}{\LcalS\left(\b f,\b b\right) &:=\frac{1}{N}\sum_{q\in [Q]}\sum_{n \in [N]} \ell_{\tau_q}\left(y_n-[f_q(\b x_n)+b_q]\right) +\lambda_{\b b} \|\b b\|^2_2 + \lambda_f \sum_{q \in [Q]}\|f_q\|^2_{\F_k}}
			{}
			{}
			\addConstraint{f_{q}(\b x)+b_{q}}{\le f_{q+1}(\b x)+b_{q+1},\, \forall q\in [Q-1],\, \forall  \b x\in \Kcons},
		\end{mini*}
		where $\lambda_{\b b} > 0$, $\lambda_f > 0$,\footnote{\citet{sangnier16joint} used the same loss function but a soft non-crossing inducing regularizer inspired by matrix-valued kernels, and also set $\lambda_{\b b} = 0$.
		} and the so-called ``pinball loss'' is defined as $\ell_{\tau} (e) = \max(\tau e, (\tau - 1)e)$ with $\tau \in  (0, 1)$. This problem can be obtained as a specific case of \eqref{opt-cons} by choosing $B=Q$,  $s=0$, $I=Q-1$, $P_i = 1$, $D_i \b f = f_{i+1}-f_i$, $\bm \Gamma_i \b b =b_{i+1}-b_i$ ($\forall i\in [I]$), $K\left(\b x,\b x'\right) = k\left(\b x,\b x'\right) \b I_Q$, $\b f_{0,i} = \b 0$, $\b b_{0} = \b 0$, $\Bsc = \R^B$. Further details and numerical illustration on the JQR problem are provided by \citet{aubin20hard}.
		\item \underline{Convoy trajectory reconstruction} (CTR): Here, the goal is to estimate vehicle trajectories based on noisy observations. This is a typical situation with GPS measurements, where the imprecision can be compensated through side information, not using only the position of every vehicle but also that of its neighbors. Assume that there are $Q$ vehicles forming a convoy (i.e.\ they do not overtake and keep a minimum inter-vehicular distance between each other) with speed limits on the vehicles. For each vehicle $q$ we have $N_q$ noisy position measurements $(y_{q,n})_{n\in [N_q]}\subset\, \R$, each corresponding to vehicle-specific time points $(x_{q,n})_{n\in [N_q]}\subset\X:=[0,T]$; this results in the samples $S=(x_{q,n},y_{q,n})_{q\in[Q],n\in [N_q]}$. Without loss of generality, let the vehicles be ordered in the lane according to their indices ($q=1$ is the first, $q=Q$ is the last one). Let $d_{\text{min}}\ge 0$ be the minimum inter-vehicular distance, and $v_{\text{min}}$ be the minimal speed to keep.\footnote{The requirement $v_{\text{min}}=0$ means that the vehicles go forward. A maximum speed constraint can be imposed similarly.} By modelling the location of the $q^{th}$ vehicle at time $x$ as $b_q+f_q(x)$ where $b_q\in \R$, $f_q \in \Fk$ and $k: \X \times \X \rightarrow \R$ is a real-valued kernel,  the CTR  task can be formulated as 
		\begin{mini*}|s|
			{\substack{\b f =[f_q]_{q\in [Q]} \in (\F_k)^Q,\\ \b b \in \R^Q}}{\hspace{-0.2cm}\LcalS\left(\b f,\b b\right) &:= \frac{1}{Q}\sum_{q=1}^{Q}\left[\left(\frac{1}{N_q}\sum_{n=1}^{N_q} |y_{q,n}- \left(b_q+f_q(x_{q,n})\right)|^2 \right) + \lambda\|f_q\|^2_{\F_k}\right]}
			{}
			{}
			\addConstraint{d_{\text{min}}+b_{q+1}+f_{q+1}(x)}{\le b_q+f_q(x), \quad \forall q  \in [Q-1],\,  \forall x \in \X}
			\addConstraint{v_{\text{min}}}{\le f_{q}'(x) \quad \forall q \in [Q],\, \forall x \in\X.}
		\end{mini*}
		This problem can be obtained as a specific case of \eqref{opt-cons} by choosing $B=Q$, $s=1$, $I=2Q-1$, $P_i=1$ $(i\in [I])$, $K\left(\b x,\b x'\right) = k\left(\b x,\b x'\right) \b I_Q$, 
		$\b D_i(\b f) = f_i -f_{i+1}$ ($i\in [Q-1]$), $\bm \Gamma_i\b b = b_i- b_{i+1}$ ($i\in [Q-1]$), $b_{0,i} = d_{\text{min}}$ ($i\in [Q-1]$), $\b  D_i(\b f) = f'_{i-(Q-1)}$ ($i\in \{Q,Q+1,\ldots,2Q-1\}$), $\bm \Gamma_i =\b 0_{1,Q}$ ($i\in \{Q,Q+1,\ldots,2Q-1\}$), $b_{0,i} = v_{\text{min}}$ ($i\in \{Q,Q+1,\ldots,2Q-1\}$).
		This  application was investigated by \citet{aubin20kernel}.
		\item \underline{Further examples}: In Section~\ref{sec:numerical-demos} we consider four complementary problems with numerical illustration. The examples cover a shape optimization task (minimizing the deformation of a catenary under its weight, with a stand underneath), safety-critical control (piloting an underwater vehicle while avoiding obstacles), robotics (estimation of robotic  arm position), and econometrics (learning production functions).
	\end{itemize}
	
	\section{Constraints} \label{sec:constraints}
	In this section we propose two approaches to handle a single hard affine SDP shape constraint ($I=1$) appearing in \eqref{def_mixded_constraint} over a (non-finite) compact\textsuperscript{\ref{footnote:non-compact}} set $\Kcons$
	\begin{align}
		C_{P} & =\left\{(\b f,\b b)\,:\, \b 0_{P\times P} \preccurlyeq  \Dsdp  ( \b f - \b  f_{0})(\b x)+ \diag(\bm{\Gamma}\b b- \b b_{0}), \forall \, \b x \in \Kcons\right\}. \label{def:mixed-constraint:P>=1}\tag{$\Csc_{P}$}
	\end{align}
	Multiple shape constraints ($I>1$) can be addressed by stacking the presented results.
	
	There are two main challenges to tackle: 
	(i) $C_{P}$  cannot be directly implemented since $\Kcons$ is non-finite, (ii) deriving a representer theorem is also problematic as the number of evaluations of $\b f$ is non-finite. To address these challenges, we propose two complementary approaches (depending on the value of $P$) to tighten $C_{P}$ through finite coverings\footnote{By considering finite coverings, we make the problem amenable to optimization. This computational aspect is elaborated in \Cref{sec:optimization}.}:
	\begin{enumerate}[labelindent=0em,leftmargin=1em,topsep=0cm,partopsep=0cm,parsep=0cm,itemsep=2mm]
		\item \tb{Compact covering in $\FK$ with balls and half-spaces, $P=1$}: This first approach focuses on the real-valued case of $P=1$, i.e.
		\begin{align}
			C_{1} & =\left\{(\b f,\b b)\,:\, 0 \le  D  ( \b f - \b  f_{0})(\b x)+ \bm{\Gamma}\b b-  b_{0}, \forall \, \b x \in \Kcons\right\}. \label{def:mixed-constraint:P=1}\tag{$\Csc_{1}$}
		\end{align}
		We show that \eqref{def:mixed-constraint:P=1} can be written as the inclusion in the vRKHS $\FK$ of a compact set in a half-space. We then tighten this inclusion by taking a finite covering of the compact set through balls and half-spaces in $\FK$, and present a general theorem which enables one to translate such inclusions into convex equations.
		\item \tb{Upper bounding the modulus of continuity, $P\ge 1$}: Our second approach tackles the general case of $P\ge 1$, i.e.\ \eqref{def:mixed-constraint:P>=1}, through an upper bound on the modulus of continuity of $\b D (\b f-\b f_0)$ defined on a finite covering of $\Kcons$. The upper bound has the form $\eta_{m,P}\|\b f-\b f_0\|_K$ which leads to second-order cone (SOC) constraints instead of affine inequalities. We will see that the two methods coincide when $P=1$ and when only ball coverings are considered.\vspace{0.1cm}
	\end{enumerate}
	
	We start with a lemma stating the reproducing property for derivatives of matrix-valued kernels.
	\begin{lemma}[Reproducing property for derivatives with matrix-valued kernels]\label{lemma:reproducing}
		Let $s \in \N$, $\X \subseteq \R^d$ be a set which is contained in the closure of its interior, $\K$ be a matrix-valued kernel such that $\K \in \C^{s,s}\left(\X\times \X,\R^{Q\times Q}\right)$, and $D \in O_{Q,s}$ be a differential operator such that
		$D(\b f)(\b x) = \sum_{q\in [Q]}\beta_q D_q f_q(\b x)$. Let 
		\begin{align}
			D K(\b x',\b x) &:=\sum_{q\in [Q]} \beta_q [D_{q,\b x} K(\b x',\b x)] \b e_q \in \R^Q, \label{eq:DK-def}
		\end{align}
		where $D_{q,\b x} K(\b x',\b x) := D_q[\b x''  \mapsto K(\b x',\b x'')](\b x) \in \R^{Q\times Q}$ and $\b e_q \in \R^Q$ is the  $q^{th}$ canonical basis vector. Then
		\begin{align}
			\b f&\in \C^s \left(\X, \R^Q\right),&
			DK(\cdot,\b x)&\in\FK,&
			D (\b f)(\b x) &= \left<\b f, DK(\cdot,\b x)\right>_\K \label{eq:repr-prop}
		\end{align}
		for  all $\b f\in\FK$  and  $\b x \in \X$.
	\end{lemma}
	
	\noindent\tb{Remark}: Specifically for $s=0$, one has that $D(\b f)(\b x)=\sum_{q\in[Q]}\beta_q f_q(\b x)=\bm{\beta}^\top \b f(\b x)$ and \eqref{eq:repr-prop} reduces to the classical reproducing property in vRKHSs, i.e.\ $\bm{\beta}^\top \b f(\b x) = \left<\b f, K(\cdot,\b x)\bm{\beta}\right>_{\K}$.
	The reproducing property for kernel derivatives has been studied over open sets $\X$, for real-valued \citep{saitoh16theory} and matrix-valued \citep{micheli14matrix} kernels, and over compact sets which are the closure of their interior for real-valued kernels \citep{zhou08derivative}. In Lemma~\ref{lemma:reproducing} we generalize these results to matrix-valued kernels and to sets $\X$ which are contained in the closure of their interior.
	
	\subsection{Constraints by Compact Covering in $\FK$}\label{sec:constr1}
	In our \tb{first approach}, applying Lemma~\ref{lemma:reproducing}, we rephrase constraint \eqref{def:mixed-constraint:P=1} as an inclusion of sets using the nonlinear embedding $\bm{\Phi}_D:\b x \in \Kcons \mapsto DK(\cdot,\b x)\in\FK$
	\begin{align}
		(\b f, \b b) \in C_{1} &\Leftrightarrow  b_{0}-\bm{\Gamma}\b b \le  D ( \b f - \b  f_{0})(\b x)=\left<\b f - \b f_0, D K(\cdot,\b x)\right>_{\K}\, \forall \, \b x \in \Kcons \nonumber\\
		&\Leftrightarrow \bm{\Phi}_D(\Kcons):=\left\{D K(\cdot,\b x)\,:\, \b x \in \Kcons\right\} \subseteq H^{+}_{\K}(\b f - \b f_0, b_{0}-\bm{\Gamma}\b b). \label{eq:inclusion}
	\end{align}
	The set $\bm{\Phi}_D(\Kcons)$ is compact in $\FK$ since $\Kcons$ is compact in $\X$ and $\bm{\Phi}_D$ is continuous.	However it is intractable to directly ensure the inclusion described in \eqref{eq:inclusion} whenever $\Kcons$ is not finite. We thus consider an approximation with a ``simpler'' set $\bar{\Omega}$ containing $\bm{\Phi}_D(\Kcons)$,\footnote{\label{footnote:tr-invariant_unitBall}A simple example for translation-invariant kernels $K(\b x,\b y)=K_0(\b x-\b y)$ is the (coarse) approximation $\bm{\Phi}_D(\Kcons) \subseteq \BB_\K\left(\b 0,\sqrt{D^\top D K_0(\b 0)}\right)$. Indeed,   $\left\|DK(\cdot,\b x)\right\|_K = \sqrt{\left<DK(\cdot,\b x), DK(\cdot,\b x)\right>_K} = \sqrt{D^\top D K(\b x,\b x)} = \sqrt{D^\top DK_0(\b 0)}$ for any $\b x\in \X$ by Lemma~\ref{lemma:reproducing}.} and require the inclusion
	\begin{align}
		\bm{\Phi}_D(\Kcons) \subseteq\bar{\Omega} \subseteq H^{+}_{\K}(\b f - \b f_0,b_{0}-\bm{\Gamma}\b b) \label{eq:tightened-inclusion}
	\end{align}
	which implies \eqref{eq:inclusion}.
	Since $\bm{\Phi}_D(\Kcons)$ is compact, drawing upon compact coverings, we assume that 
	\begin{align}
		\bar{\Omega} = \cup_{m\in [M]} \bar{\Omega}_m,    \label{eq:Omega}
	\end{align}
	where each $\bar{\Omega}_m$ is the closure of a non-empty finite intersection ($J_{B,m}, J_{H,m} \in \N$) of non-trivial ($r_{m,j}>0$, $\b v_{m,j} \neq \b 0$) open balls and open half-spaces
	\begin{align}
		\Omega_m &= \left(\bigcap_{j\in [J_{B,m}]} \dBB_{\K}\left(\b c_{m,j},r_{m,j}\right)\right) \cap \left(\bigcap_{j \in [J_{H,m}]} \dH^{-}_{\K} \left(\b v_{m,j},\rho_{m,j}\right)\right).
		\label{eq:Omega_m}
	\end{align}
	\begin{figure}
		\centering
		\begin{minipage}[c]{.48\linewidth}
			\begin{center}			
				\subfloat[][]{\resizebox{\linewidth}{!}{\label{Diagram_ball_covering}\includegraphics[page={2},keepaspectratio, width=\textwidth]{FIG/Soap_bubble.pdf}}}
			\end{center}
		\end{minipage}
		\begin{minipage}[c]{.48\linewidth}
			\begin{center}			
				\subfloat[][]{\resizebox{\linewidth}{!}{\label{Diagram_halfspace_covering}\includegraphics[page={1},keepaspectratio, width=\textwidth]{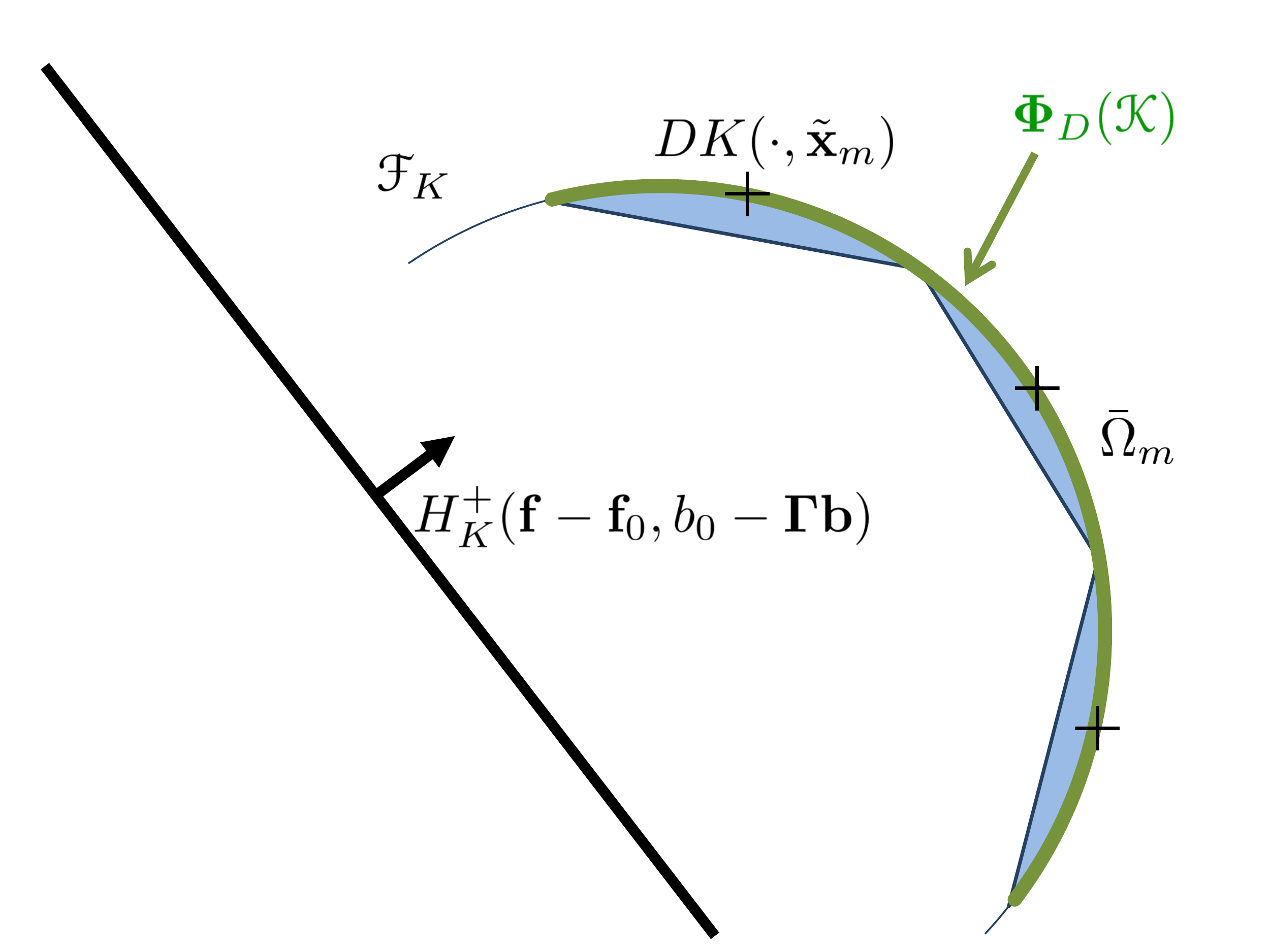}}}
			\end{center}
		\end{minipage}
		
		\caption{Two examples of coverings in $\FK$ of $\bm{\Phi}_D(\Kcons)$ by a set $\bar{\Omega}= \cup_{m\in [M]} \bar{\Omega}_m$ contained in the halfspace $H^{+}_{\K}(\b f - \b f_0,b_{0}-\bm{\Gamma}\b b)$. \protect\subref{Diagram_ball_covering}: covering through balls $\Omega_m=\mathring{\BB}_{\K}\left( DK(\cdot, \tilde{\b x}_{m}), \eta_m\right)$. \protect\subref{Diagram_halfspace_covering}: covering through a ball intersected with halfspaces ($J_{B,m}=J_{H,m}=1$).}
		\label{Diagram_covering}
	\end{figure}
	\noindent\tb{Remarks}: 
	\begin{itemize}[labelindent=0em,leftmargin=1em,topsep=0cm,partopsep=0cm,parsep=0cm,itemsep=2mm]
		\item Form of \eqref{eq:Omega_m}:
		The motivation for considering $\bar{\Omega}$ and $\bar{\Omega}_m$ of the form \eqref{eq:Omega} and \eqref{eq:Omega_m} is several-fold. Having a finite description enables one to derive a representer theorem. However, only a few sets (mainly points, balls and half-spaces) enjoy explicit convex separation formulas.\footnote{\label{footnote:convex-separation}Since $\Omega_m$ is convex, the inclusion $\Omega_m \subseteq H^{+}_{\K}(\b f - \b f_0,b_{0}-\bm{\Gamma}\b b)$ in \eqref{eq:tightened-inclusion} is equivalent to $\Omega_m \cap \dH^{-}_{\K}(\b f - \b f_0,b_{0}-\bm{\Gamma}\b b)=\emptyset$ which can be interpreted as a convex separation.} Focusing on \emph{points} leads to a discretization of $\bm{\Phi}_D(\Kcons) $ and greedy strategies (such as the Frank-Wolfe algorithm), but without guarantees outside of the points considered. A finite union of \emph{balls} can approximate any compact set in the Hausdorff metric, but balls result in enforcing ``buffers'' in every direction of $\FK$. A finite intersection of \emph{half-spaces} can approximate any convex set,\footnote{To motivate the use of half-spaces in \eqref{eq:Omega_m}: notice that since the half-space on the r.h.s.\ of \eqref{eq:inclusion} is closed and convex, \eqref{eq:inclusion} is equivalent to the fact that the closed convex hull $\cob (\Phi_{D}(\Kcons))$ is a subset of $H^{+}_{\K}(\b f - \b f_0,b_{0}-\bm \Gamma\b b)$. Using the support function characterization of closed convex sets, we have that $\cob (\Phi_{D}(\Kcons))=\bigcap_{\b g\in \FK}H^-_K(\b g,\, \sigma_{\Kcons}(\b g))$ where  $\sigma_{\Kcons}(\b g):=\sup_{x\in \Kcons} D (\b g)(x)$ has to be computed. Considering any finite collection ($\tilde{\F}_K$) of $g\in\FK$, the resulting intersection $\bigcap_{\b g\in \tilde{\F}_K}H^-_K(\b g,\, \sigma_{\Kcons}(\b g))$ has finite description and contains $\cob (\Phi_{D}(\Kcons))$.} but this finite intersection is always unbounded for infinite-dimensional $\FK$ resulting in a poor approximation of compact sets. Motivated by obtaining guarantees, we thus consider combinations of balls and half-spaces as in \eqref{eq:Omega}-\eqref{eq:Omega_m}.
		\item Non-compact $\Kcons$: 
		Coverings of the form \eqref{eq:Omega} and \eqref{eq:Omega_m} exist for any bounded $\bm{\Phi}_D(\Kcons)$. In particular, if $D^\top DK(\cdot,\cdot)$, defined in \eqref{def:DtopD}, is bounded over $\X\times\X$ (in other words, $\sup_{\b x,\b x' \in \X}|D^\top DK\left(\b x,\b x'\right)| <\infty$), then for any set $\Kcons\subseteq\X$, $\bm{\Phi}_D(\Kcons)\subset \FK$ is bounded as well. Consequently the proposed method can be applied to non-compact $\Kcons$ provided that the derivatives of the chosen kernel are bounded.
		\vspace{0.1cm}
	\end{itemize}
	
	Theorem \ref{thm:inclusion} below provides an explicit convex formula to be satisfied that is equivalent to the tightened inclusion \eqref{eq:tightened-inclusion} under the choice \eqref{eq:Omega}-\eqref{eq:Omega_m}. Since $\bar{\Omega}_m$ is the closure of the non-empty open set $\Omega_m$, and the half-spaces $ H^{+}_{\K}$ are closed,
	\begin{align}
		\bar{\Omega}  = \cup_{m\in [M]} \bar{\Omega}_m \subseteq H^{+}_{\K}(\b f - \b f_0,b_{0}-\bm{\Gamma}\b b)\, &\Leftrightarrow\, \bar{\Omega}_m \subseteq H^{+}_{\K}(\b f - \b f_0,b_{0}-\bm{\Gamma}\b b)\, \forall\,m\in [M]\nonumber\\
		&\Leftrightarrow\, \Omega_m \subseteq H^{+}_{\K}(\b f - \b f_0,b_{0}-\bm{\Gamma}\b b)\, \forall\,m\in [M]. \label{eq:inclusion:m-fixed}
	\end{align}
	Hence we can consider separately the inclusion of each $\Omega_m$, formulate the theorem for $M=1$ and drop the index $m$.
	
	\begin{theorem}[Inclusion formula for balls and half-spaces]\label{thm:inclusion} Let $\FK$ be the vRKHS associated to an $\R^{Q\times Q}$-valued kernel $\K$.
		Then the following statements are equivalent:
		\begin{enumerate}[labelindent=0em,leftmargin=1em,topsep=0.2cm,partopsep=0cm,parsep=0cm,itemsep=2mm]
			\item $\emptyset \neq\Omega :=\left(\bigcap_{j\in [J_B]} \dBB_{\K}(\b c_{j},r_{j})\right) \cap \left(\bigcap_{j \in [J_H]} \dH^{-}_{\K} (\b v_{j},\rho_{j})\right) \subseteq H^{+}_{\K}\left(\b f-\b f_0,b_0-\bm{\Gamma} \b b\right)\neq \FK$.
			\item There exists $J_B$ functions $(\b g_j)_{j\in [J_B]} \subset \Sp\left(\b f-\b f_0, \left\{\b c_j\right\}_{j\in [J_B]},\left\{\b v_j\right\}_{j\in [J_H]}\right)$ and $J_H$ non-negative coefficients $(\xi_j)_{j\in [J_H]} \in \Rnn^{J_H}$ such that
			\begin{equation} \label{ineq-eq_thm_DM} 
				\begin{split}
					-  b_0+\bm{\Gamma} \b b + \sum_{j \in [J_B]} \left<\b g_j,\b c_j\right>_{\K} - \sum_{j\in [J_H]} \xi_j \rho_j - \sum_{j\in [J_B]} r_j \left\|\b g_j\right\|_{\K} &\ge 0, \\
					-(\b f-\b f_0) + \sum_{j\in [J_B]}\b g_j - \sum_{j \in [J_H]}\xi_j \b v_j &= \b 0.
				\end{split}
			\end{equation}
		\end{enumerate}
	\end{theorem}
	\tb{Remarks}: 
	\begin{itemize}[labelindent=0em,leftmargin=1em,topsep=0cm,partopsep=0cm,parsep=0cm,itemsep=2mm]
		\item \eqref{def:mixed-constraint:P=1_SOC} is tighter than \eqref{def:mixed-constraint:P=1}:
		Theorem~\ref{thm:inclusion} provides a general finite-dimensional formula for separating convex sets combining balls and half-spaces. It can be of independent interest for studies in RKHSs. Specifically,  using the notation
		\begin{align}
			C_{1,\Omega} & =\left\{(\b f,\b b)\,:\, \eqref{ineq-eq_thm_DM} \text{ holds for all $m\in[M]$} \right\}, \label{def:mixed-constraint:P=1_SOC}\tag{$\Csc_{1,\Omega}$}
		\end{align}
		requiring $(\b f,\b b)\in C_{1,\Omega}$ is equivalent to having $\bar{\Omega}  \subseteq H^{+}_{\K}(\b f - \b f_0,b_{0}-\bm{\Gamma}\b b)$ by Theorem \ref{thm:inclusion} and \eqref{eq:tightened-inclusion}. Hence, owing to \eqref{eq:inclusion}, \eqref{def:mixed-constraint:P=1_SOC} is a tighter constraint than \eqref{def:mixed-constraint:P=1}, i.e.\ $C_{1,\Omega}\subseteq C_{1}$.\textsuperscript{\ref{footnote:relax-tighten}} 
		\item Illustration: A visual illustration of the inclusion relation for a single ball ($J_B=1, J_H=0$), and for one ball and one half-space ($J_B=J_H=1$) is given in Fig.~\ref{Diagram_covering}(a) and Fig.~\ref{Diagram_covering}(b), respectively.
		\item Case of a single ball ($J_B=1$, $J_H=0$): In the simplest case where there is a single ball and no half-spaces, i.e.\ $\Omega=\dBB_{\K}(\b c, r)$, then \eqref{ineq-eq_thm_DM} reduces to $\b g =\b f-\b f_0$ and $\left<\b g,\b c\right>_{\K} \ge r \left\|\b g\right\|_{\K}+  b_0-\bm{\Gamma} \b b$, thus
		\begin{align}
			\left<\b f-\b f_0,\b c\right>_{\K} \ge r \left\|\b f-\b f_0\right\|_{\K}+  b_0-\bm{\Gamma} \b b. \label{eq:1-ball}
		\end{align}
		\item Case of one ball and one half-space ($J_B=J_H=1$): In this case $\Omega=\dBB_{\K}(\b c, r)\cap \dH^{-}_{\K} (\b v,\rho)$ and  \eqref{ineq-eq_thm_DM} writes as 
		$\b f - \b f_0 = \b g -\xi \b v$ and $\left<\b g,\b c\right>_{\K} \ge r \left\|\b g\right\|_{\K}+  b_0-\bm{\Gamma} \b b +\xi \rho$, thus
		\begin{align}
			\left<\b f-\b f_0+\xi \b v,\b c\right>_{\K} \ge r \left\|\b f-\b f_0+\xi \b v\right\|_{\K}+  b_0-\bm{\Gamma} \b b +\xi \rho. \label{eq:construction1:1ball-1hyperplane}
		\end{align}
		\item Constructing $\Omega$ using the compactness of $\Kcons$: A natural choice of $\Omega$ of the form \eqref{eq:Omega} can be obtained by leveraging the compactness of $\Kcons$. Indeed, let us take any finite covering 
		of $\Kcons$ through balls centered at $M$ points $\{\tilde{\b x}_{m}\}_{m\in[M]}$ with radius $\delta_m>0$. Then one can cover the sets  $\Phi_{D}(\BB_{\X}(\tilde{\b x}_{m},\delta_m)) \subset \FK$ by balls $\Omega_m=\mathring{\BB}_{\K}( DK(\cdot, \tilde{\b x}_{m}), \eta_m)$ with radii 
		\begin{align}
			\eta_m &= \sup_{\b x \,\in\,\BB_{\X}\left(\tilde{\b x}_{m},\delta_{m}\right)} \hspace*{-0.4cm} \|DK(\cdot, \tilde{\b x}_{m})-DK(\cdot, \b x)\|_K,\, m\in [M]. \label{eq:eta-m-def}
		\end{align}
		In other words, $\Phi_{D}(\Kcons)\subseteq\bigcup_{m\in [M]}\Phi_{D}(\BB_{\X}(\tilde{\b x}_{m},\delta_m))\subseteq \bigcup_{m\in [M]} \bar{\Omega}_m =:\bar{\Omega}$, hence  $\bar{\Omega}$ satisfies \eqref{eq:tightened-inclusion} and $(\b f,\b b)\in C_{1,\Omega}$. In this case  \eqref{def:mixed-constraint:P=1} has been strengthened to the SOC constraints
		\begin{align}
			\left<\b f-\b f_0,\b c_m\right>_{\K} \ge r_m \left\|\b f-\b f_0\right\|_{\K}+  b_0-\bm{\Gamma} \b b,\quad \forall m\in [M],
			\label{eq:approach1:ball-only-specialization}
		\end{align}
		where $c_m := DK(\cdot, \tilde{\b x}_{m})$ and $r_m := \eta_m$, by  using \eqref{eq:1-ball}.\vspace{0.1cm}
	\end{itemize}
	The tightening we detailed in this section allows for a large class of coverings based on balls and half-spaces.  However the reformulation \eqref{eq:inclusion} heavily relies on the assumption of $P=1$. 
	
	\subsection{Constraints by Upper Bounding the Modulus of Continuity}
	We now present a \tb{second approach} capable of handling $P\ge 1$, i.e.\ the affine SDP constraint \eqref{def:mixed-constraint:P>=1}. The method relies on an upper bound of the modulus of continuity of $\b D (\b f-\b f_0)$ over a finite covering of a compact $\Kcons\subseteq \bigcup_{m\in[M]} \BB_{\X}(\tilde{\b x}_{m},\delta_{m})$. For simplicity, we present the high-level idea for $P=1$. Let the modulus of continuity of $D (\b f-\b f_0)$ on $\BB_{\X}\left(\tilde{\b x}_{m},\delta_{m}\right)$ be defined as 
	\begin{equation} \label{eq:modulus_continuity}
		\omega_{D (\b f-\b f_0)}(\tilde{\b x}_{m},\delta_{m}):=\sup_{\substack{\b x \,\in\,\BB_{\X}\left(\tilde{\b x}_{m},\delta_{m}\right)}}\left|D (\b f-\b f_0)(\b x)-D (\b f-\b f_0)\left(\tilde{\b x}_{m}\right)\right|.
	\end{equation}
	Assume that we have an exact finite covering, in other words  $\Kcons= \bigcup_{m\in[M]} \BB_{\X}(\tilde{\b x}_{m},\delta_{m})$. If $\omega_{D (\b f-\b f_0)}(\tilde{\b x}_{m},\delta_{m})$ was known for every $m\in[M]$, then the constraint $(\b f,\b b)\in C_{P}$ would be implied by
	\begin{equation}
		\omega_{D (\b f-\b f_0)}(\tilde{\b x}_{m},\delta_{m}) \le  D  ( \b f - \b  f_{0})(\tilde{\b x}_{m})+ \bm{\Gamma}\b b-  b_{0},\,\forall\,m\in[M]. \label{eq:MOC}
	\end{equation}
	The implication follows from \eqref{eq:modulus_continuity} since the modulus of continuity is the smallest upper bound on the variations of the values. Applying the reproducing property for derivatives (Lemma \ref{lemma:reproducing}) and the Cauchy-Schwarz inequality, we obtain an upper bound
	\begin{equation} \label{eq:modulus_continuity_bound}
		\omega_{D (\b f-\b f_0)}(\tilde{\b x}_{m},\delta_{m}) =  \sup_{\substack{\b x \,\in\,\BB_{\X}(\tilde{\b x}_{m},\delta_{m})}}\left|\left<\b f-\b f_0, DK(\cdot,\b x) -DK(\cdot,\tilde{\b x}_{m})\right>_\K\right|\le \eta_{m} \|\b f-\b f_0\|_K
	\end{equation}
	with  $\eta_{m}$ defined as in \eqref{eq:eta-m-def}. While the  original quantity $\omega_{D (\b f-\b f_0)}(\tilde{\b x}_{m},\delta_{m})$ can be hard to evaluate, the bound $\eta_m \left\| \b f - \b f_0\right\|_\K$ is much more favourable from a computational perspective. Indeed, the term $\eta_m$
	has an explicit finite-dimensional description (see Lemma~\ref{lemma:eta_SDP_4Dtensor} below), and combining \eqref{eq:modulus_continuity_bound} with \eqref{eq:MOC} gives rise to the tightened second-order cone (SOC) constraints
	\begin{align*}
		\eta_m \left\| \b f - \b f_0\right\|_\K & \le  D  ( \b f - \b  f_{0})(\tilde{\b x}_{m})+ \bm{\Gamma}\b b-  b_{0},\,\forall\,m\in[M]
	\end{align*}
	for which the term $\left\| \b f - \b f_0\right\|_\K$  of \eqref{eq:modulus_continuity_bound} ensures that the problem is still convex and implementable.
	
	The following theorem extends the idea presented in \eqref{eq:modulus_continuity_bound} to affine SDP constraints and states our result on how to translate a finite ball-covering of $\Kcons$ (meant w.r.t.\ a norm $\left\|\cdot\right\|_{\X}$) into a SOC tightening of \eqref{def:mixed-constraint:P>=1}.
	
	\begin{theorem}[Tighter constraint for ball covering in $\X$ and $P\geq 1$]\label{thm:SDP} Assume that the points $\{\tilde{\b x}_{m}\}_{m\in[M]}~\subset~\X$ associated with radii $\delta_m>0$ form a ball-covering of $\Kcons$, i.e.\  $\Kcons\subseteq \bigcup_{m\in[M]} \BB_{\X}(\tilde{\b x}_{m},\delta_{m})$. Let
		$\b D=\left[D_{p_1,p_2}\right]_{p_1,p_2\in [P]}$ ($D_{p_1,p_2}\in O_{Q,s}$), and 
		\begin{gather}\label{def:eta_SDP}
			\eta_{m,P} := \sup_{\substack{\b x \,\in\,\BB_{\X}(\tilde{\b x}_{m},\delta_{m}),\\\b u = [u_p]_{p\in [P]}\in \S^{P-1}}} \left\|\sum_{p_1,\, p_2\in[P]} u_{p_1}u_{p_2}\big[ D_{p_1,p_2}K(\cdot,\tilde{\b x}_{m})-D_{p_1,p_2}K(\cdot,\b x) \big] \right\|_K,\\
			C_{P,\text{SOC}} :=\left\{(\b f,\b b)\,:\, \eta_{m,P}\|\b f-\b f_0\|_K \b I_{P} \preccurlyeq \b D(\b f-\b f_0)(\tilde{\b x}_{m}) + \diag(\bm{\Gamma}\b b - \b b_{0} ), \forall \, m\in[M]\right\}.\ \label{def:mixed-constraint:P>=1_SOC}\tag{$\Csc_{P,\text{SOC}}$}
		\end{gather}
		Then \eqref{def:mixed-constraint:P>=1_SOC} is tighter than \eqref{def:mixed-constraint:P>=1}, i.e. $C_{P,\text{SOC}}\subseteq C_{P}$.
	\end{theorem}
	The following lemma provides a more explicit, finite-dimensional  description of 	$\eta_{m,P}$.
	
	\begin{lemma}[Finite-dimensional description of $\eta_{m,P}$]\label{lemma:eta_SDP_4Dtensor} 
		For $\b z \in \X$, $\delta >0$, and a differential operator $\b D=\left[D_{p_1,p_2}\right]_{p_1,p_2\in [P]}$ ($D_{p_1,p_2}\in O_{Q,s}$), let 
		\begin{align}
			\eta(\b z, \delta; \b D) &:= \sup_{\substack{\b x \,\in\,\BB_{\X}\left(\b z,\delta\right),\\\b u\in \S^{P-1}}} |\langle \b u\otimes\b u,\left[\Ktens(\b z,\b z)+\Ktens(\b x,\b x)-2\Ktens(\b z,\b x)\right](\b u\otimes\b u) \rangle_F|^{1/2}, \label{eq:eta-function}
		\end{align}
		with the symmetric 4D-tensor $\Ktens(\b x',\b x)=\left[\Ktens(\b x',\b x)_{p_1,\,p_2,\,p'_1,\,p'_2}\right]_{p_1,\,p_2,\,p_1',\,p_2'\in [P]}\in \R^{P\times P\times P\times P}$, having elements $\Ktens(\b x',\b x)_{p_1,\,p_2,\,p'_1,\,p'_2}:=D_{p'_1,\,p'_2}^\top D_{p_1,\,p_2}K(\b x',\b x)$ and acting as a linear operator over matrices of $\R^{P\times P}$. Then, the quantity $\eta_{m,P}$ defined in \eqref{def:eta_SDP} can be written as 
		\begin{align}\label{eq:eta_SDP_explicit}
			\eta_{m,P} &= \eta\left(\tilde{\b x}_{m},\delta_{m};\b D\right).
		\end{align}
	\end{lemma}
	
	\noindent\tb{Remarks}: 
	\begin{itemize}[labelindent=0em,leftmargin=1em,topsep=0cm,partopsep=0cm,parsep=0cm,itemsep=2mm]
		\item Relation of $\eta_{m,P}$ to the eigenvalues of $\Ktens$: Since $\b u \in \S^{p-1}$,
		$\left\|\b u \otimes \b u\right\|_F^2 =1$. This means that $\eta_{m,P}$ can be upper bounded by the supremum over the ball $\BB_{\X}\left(\tilde{\b x}_{m},\delta_{m}\right)$ of the square root of the maximal eigenvalue  $\lambda_{\text{max}}$ of the 4D-tensor $\Ktens(\tilde{\b x}_{m},\tilde{\b x}_{m})+\Ktens(\b x,\b x)-2\Ktens(\tilde{\b x}_{m},\b x)$. Indeed,
		\begin{align*}
			\eta_{m,P} &= \sup_{\substack{\b x \,\in\,\BB_{\X}\left(\tilde{\b x}_{m},\delta_{m}\right),\\\b u\in \S^{P-1}}} \left|\langle\b u\otimes\b u,\left[\Ktens(\tilde{\b x}_{m},\tilde{\b x}_{m})+\Ktens(\b x,\b x)-2\Ktens(\tilde{\b x}_{m},\b x)\right](\b u\otimes\b u) \rangle_F\right|^{1/2}\\
			&\le \sup_{\substack{\b x \,\in\,\BB_{\X}\left(\tilde{\b x}_{m},\delta_{m}\right),\\\b U\in \R^{P\times P}\,:\, \left\|\b U\right\|_F=1}} \left|\left<\b U,\left[\Ktens(\tilde{\b x}_{m},\tilde{\b x}_{m})+\Ktens(\b x,\b x)-2\Ktens(\tilde{\b x}_{m},\b x)\right]\b U \right>_F\right|^{1/2}\\
			& = \sup_{\b x \,\in\,\BB_{\X}\left(\tilde{\b x}_{m},\delta_{m}\right)} \lambda_{\text{max}}^{1/2}\left(\Ktens(\tilde{\b x}_{m},\tilde{\b x}_{m})+\Ktens(\b x,\b x)-2\Ktens(\tilde{\b x}_{m},\b x)\right).
		\end{align*}
		In particular, by continuity of the spectral radius, this ensures that $\eta_{m,P}=\eta_{m,P}(\delta_{m})$ converges to zero when $\delta_{m}$ goes to zero. Hence when the discretization steps $(\delta_{m})_{m\in[M]}$ decrease to zero, we recover the original constraint \eqref{def:mixed-constraint:P>=1}. 
		\item Equivalence of Theorem \ref{thm:inclusion} and Theorem \ref{thm:SDP} for balls and $P=1$: When $P=1$, $u \in \S^0$ means that $|u|=1$. Hence  $|u_1 u_2|=1$ can be pulled out from \eqref{def:eta_SDP} and $\eta_{m,1}$ reduces to
		\begin{align*}
			\eta_{m,1} & =\sup_{\b x \,\in\,\BB_{\X}(\tilde{\b x}_{m},\delta_{m})} \hspace*{-0.4cm} \|DK(\cdot, \tilde{\b x}_{m})-DK(\cdot, \b x)\|_K, 
		\end{align*}
		so we recover $\eta_m$ as defined in \eqref{eq:eta-m-def}, and as anticipated in \eqref{eq:modulus_continuity_bound}. In other words, for $P=1$, when choosing a ball covering $\Omega_m=\mathring{\BB}_{\K}(DK(\cdot,\tilde{\b x}_{m}),\eta_{m})$, Theorem \ref{thm:inclusion} coincides with Theorem \ref{thm:SDP}. This specific choice  was followed by \citet{aubin20hard}. The two theorems presented here have complementary advantages: for real-valued constraints, Theorem \ref{thm:inclusion} allows more general coverings than just balls, whereas Theorem \ref{thm:SDP} is able to handle affine SDP constraints with $P> 1$.
		\item Computation of $\eta_{m,P}$: The value of 
		$\eta_{m,P}$ can be computed analytically in various cases. For instance, for $P=1$ and $D=\Id$, with a monotonically decreasing radial kernel $K(\b x,\b y)=K_0(\left\|\b x-\b y\right\|_\X)$ (such as the Gaussian kernel), \eqref{eq:eta_SDP_explicit} simplifies to 
		\begin{align}\label{eq:eta_SDP_simple}
			\eta_{m,1}(\delta_{m}) =\sup_{\b x\in \BB_{\X}(\b 0,\delta_{m})}\hspace*{-0.1cm} \sqrt{ \left|2 K_0(0)- 2K_0\left(\left\|\b x\right\|_\X\right)\right|} = 
			\sqrt{ \left|2 K_0(0)- 2K_0\left(\delta_{m}\right)\right|}.
		\end{align}
		Depending on the choice of the kernel, similar computations could be carried out for higher-order derivatives. For translation-invariant kernels, $\eta_{m,P}$ can be computed on a single $\delta_{m}$-ball around the origin as in \eqref{eq:eta_SDP_simple}. A fast approximation of $\eta_{m,P}$ can also for instance be obtained by sampling $\b x$ (resp.\ $\b u$) in the ball $\BB_{\X}(\tilde{\b x}_{m},\delta_{m})$ (resp. sphere $\S^{P-1}$). Moreover, as $\eta_{m,P}$ is related to the modulus of continuity of $DK$, the smoother the kernel, the smaller $\eta_{m,P}$ and the tighter the approximation of $\bm{\Phi}_D(\Kcons)$. As intuitively explained in \eqref{eq:modulus_continuity_bound}, $\eta_{m,P}$ is \emph{one} possible upper bound on the modulus of continuity, enabling guarantees for hard shape constraints. Depending on the objective function $\Lcal$, this bound is also tight in the equality case of the Cauchy-Schwarz inequality \eqref{eq:modulus_continuity_bound}.
	\end{itemize}
		
	\section{Objective Function} \label{sec:optimization}
	In Section \ref{sec:constraints} we detailed how one can tighten an infinite number of affine SDP constraints over a compact set of $\X$ into finitely many convex constraints in RKHSs through finite coverings of compact sets in $\X$ or in $\FK$. The proposed construction tightens the constraints  \eqref{def_mixded_constraint} into the ones defined in \eqref{def:mixed-constraint:P>=1_SOC} and \eqref{def:mixed-constraint:P=1_SOC}. In this section we show the existence of solution and a certificate of optimality (Theorem~\ref{thm:certificate}) using these tightenings, followed by a posteriori and a priori bounds and convergence guarantees (Corollary~\ref{thm:aposteriori_bound},   Proposition~\ref{thm:apriori_bound}). Then we derive a representer theorem (Proposition~\ref{prop:reproducing}) which expresses the tightened optimization problem as a finite-dimensional one and hence enables numerical solutions. 
	
	\begin{theorem}[Existence, Certificate]\label{thm:certificate}
		Let $\X\subseteq \R^d$ be a set which is contained in the closure of its interior and is endowed with a  matrix-valued kernel $K\in \C^{s,s}\left(\X \times \X,\R^{Q\times Q}\right)$ for some $s\in \N$. Partition $[I]$ with $I\in \N$ into two disjoint index sets $\mathcal{I}_{\text{SOC}}$ and $\mathcal{I}_{\Omega}$ (i.e.\ $[I]=\mathcal{I}_{\text{SOC}} \mathop{\dot\cup} \mathcal{I}_{\Omega}$). Define the optimization problem 
		\begin{align}
			\left(\bar{\b f}_{\app}, \bar{\b b}_{\app}\right) &\in \argmin_{\substack{\b f\,\in\,\hatFK,\, \b b\,\in\,\R^B\\ (\b f,\b b) \in C_{\app}}} \Lcal\left(\b f,\b b\right),\label{opt_cons_SOC}\tag{$\Psc_{app}$}
		\end{align}
		where  $\Lcal:\FK\times\R^B\rightarrow \R\cup\{\infty\}$, $\hatFK$ is a closed subspace of $\FK$ equipped with $\|\cdot\|_K$, and $C_{\app} := \left(\bigcap_{i\in\mathcal{I}_{\text{SOC}}} C^i_{P_i,\text{SOC}}\right)\cap \left(\bigcap_{i\in\mathcal{I}_{\Omega}} C^i_{1,\Omega}\right)$,  $\left\{C^i_{P_i,\text{SOC}}\right\}_{i\in\mathcal{I}_{\text{SOC}}}$ and $\left\{C^i_{1,\Omega}\right\}_{i\in\mathcal{I}_{\Omega}}$ being specified in \eqref{def:mixed-constraint:P>=1_SOC} and \eqref{def:mixed-constraint:P=1_SOC}.
		\begin{enumerate}[labelindent=0em,leftmargin=1em,topsep=0.2cm,partopsep=0cm,parsep=0cm,itemsep=2mm]
			\item Existence: Assume that (i) $\Lcal$ is weakly lower semi-continuous (or shortly w-l.s.c) and coercive over $\FK\times \R^B$, and (ii) there exists an admissible pair $(\b f, \b b)\in C_{\app}\cap \dom(\Lcal)\cap (\hatFK\times \R^B)$.\footnote{An extended real-valued function $g:\Y\rightarrow \R\cup\{\infty\}$ over a Hilbert space $\Y$ is w-l.s.c.\ if its  sublevel sets $lev_\gamma(g):=\{y\in\Y\,:\, g(y)\le \gamma \}$ are weakly closed in $\Y$ for all $\gamma\in\R$, and coercive if the sets $lev_\gamma(g)$ are all bounded in $\Y$ \citep[see e.g.][Chapter 3.2]{attouch14variational}. The (effective) domain of $g$ is defined as $\dom(g)=\{y \in \Y  \,:\, g(y)<\infty\}$.\label{footnote:def:wlsc-coercive}} Then there exists a minimizer $\left(\bar{\b f}_{\app}, \bar{\b b}_{\app}\right)$ of \eqref{opt_cons_SOC} and a solution $(\bar{\b f},\bar{\b b})$ to \eqref{opt-cons}.
			\item Certificate of optimality: Let $v_\app$, $\bar{v}$ and $v_{\text{relax}}$ be the optimal values of \eqref{opt_cons_SOC}, \eqref{opt-cons} and of any given relaxation\textsuperscript{\ref{footnote:relax-tighten}} of \eqref{opt-cons}, then $v_{\text{relax}}\leq \bar{v} \leq v_{\app}$.
		\end{enumerate}
	\end{theorem}
    
\noindent\emph{Proof idea.} After showing that $C_{\app}$ is weakly closed, the existence of solution stems from a classical result in optimization, see \citet[Theorem 3.2.5]{attouch14variational}. The certificate is a direct consequence of the fact we provided a tightening.
\vspace{0.1cm}

\noindent\tb{Remark (Nyström method)}:
	In the RKHS literature one often reduces the search space $\FK$ to a subspace $\hatFK$, for instance by performing subsampling (known as Nyström approximation). Since a finite-dimensional subspace of a Hilbert space is closed, the Nyström scheme is specifically encompassed in Theorem~\ref{thm:certificate}.

	We now derive a posteriori bounds on the error of the variables in the strongly convex case, and a priori bounds which underline the role of the tightness of the covering.
	
	\begin{corollary}[A Posteriori Bound]\label{thm:aposteriori_bound}
		With the notations of \Cref{thm:certificate}, if $\Lcal$ is w-l.s.c and $(\mu_{\b f},\mu_{\b b})$-strongly convex w.r.t.\ $(\b f,\b b)$ and there exists an admissible pair $(\b f, \b b)\in C_{\app}\cap \dom(\Lcal)\cap (\hatFK\times \R^B)$, then $(\bar{\b f},\bar{\b b})$ and $\left(\bar{\b f}_{\app}, \bar{\b b}_{\app}\right)$ exist, are unique, and
		\begin{align}\label{ineq_aposteriori}
			\left\|\bar{\b f}_{\app}-\bar{\b f}\right\|_K &\leq \sqrt{\frac{2 (v_{\app}-v_{\text{relax}})}{\mu_{\b f}}}, &
			\left\|\b b_{\app}-\bar{\b b}\right\|_2 &\leq \sqrt{\frac{2 (v_{\app}-v_{\text{relax}})}{\mu_{\b b}}}.
		\end{align}
	\end{corollary}
	\noindent\emph{Proof idea.} This bound comes from a general result on strongly convex functions, \citep[Proposition~3.23]{peypouquet15convex}.
	\begin{proposition}[A Priori Bound]\label{thm:apriori_bound} 
		Let us use the notations of \Cref{thm:certificate}.\footnote{Recall that (i) $\bm{\Gamma}_i$ $(i\in \mathcal{I})$ is from \eqref{def_mixded_constraint}, (ii) the points $\left\{\tilde{\b x}_{i,m}\right\}_{m\in [M_i]}$ form a covering of the compact set $\Kcons_i$ $(i\in \mathcal{I}_{\text{SOC}})$, (iii) $\eta_{i,m,P_i}$ $(i\in \mathcal{I}_{\text{SOC}})$ is specified in \eqref{def:eta_SDP}, (iv) $\Omega_{i,m}$ ($i\in \mathcal{I}_{\Omega}$, $m\in[M_i]$) is defined according to \eqref{eq:Omega_m}.}		Assume that (i) $\left(\bar{\b f}_{\app}, \bar{\b b}_{\app}\right)$ of \eqref{opt_cons_SOC} and $(\bar{\b f},\bar{\b b})$ exist, (ii) $\hatFK=\FK$ and $\dom(\Lcal(\bar{\b f},\cdot))=\RR^B$, (iii) there exists $\bm{\beta}\in\R^B$ such that $\bm{\Gamma}_i \bm{\beta}>\b 0$ for all $i\in \mathcal{I}$, (iv) $\left\{\tilde{\b x}_{i,m}\right\}_{m\in [M_i]}\subseteq\Kcons_i,\, \forall i\in \mathcal{I}_{\text{SOC}}$,  and (v) $\Lcal( \bar{\b f},\cdot)$ is $L_b-$Lipschitz continuous on 
		$\BB_{\left\|\cdot\right\|_2}\left(\bar{\b b},\eta_\infty c_f\big\| \bm{\beta}\big\|_2 \right)$ where 
		$c_f:=\frac{\max_{i\in [I]}  \left\|\bar{\b f} -\b f_{0,i}\right\|_K}{\min_{i\in [I],p\in P_i}(\bm{\Gamma}_i \bm{\beta})_p}$ and
		\begin{align}\label{def:eta_infty}
			\eta_\infty &:=\max\left(\max_{i\in\mathcal{I}_{\text{SOC}},\, m\in[M_i]} \eta_{i,m,P_i}, \max_{i\in \mathcal{I}_{\Omega},\,m\in[M_i]} \diam(\Omega_{i,m})\right).
		\end{align}
		Then
		\begin{align}\label{ineq_apriori}
			0\le v_{\app}  - \bar{v} \le L_b c_f\big\| \bm{\beta}\big\|_2  \eta_\infty.	
		\end{align}
	\end{proposition}
	\noindent\emph{Proof idea.} The assumption that $\FK=\hatFK$ and $\dom(\Lcal(\bar{\b f},\cdot))=\RR^B$ ensures that $\bar{\b f}\in \hatFK$, and $(\bar{\b f},\bar{\b b}+\bm{\tilde \beta})$ is hence admissible for \eqref{opt_cons_SOC} for a well-chosen $\bm{\tilde \beta}$. We then have that $\Lcal(\bar{\b f},\bar{\b b})\le \Lcal\left(\bar{\b f}_{\app}, \bar{\b b}_{\app}\right)\le \Lcal(\bar{\b f},\bar{\b b}+\bm{\tilde \beta})\le \Lcal(\bar{\b f},\bar{\b b}) + L_b \|\bm{\tilde \beta}\|_2$.
	\vspace{0.1cm}
	
	\noindent{\tb{Remarks} (Proposition~\ref{thm:apriori_bound})}:
	\begin{itemize}[labelindent=0em,leftmargin=1em,topsep=0cm,partopsep=0cm,parsep=0cm,itemsep=2mm]
 	    \item If $\Lcal$ is also strongly convex and w-l.s.c., then 
	    one can insert \eqref{ineq_apriori} to \eqref{ineq_aposteriori} with $v_{\text{relax}}=\bar{v}$, by taking as relaxation the original problem itself, making more explicit the role of $\eta_\infty$.
    	\item Relating the a priori bound to the fill distance: The constant $\eta_\infty$ can be seen as a bound on the Hausdorff distance between $\Phi(\Kcons)$ and its covering in $\FK$ as depicted in \Cref{Diagram_covering}. It is also related to the fill distance. Indeed, the fill distance of a family of points $\{\tilde{\b x}_{i,m}\}_{m\in[M]}$ to a compact set $\Kcons_i$ is defined as the largest distance from a point of $\Kcons_i$ to the samples, and we then take the maximum over $i\in[I]$:
	\begin{align}\label{eq:fill_distance}
		h&:=\max_{i\in[I]}\max_{\b x \in \Kcons_i}\min_{m\in[M]}\|\b x-\tilde{\b x}_{i,m}\|_\X.
	\end{align}
	Assume that the functions $\b x\mapsto D_iK(\cdot, \b x)\in\FK$ are $c_K$-Lipschitz. Then for a ball-covering with $\eta_{m,P_i}$ as in \eqref{def:eta_SDP}, $\eta_{i,m,P_i}\le c_K h$. Since $h=\mathcal{O}\left(M^{-1/d}\right)$ one faces a curse of dimensionality issue due to the covering procedure which makes the solution best-suited for smaller scale problems (this can be mitigated through adaptive coverings, see \Cref{sec:covering_algorithms}). This bound \eqref{ineq_apriori} on the approximation error is similar to the one of \citet{muzellec22learning} where, following \citet{rudi20finding}, they get a bound $|v_{\app}  - \bar{v}|\le C_{kSoS}M^{-s/d}$ with $s>d/2$,  though with a constant $C_{kSoS}$ that is exponential in $d$ and for a $(s+2)$-smooth Sobolev-like kernel $K$. The important difference is that our bound results from a tightening, so it is an upper bound, while their bound does not fall within the tightening/relaxation ordering. Furthermore we do not assume any extra smoothness of the kernel. Besides, \citet{muzellec22learning} do not guarantee that the constraints are satisfied for a given iterate, only asymptotically. Their analysis proceeds from scattering inequalities in approximation theory turning inequalities into equalities. They thus lose the sparsity of coefficients of the solution, which is a property induced by inequalities and  classically exploited with support vector machines (the Lagrange multiplier vanishing when the inequality is inactive).\\
	\end{itemize}
	A natural choice of relaxation of \eqref{opt-cons} is to set $\hatFK=\FK$ and to discretize \eqref{eq:pointwise-constraints} at a finite number of points, in which case the classical representer theorem holds. The first part of our next result shows that the tightened task \eqref{opt_cons_SOC}---assuming that there exists a minimizer of  \eqref{opt_cons_SOC}---can also be reduced to a finite-dimensional optimization problem. Its second part guarantees existence owing to Theorem~\ref{thm:certificate}. These results hold  under mild conditions for objectives based on a finite number of samples.
	\color{black}
	\begin{proposition}[Representer theorem for \eqref{opt_cons_SOC}, $\hatFK = \FK$]\label{prop:reproducing}
		Let $\X$, $K$, $\mathcal{I}_{\text{SOC}}$ and $\mathcal{I}_{\Omega}$ be defined according to Theorem~\ref{thm:certificate} with $\hatFK = \FK$.\footnote{Recall that (i) the covering points  in $C^i_{P_i,SOC}$ are $\left\{\tilde{\b x}_{i,m}\right\}_{m\in [M_i]}$ ($i \in \mathcal{I}_{\text{SOC}}$), (ii) the centers and the normal vectors of $\Omega_{i,m}$ ($i \in \mathcal{I}_{\Omega}$, $m\in [M_i]$) associated to $C^i_{1,\Omega}$ are $\{\b c_{i,m,j}\}_{j\in[J_{B,i,m}]}$  and $\{\b v_{i,m,j}\}_{j\in[J_{B,i,m}]}$ respectively, (iii) the affine biases are $\b f_{0,i}$ ($i\in [I]$) as in \eqref{def_mixded_constraint}.} Assume there exists a minimizer to \eqref{opt_cons_SOC}, and that for fixed samples $S=(\b x_n)_{n\in [N]} \subset \X$ the objective writes as
		\begin{align}
			\LcalS(\b f,\b b) & = L\left(\b b,\left(\left(D^0_{n,j}(\b f)(\b x_n)\right)_{j\in J_n}\right)_{n\in [N]}\right) + R\left(\|\b f\|_K\right), \label{eq:obj-function_new}
		\end{align}
		with some linear differential operators\footnote{The number of differential operators ($\#J_n$) associated to a given sample $\b x_n$ can differ for different $n$-s.} $(D^0_{n,j})_{j \in J_n} \subset O_{Q,s}$, loss $L:\R^B \times \R^{\sum_{n\in [N]}\#J_n} \rightarrow \R \cup \{\infty\}$, and non-decreasing regularizer $R: \Rnn \rightarrow \R$. Then there also exists a minimizer $\bar{\b f}_{\app}$   such that
		\begin{align}
			\bar{\b f}_{\app} & = \underbrace{\sum_{n \in [N]} \sum_{j \in J_n} a_{L,n,j} D_{n,j}^0K(\cdot,\b x_n)}_{\text{input samples $\b x_n$}} + \underbrace{\sum_{i \in \mathcal{I}_{\text{SOC}}} \sum_{p_1,p_2\in[P_i]}\sum_{m\in [M_i]} a_{S,i,p_1,p_2} D^i_{p_1,p_2}K(\cdot,\tilde{\b x}_{i,m})}_{\text{virtual points $\tilde{\b x}_{i,m}$ in $C^i_{P_i,SOC}$}} \nonumber\\
			& + \underbrace{\sum_{i \in \mathcal{I}_{\Omega}} \sum_{m\in [M_i]} \left[\sum_{j\in[J_{B,i,m}]} a_{B,i,m,j} \b c_{i,m,j} + \sum_{j\in[J_{H,i,m}]} a_{H,i,m,j} \b v_{i,m,j}\right]}_{\text{centers $\b c_{i,m,j}$ and normal vectors $\b v_{i,m,j}$ of $\Omega_{i,m}$ associated to $C^i_{1,\Omega}$}}+\underbrace{\sum_{i \in [I]} a_{0,i} \b f_{0,i}}_{\text{affine biases $\b f_{0,i}$}},
			\label{eq:f-reprT}
		\end{align}
		with some coefficients   $\{a_{L,n,j}\}_{n\in[N],\, j\in J_n}$, $\{a_{S,i,\,p_1,\,p_2}\}_{i\in \mathcal{I}_{\text{SOC}},\, p_1,\, p_2 \in [P_i]}$, $\{a_{B,i,m,j}\}_{i\in \mathcal{I}_{\Omega},\, m\in [M_i],\, j \in [J_{B,i,m}]}$,\linebreak $\{a_{H,i,m,j}\}_{i\in \mathcal{I}_{\Omega},\, m\in [M_i],\, j \in [J_{H,i,m}]}$,  $\{a_{0,i}\}_{i\in [I]} \subset \R$, where the functions $D_{n,j}^0 K(\cdot,\b x_n)$ and $D^i_{p_1,p_2}K(\cdot,\b x_{i,m})$ are defined as in \eqref{eq:DK-def}.
		
		\tb{Existence}: Furthermore $\LcalS + \chi_{C_{\app}}$ is weakly lower semi-continuous and coercive provided that: (i) $R$ satisfies $\lim_{z\rightarrow \infty}R(z)=\infty$, (ii) $L$ is ``uniformly'' coercive in $\b b$, i.e.\ $\lim_{\|\b b\|_2\rightarrow \infty} \inf_{\b y \in \BB_{\left\|\cdot\right\|_2}(\b 0,r)} L\left(\b b,\b y \right)=\infty$ for any $r>0$, (iii) $L$ is lower bounded over $C_{\app}$, (iv) the functions $L$ and $R$ are lower semi-continuous.
	\end{proposition}	
	
	\noindent\tb{Remarks:}
	\begin{itemize}[labelindent=0em,leftmargin=1em,topsep=0cm,partopsep=0cm,parsep=0cm,itemsep=2mm]
		\item Existence for our examples: All the examples provided at the end of Section~\ref{sec:problem-formulation} satisfy the conditions of our existence result. For instance, for the JQR problem, $R$ is quadratic, $L$ is continuous, nonnegative and $\inf_{\b y \in \BB_{\left\|\cdot\right\|_2}(\b 0,r)}L(\b b, \b y)/\|\b b\|_2\xrightarrow{\left\|\b b\right\|_2\rightarrow \infty}\infty$ for any $r>0$.
		\item Representer theorem $\Rightarrow$ finite-dimensional optimization task: Using the parameterization of $\bar{\b f}_{\app}$ in \eqref{eq:f-reprT} with the reproducing property (Lemma~\ref{lemma:reproducing}), the finite-dimensional optimization problem over the coefficients of \eqref{eq:f-reprT} immediately follows. Such a reformulation was exemplified by \citet{aubin20hard} for $Q=P=1$. \vspace{0.1cm}
	\end{itemize}
	
	In the next section, we present the ``soap bubble'' algorithm which is capable of achieving convergence \emph{without} having to refine the covering everywhere.
	
	\section{Adaptive Covering Algorithm of Compact Sets in RKHSs}\label{sec:covering_algorithms}
	In this section we present an adaptive approach for the solution of \eqref{opt-cons}, the soap bubble algorithm which provides a non-uniform covering relying on the objective $\Lcal$. The rationale behind this algorithm is to avoid (i) applying a uniformly refined covering and (ii) tightening \eqref{opt-cons} independently of $\Lcal$.  Instead, the soap bubble algorithm starts from a coarse covering (which allows faster computation), and then it gradually refines the covering where the constraints are saturated. It is moreover well-suited for ``warm starting``, i.e.\ initializing at the previous iterate, when performing the iterations.
	
	Throughout this section we assume to have access to some covering oracles: Alg.~\ref{alg:covering} and Alg.~\ref{alg:covering_Omega}. 	The first algorithm operates in $\X$, and for any compact set $\Kcons\subseteq \X$ and radius $\delta_{\text{max}}$ it outputs a covering of $\Kcons$ with balls of radius at most $\delta_{\text{max}}$. The second one is performed over
	$\FK$, and  for any compact set $\Kcons\subseteq \X$ and diameter $d_{\text{max}}$ it outputs a covering of $\bm{\Phi}_D(\Kcons)$ with sets $\bar{\Omega}_m$ of diameter at most $d_{\text{max}}$ where $\Omega_m$-s are of the form \eqref{eq:Omega_m}.

	The soap bubble algorithm iterates between solving a tightened optimization problem given a covering of $\bm{\Phi}_D(\Kcons)$  and refining the covering by a factor of $\gamma$ for the covering subsets in $\FK$ which saturate the constraints. The resulting algorithm (Alg.~\ref{alg:soap_Omega}) is instantiated in the framework of Theorem~\ref{thm:inclusion} with sets $\bar{\Omega}_m$ and using the covering oracle  Alg.~\ref{alg:covering_Omega}. The method writes as Alg.~\ref{alg:soap} in the framework of Theorem~\ref{thm:SDP} with ball-coverings and using the covering oracle  Alg.~\ref{alg:covering}.
	
	\noindent \tb{Remark:} For $P=I=1$ and ball covering $\bar{\Omega}^{(k)}_m=\BB_{K}\left(DK\left(\cdot,\tilde{\b x}_{m}^{(k)}\right),\eta_{m,1}^{(k)}\right)$, Alg.~\ref{alg:soap_Omega} and Alg.~\ref{alg:soap} coincide. In this case, saturating the constraints at the $k^{th}$ iteration corresponds to being tangent to the affine hyperplane $H_K\left(\b f^{(k)} - \b f_0, b_0 - \bm{\Gamma}\b  b^{(k)}\right)$. For an illustration, see Fig.~\ref{Diagram_SoapBubble}.
	
	\begin{algorithm}
		\caption{Ball covering in $\X$ (shortly Cover)}
		\label{alg:covering} 
		\begin{algorithmic}
			\STATE \tb{Input:} Compact set $\Kcons\subseteq \X$,  maximal covering radius $\delta_{\text{max}}>0$, norm $\left\|\cdot\right\|_\X$.
			\STATE \tb{Output:} Covering $\left(\tilde{\b x}_m, \b \delta_m\right)_{m \in [M]}$ s.t.\ $\Kcons \subseteq \cup_{m\in [M]}\BB_{\X}\left( \tilde{\b x}_m,\delta_m\right)$ and $\max_{m\in [M]} \delta_m\le \delta_{\text{max}}$.\footnote{We assume that superfluous covering sets, in other words for which $\BB_{\X}\left( \b x_m,\delta_m\right)\cap \Kcons=\emptyset$ are not generated (a requirement for the proof of Theorem~\ref{thm:soap_bubble}).}
		\end{algorithmic}
	\end{algorithm}
	\begin{algorithm}
		\caption{$\Omega$-covering in $\FK$ (shortly $\Omega$-Cover)}
		\label{alg:covering_Omega} 
		\begin{algorithmic}
			\STATE \tb{Input:} Compact set $\Kcons\subseteq \X$, kernel $K$, differential operator $D$, maximal covering diameter $d_{\text{max}}>0$.
			\STATE \tb{Output:} Covering $\left(\bar{\Omega}_m\right)_{m \in [M]}$ s.t.\ $\bm{\Phi}_D(\Kcons) \subseteq \bar{\Omega}:=\cup_{m\in [M]}\bar{\Omega}_m$ and $\max_{m\in [M]} \diam\left(\bar{\Omega}_m\right)\le d_{\text{max}}$, with $\Omega_m$ of the form \eqref{eq:Omega_m}.
		\end{algorithmic}
	\end{algorithm}
	
	\begin{figure}
		\centering
		\begin{minipage}[c]{.48\linewidth}
			\begin{center}			
				\subfloat[][]{\resizebox{\linewidth}{!}{\label{Diagram_SoapBubble_preBurst}\includegraphics[page={3},keepaspectratio, width=\textwidth]{FIG/Soap_bubble.pdf}}}
			\end{center}
		\end{minipage}
		\begin{minipage}[c]{.48\linewidth}
			\begin{center}			
				\subfloat[][]{\resizebox{\linewidth}{!}{\label{Diagram_SoapBubble_postBurst}\includegraphics[page={4},keepaspectratio, width=\textwidth]{FIG/Soap_bubble.pdf}}}
			\end{center}
		\end{minipage}
		\caption{Illustration of one iteration of the soap bubble algorithm Alg.~\ref{alg:soap_Omega} with ball covering (corresponding to Alg.~\ref{alg:soap} with $P=I=1$).  \protect\subref{Diagram_SoapBubble_preBurst}: After computing the optimal $\left(\b f^{(k)},\b b^{(k)}\right)$ for a given covering at step $k$, the elements  of the covering that are tangent to the hyperplane burst (red). The other elements (blue) are kept for the next iteration. \protect\subref{Diagram_SoapBubble_postBurst}: The elements that are tangent are replaced by a new covering of the subset of $\Phi_{D}(\Kcons)$ that they covered. This covering is chosen such that its radii are smaller by  at least a factor $\gamma$ than the previous radii. The covering at step $k+1$ of $\Phi_{D}(\Kcons)$ is formed by combining the elements untouched at step $k$ with the new elements. These new constraints define a new optimization problem leading to $\left(\b f^{(k+1)},\b b^{(k+1)}\right)$.}
		\label{Diagram_SoapBubble}
	\end{figure}
	\begin{algorithm}
		\caption{Soap Bubble Algorithm with $\Omega$-covering $(I=P=1)$}
		\label{alg:soap_Omega} 
		\begin{algorithmic}
			\STATE \tb{Input:} Compact set $\Kcons\subseteq \X$, kernel $K$, closed subspace $\hatFK \subseteq \FK$, objective $\Lcal$, bias $\b f_{0}$ and $b_{0}$, linear transformation $\bm \Gamma\in\R^{1\times B}$, differential operator $D$, refinement rate $\gamma \in (0,1)$, number of iterations $k_{\text{max}}\in \N$, maximal initial covering diameter $d^{(0)}_{\text{max}}>0$.
			\STATE \tb{Initialization:} $\left(\bar{\Omega}^{(0)}_m\right)_{m\in \left[M^{(0)}\right]}$ := $\Omega$-Cover$\left(\bm{\Phi}_D(\Kcons),d_{\text{max}}^{(0)}\right)$. 
			\FOR{$k=0$ \TO $k_{\text{max}}$}
			\STATE Solve the tightening with $\Omega$-covering $\bar{\Omega}^{(k)}=\bigcup_{m\in \left[M^{(k)}\right]}\bar{\Omega}^{(k)}_m$: 
			
			\begin{align}
				\left(\b f^{(k)},\,\b b^{(k)}\right) &= \argmin_{\b f \in \hatFK, \b b\in \R^B}\Lcal(\b f,\, \b b), \label{opt:soap-Omega}\tag{$\Psc\left(\bar{\Omega}^{(k)}\right)$}\\
				& \quad \text{ s.t. }  \bar{\Omega}^{(k)}_m \subseteq H^{+}_{\K}(\b f - \b f_0,b_{0}-\bm{\Gamma}\b b)\,\, \forall\,m\in \left[M^{(k)}\right].\nonumber
			\end{align} 
			\item Find the indices $\mathcal{I}^{(k)}\subseteq\left[M^{(k)}\right]$ for which the sets intersect the hyperplane:
			\STATE 	
			\vspace{-0.4cm}
			\begin{align*}
				\mathcal{I}^{(k)} & := \left\{m\in \left[M^{(k)}\right]\,:\, \bar{\Omega}^{(k)}_m \cap H_{\K}\left(\b f^{(k)} - \b f_0,b_{0}-\bm{\Gamma}\b b^{(k)}\right) \neq\emptyset\right\}.
			\end{align*}
			\vspace{-0.7cm}
			\STATE The associated $\left(\bar{\Omega}^{(k)}_m\right)_{m\in \mathcal{I}^{(k)}}$ burst and give rise to a finer covering:
			\FOR{$j\in \mathcal{I}^{(k)}$}
			\STATE New $\Omega$-covering: $\left(\bar{\Omega}^{(k+1)}_{j,m}\right)_{m\in\left[M_{j}^{(k+1)}\right]}:=\Omega$-Cover$\left(\bm{\Phi}_D(\Kcons)\cap \bar{\Omega}^{(k)}_j, \gamma\diam\left(\bar{\Omega}^{(k)}_j\right)\right)$
			\ENDFOR
			
			$\left(\bar{\Omega}^{(k+1)}_{m}\right)_{m\in\left[M^{(k+1)}\right]}=\underbrace{\left(\bar{\Omega}^{(k)}_m\right)_{m\in\left[M^{(k)}\right]\setminus \mathcal{I}^{(k)}}}_{\text{non-burst coverings}} \cup \underbrace{\bigcup_{j\in\mathcal{I}^{(k)}} \left(\bar{\Omega}^{(k+1)}_{j,m}\right)_{m\in\left[M_{j}^{(k+1)}\right]}}_{\text{burst $\Rightarrow$ refined coverings}}$
			\ENDFOR
			\STATE \tb{Final estimate:} $\left(\b f^{(k_{\text{max}}+1)}, \b b^{(k_{\text{max}}+1)}\right)$ := solution of $\Psc\left(\bar{\Omega}^{(k_{\text{max}}+1)}\right)$.
		\end{algorithmic}
	\end{algorithm}
	\begin{algorithm}
		\caption{Soap Bubble Algorithm with ball coverings $I\geq 1$, $P_i\geq 1$}
		\label{alg:soap} 
		\begin{algorithmic}
			\STATE \tb{Input:} Compact sets $\{\Kcons_i\}_{i\in [I]}$, norm $\left\|\cdot\right\|_\X$, kernel $K$, closed subspace $\hatFK \subseteq \FK$, objective $\Lcal$, biases $\{\b f_{0,i}\}_{i\in [I]}$ and $\{\b b_{0,i}\}_{i\in [I]}$, linear transformations $\{\bm{\Gamma}_i\}_{i\in [I]}$, differential operators $\{\b D_i\}_{i\in [I]}$, refinement rate $\gamma \in (0,1)$, number of iterations $k_{\text{max}}\in \N$, maximal initial covering radii $\left\{\delta_{i,\text{max}}^{(0)}\right\}_{i\in [I]}$.
			\STATE \tb{Initialization:} $\left(\tilde{\b x}_{i,m}^{(0)},\delta_{i,m}^{(0)}\right)_{m\in \left[M_i^{(0)}\right]}$ := Cover$\left(\Kcons_i,\delta_{i,\text{max}}^{(0)}, \left\|\cdot\right\|_\X\right)$ for all $ i\in [I]$. 
			\FOR{$k=0$ \TO $k_{\text{max}}$}
			\STATE Compute buffers using \eqref{eq:eta_SDP_explicit}: 
			\begin{align*}
				\bm \eta^{(k)}_i :=\left(\eta_{i,m,P_i}^{(k)}\right)_{m\in \left[M_i^{(k)}\right]} = \left(\eta\left(\tilde{\b x}_{i,m}^{(k)},\delta_{i,m}^{(k)};\b D_i\right)\right)_{m\in \left[M_i^{(k)}\right]} \quad  \forall i\in [I].
			\end{align*}
			\STATE Solve the tightening with buffers $\left\{\bm \eta_i^{(k)}\right\}_{i\in [I]}$ and anchors $\left\{\tilde{\b x}_{i,m}^{(k)}\right\}_{i\in [I],\, m\in \left[M_i^{(k)}\right]}$: 
			\vspace*{-0.4cm}
			\begin{align}
				\left(\b f^{(k)},\,\b b^{(k)}\right) &= \argmin_{\b f \in \hatFK, \b b\in\R^B}\Lcal(\b f,\, \b b),\tag{$\Psc^{(k)}$}\\
				& \quad \text{ s.t. }  \eta^{(k)}_{i,m,P_i}\|\b f-\b f_0\|_K \b I_{P_i} \preccurlyeq \b D_i\left(\b f-\b f_{0,i} \right)\left(\tilde{\b x}^{(k)}_{i,m}\right) + \diag\left(\bm{\Gamma}_i\b b- \b b_{0,i}\right)\,\, \nonumber\\
				&\hspace{4.9cm}\forall i\in [I], \forall m\in \left[M_i^{(k)}\right]. \nonumber
			\end{align} 
			\vspace*{-0.4cm}
			\item Find the indices $\mathcal{I}_i^{(k)}\subseteq\left[M_i^{(k)}\right]$ for which the constraints are saturated; the associated $\BB_{K}\left(DK\left(\cdot,\tilde{\b x}_{i,m}^{(k)}\right),\eta_{i,m,P_i}^{(k)}\right)$ balls burst and give rise  to a finer covering:
			\FOR{$i \in [I]$}
			\STATE 	
			\vspace{-0.5cm}
			\begin{align*}
				\mathcal{I}_i^{(k)} & := \left\{m\in \left[M_i^{(k)}\right]\,:\, \eta^{(k)}_{i,m,P_i}\left\|\b f^{(k)}-\b f_{0,i}\right\|_K \b I_{P_i} = \b D_i\left(\b f^{(k)}-\b f_{0,i}\right)\left(\tilde{\b x}^{(k)}_{i,m}\right)\right.\\
				&\hspace{7.6cm}\left. + \diag\left(\bm{\Gamma}_i\b b^{(k)}- \b b_{0,i}\right)\right\}.
			\end{align*}
			\vspace{-0.7cm}
			\STATE Refine the covering on $\mathcal{I}_i^{(k)}$:
			\FOR{$j\in \mathcal{I}_i^{(k)}$}
			\STATE $\delta_{i,j,\text{max}}^{(k+1)}:=$ largest solution of the equation over $\delta$:
			$\eta\left(\tilde{\b x}_{i,j}^{(k)},\delta;\b D_i\right) = \gamma \eta_{i}^{(k)}$ with $\eta$ defined in \eqref{eq:eta-function}. Implied covering in $\X$:
			\begin{align*}
			\left(\tilde{\b x}^{(k+1)}_{i,j,m},\delta^{(k+1)}_{i,j}\right)_{m\in\left[M_{i,j}^{(k+1)}\right]} := \text{Cover}\left(\Kcons\cap\BB_\X\left(\tilde{\b x}^{(k)}_{i,j}, \delta^{(k)}_{i,j}\right),\delta_{i,j,\text{max}}^{(k+1)}, \left\|\cdot\right\|_\X\right)    
			\end{align*}
    		\ENDFOR
    		\begin{align*}
    		    \left(\tilde{\b x}^{(k+1)}_{i,m},\delta_{i,m}^{(k+1)}\right)_{m\in\left[M_i^{(k+1)}\right]} &= \underbrace{\left(\tilde{\b x}^{(k)}_{i,m},\delta_{i,m}^{(k)}\right)_{m\in\left[M_i^{(k)}\right]\setminus \mathcal{I}_i^{(k)}}}_{\text{non-burst coverings}} \cup \underbrace{\bigcup_{j\in\mathcal{I}_i^{(k)}} \left(\tilde{\b x}^{(k+1)}_{i,j,m},\delta^{(k+1)}_{i,j}\right)_{m\in\left[M_{i,j}^{(k+1)}\right]}}_{\text{burst $\Rightarrow$ refined coverings}}
    		\end{align*}			
			\vspace*{-0.9cm}
			\ENDFOR			
			\ENDFOR
			\STATE \tb{Final estimate:} $\left(\b f^{(k_{\text{max}}+1)}, \b b^{(k_{\text{max}}+1)}\right)$ := solution of $\Psc^{(k_{\text{max}}+1)}$
		\end{algorithmic}
	\end{algorithm}
	Our next result shows the convergence of the soap bubble algorithm when $I=P=1$ for general covering sets of the form \eqref{eq:Omega}-\eqref{eq:Omega_m}.
	\begin{theorem}[Convergence of Alg.~\ref{alg:soap_Omega}]~ \label{thm:soap_bubble}
	Let us consider Alg.~\ref{alg:soap_Omega} relying on $\Omega$-coverings (Alg.~\ref{alg:covering_Omega}) with elements defined as in \eqref{eq:Omega}-\eqref{eq:Omega_m}, in other words, with balls and half-spaces. Let the covering of $\Omega$ generated at the $k^{th}$ iteration be denoted by $\bar{\Omega}^{(k)}$ and the associated tightened optimization problem  by $\Psc\left(\bar{\Omega}^{(k)}\right)$ for $k\in \N$. Assume that $k_{max}=\infty$ in Alg.~\ref{alg:soap_Omega}.
		\begin{enumerate}[labelindent=0em,leftmargin=1em,topsep=0cm,partopsep=0cm,parsep=0cm,itemsep=2mm]
			\item Limit covering: 
			If all the iterates $\left(\b f^{(k)},\b b^{(k)}\right)_{k\in\N}$ of Alg.~\ref{alg:soap_Omega} exist, then the corresponding coverings $\left(\bar{\Omega}^{(k)}\right)_{k\in\N}$ converge in Hausdorff distance to a limit set $\bar{\Omega}^{(\infty)}$ containing $\bm{\Phi}_D(\Kcons)$. Moreover, if $\hatFK=\FK$, the solutions of $\Psc\left(\bar{\Omega}^{(\infty)}\right)$ also solve the original problem.
			\item Convergence of $\left(\b f^{(k)},\b b^{(k)}\right)_{k\in\N}$: 
			Assume that (i) $\Lcal$ is weakly lower semi-continuous and coercive over $\FK\times \R^B$, (ii) there exists an admissible pair $\left(\hat{\b f}, \hat{\b b}\right)$ for $\Psc\left(\bar{\Omega}^{(0)}\right)$, (iii) $\dom(\Lcal(\b f,\cdot))=\RR^B$ and $\Lcal(\b f,\cdot)$ is continuous for all $\b f$ in its domain, and (iv) $\bm{\Gamma}\neq \b 0$ in \eqref{def:mixed-constraint:P=1}. Then the sequence of iterates $\left(\b f^{(k)},\b b^{(k)}\right)_{k\in\N}$ exists and is bounded in $\FK\times\R^B$. Moreover, if $\hatFK=\FK$, every weakly-converging sub-sequence converges to a solution of the original problem. If $\left(\bar{\b f},\bar{\b b}\right)$ is unique, then the iterates $\left(\b f^{(k)},\b b^{(k)}\right)_{k\in\N}$ converge weakly to $\left(\bar{\b f},\bar{\b b}\right)$.
		\end{enumerate}
	\end{theorem}
	
	\noindent\tb{Remark:} The assumptions (i)-(iv) are stronger than that of Theorem \ref{thm:certificate} as instead of having a single problem \eqref{opt_cons_SOC} to solve, we consider a sequence of tasks  $\eqref{opt:soap-Omega}_{k\in\N}$. This requires additional regularity on the objective.
	
	\section{Numerical Experiments} \label{sec:numerical-demos}
	In this section we demonstrate the efficiency of the proposed tightened schemes.\footnote{The code replicating our numerical experiments is available at \url{https://github.com/PCAubin/Handling-Hard-Affine-SDP-Shape-Constraints-in-RKHSs}.} Particularly, we designed the following experiments:
	\begin{itemize}[labelindent=0em,leftmargin=1em,topsep=0.1cm,partopsep=0cm,parsep=0cm,itemsep=2mm]
		\item \tb{Experiment-1}: We show that the soap bubble algorithm (Section~\ref{sec:covering_algorithms}) can be more efficient both in terms of accuracy and of computation time when compared to non-adaptive techniques (Section~\ref{sec:optimization}). We illustrate this result on a 1D-shape optimization problem ($Q=1$) with a single constraint over a large domain. This simple, synthetic example serves the purpose of visualization and better understanding of the methods, thanks to its analytical solution.
		\item \tb{Experiment-2}: In our second application we tackle a linear-quadratic optimal control problem with state constraints. This is a vector-valued example ($Q>1$) where we show how the proposed hard shape-constrained technique enables one to guarantee obstacle avoidance when piloting an underwater vehicle, in contrast to classical discretization-based approaches.
		\item \tb{Experiment-3}: The third experiment is about estimating the end pose of a robotic arm based on the length of the links and the angle of the joints, using noisy observations. This is a vector-valued example ($Q=3$) with constraints on the first derivatives, which also goes beyond the state-of-the-art in terms of the input dimension considered ($d=6$), showing the applicability of our method in moderate dimensions.
		\item \tb{Experiment-4}: Our fourth example pertains to econometrics, the goal being to learn production functions based on only a few samples. This example underlines how shape constraints interpreted as side information can empirically improve generalization properties. In this case the function to be determined is real-valued ($Q=1$) with several shape constraints including an SDP one (joint convexity, $P_i>1$). 
	\end{itemize}
	
	\subsection{Experiment-1: Soap Bubble Algorithm}\label{sec:app:catenary}
	
	\begin{figure*}
		\centering
		\begin{minipage}[c]{.48\linewidth}
			\begin{center}			
				\subfloat[][]{\resizebox{\linewidth}{!}{\label{Catenary_XvsY}\includegraphics{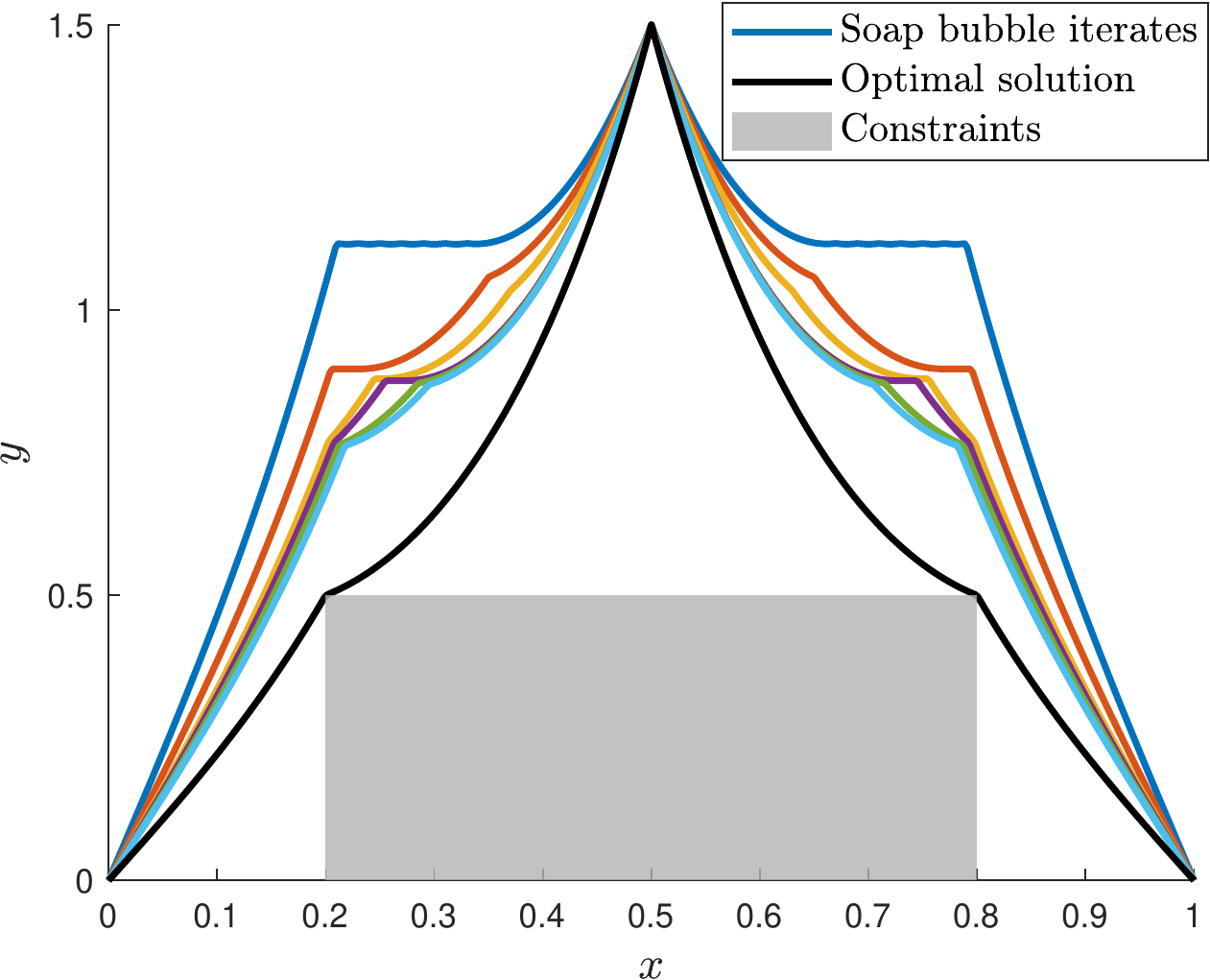}}}
			\end{center}
		\end{minipage}
		\begin{minipage}[c]{.48\linewidth}
			\begin{center}			
				\subfloat[][]{\resizebox{\linewidth}{!}{\label{Catenary_MvsDiffVal}\includegraphics{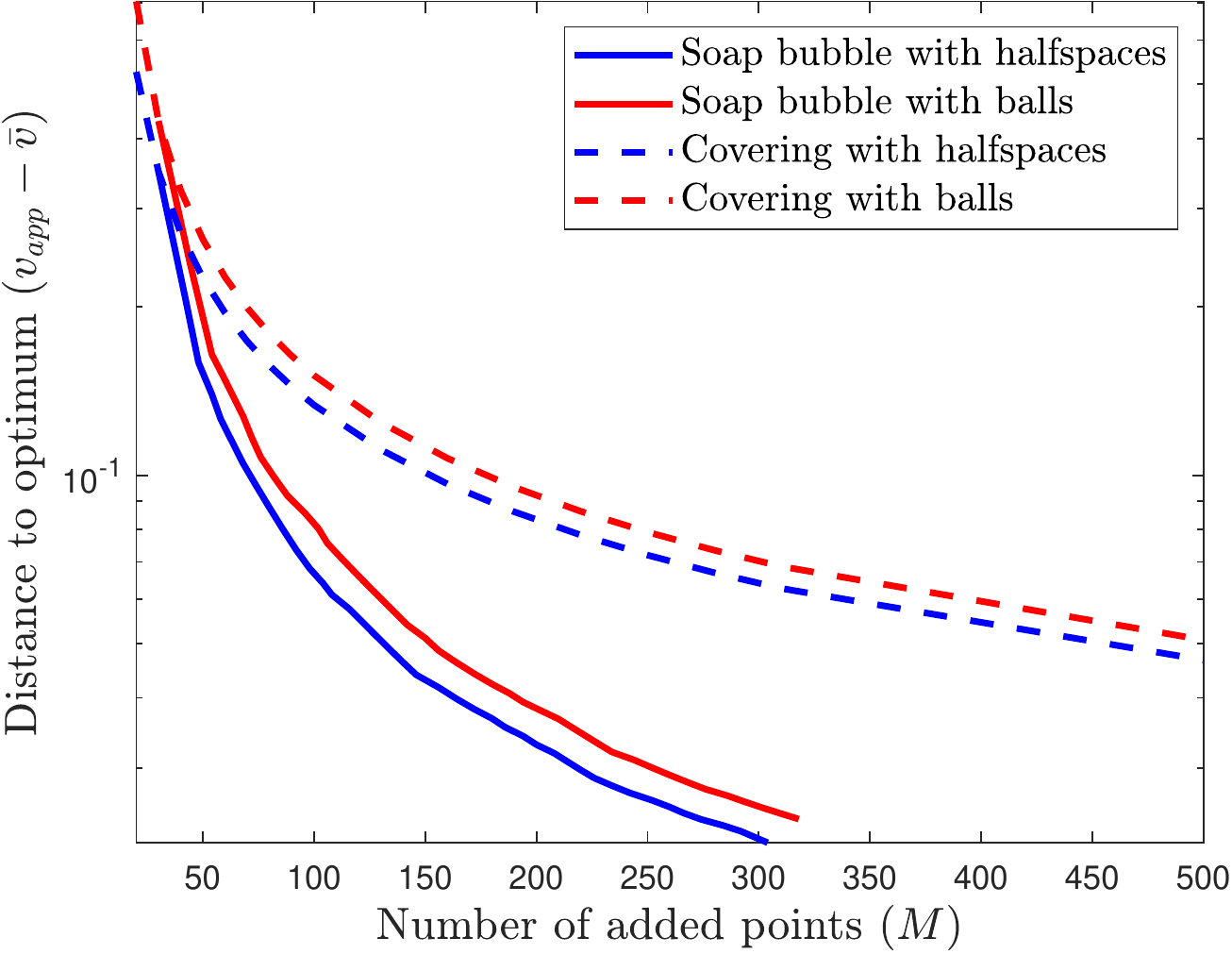}}}
			\end{center}
		\end{minipage}
		
		\begin{minipage}[c]{.48\linewidth}
			\begin{center}			
				\subfloat[][]{\resizebox{\linewidth}{!}{\label{Catenary_XvsIter}\includegraphics{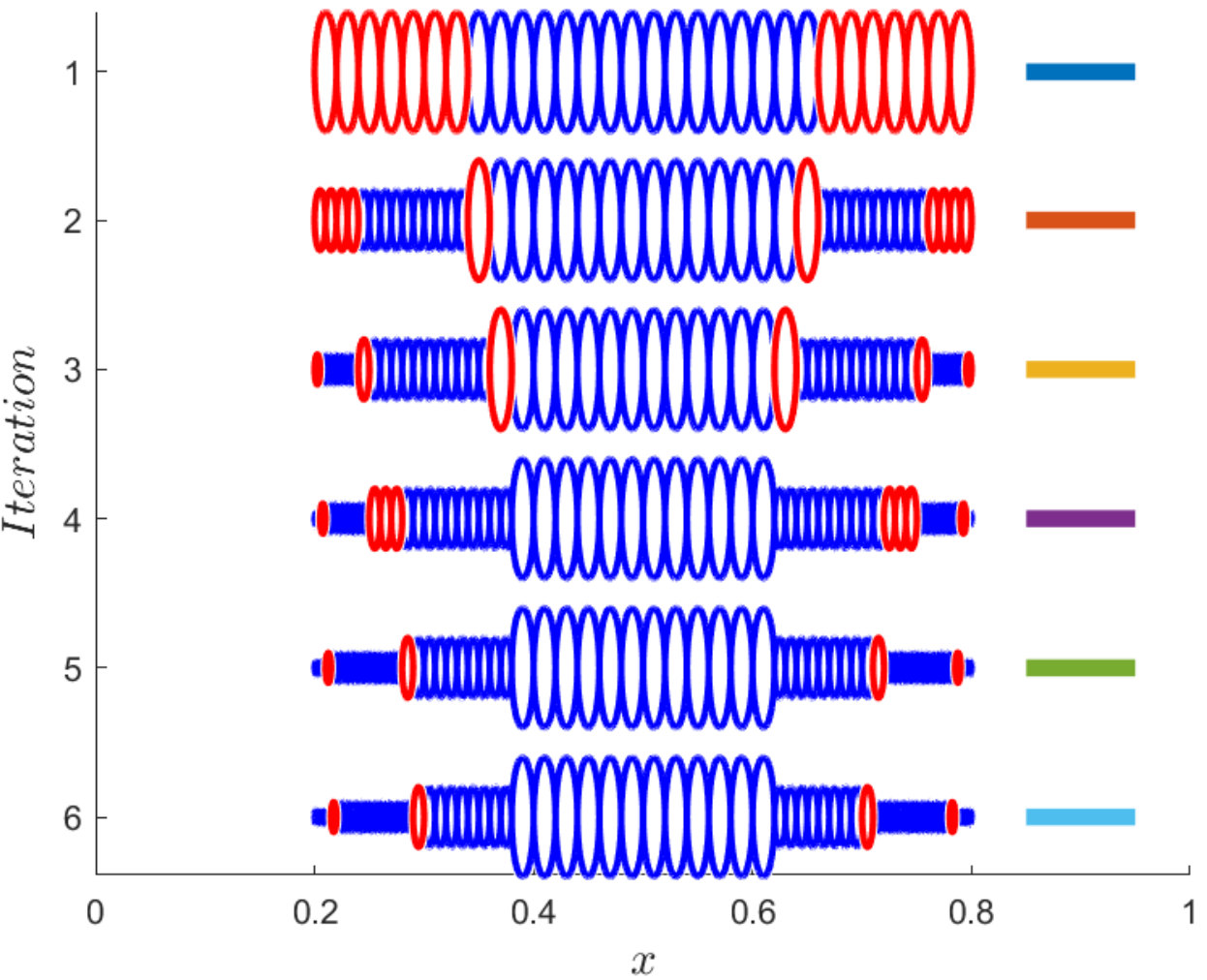}}}
			\end{center}				
		\end{minipage}
		\begin{minipage}[c]{.48\linewidth}
			\begin{center}			
				\subfloat[][]{\resizebox{\linewidth}{!}{\label{Catenary_DiffValvsT}\includegraphics{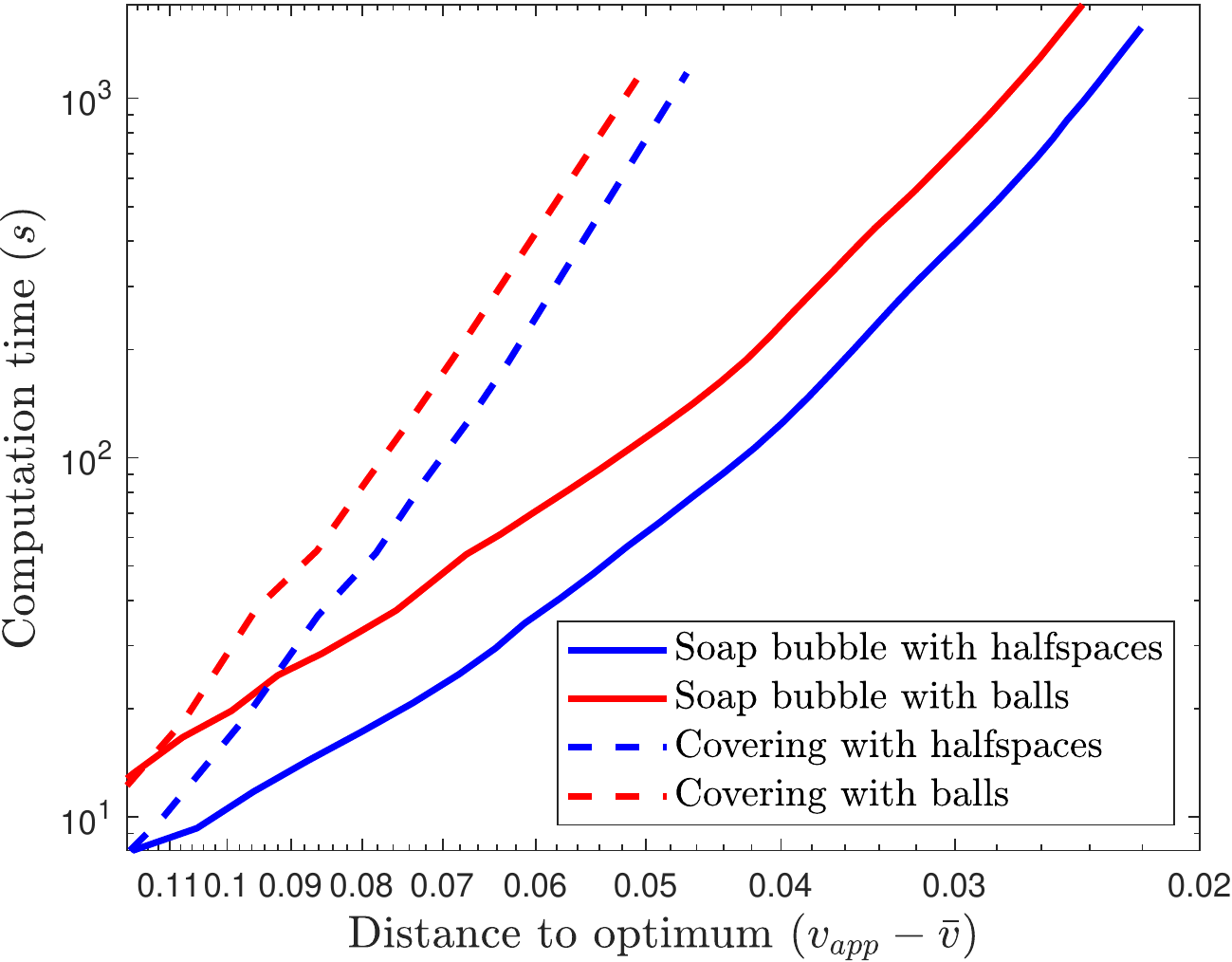}}}
			\end{center}				
		\end{minipage}
		\caption{Illustration of the soap bubble algorithm to optimize the shape of a constrained catenary. Compared techniques: (uniform) covering with balls, covering with balls and half-spaces, soap bubble covering with balls, soap bubble covering with balls and half-spaces. \protect\subref{Catenary_XvsY}: Shape constraint on $[0.2,0.8]$ (grey), optimal solution (black), first 6 iterates of the estimates using the soap bubble technique with balls (coloured curves, first: blue, sixth: cyan). 		 \protect\subref{Catenary_MvsDiffVal}: Performance  as a function of the number of elements in the covering ($M$). \protect\subref{Catenary_XvsIter}: Illustration of bursting in \protect\subref{Catenary_XvsY}; kept balls (blue); burst balls (red).  \protect\subref{Catenary_DiffValvsT}: Computational time as a function of accuracy.}
		\label{fig:catenary}
	\end{figure*}
	\noindent In our \tb{first experiment} we demonstrate the efficiency of the soap bubble algorithm (Alg.~\ref{alg:soap_Omega}) compared to non-adaptive schemes. Our benchmark task corresponds to a shape optimization problem. Particularly, the goal is to determine the deformation of a catenary under its weight. This is equivalent to minimizing the potential energy of its shape. Our domain is $\X=[0,1]$, the form of the catenary is described by a function $f \in \FK$ where
	\begin{align*}
	    \K(x,x')=e^{-\lambda|x-x'|}\quad (\lambda>0)
	\end{align*} is the Laplacian kernel.  This form is constrained at 
	$3$ points $x\in \{0,0.5,1\}$ to be equal to $0$, $1.5$ and $0$ respectively, and the catenary has to be above the value $0.5$ on the whole interval $[0.2,0.8]$. The resulting optimization problem can be expressed as
	\begin{mini*}|s|
		{\substack{f\in \FK}}{\left\|f\right\|_{K}}
		{\label{opt-catenary}}
		{}
		\addConstraint{f(0)}{=0,\, f(0.5)=1.5,\,f(1)=0}
		\addConstraint{0.5}{\le f(x),\,\forall \, x\in [0.2,0.8],}
	\end{mini*}
	with samples $S=((0,0), (0.5,1.5), (1,0))$.
	This task can be written equivalently as
	\begin{mini*}|s|
		{\substack{f\in \FK}}{\Lcal_{S}(f):=\left\|f\right\|_{K} + \chi_{\{0\}}(f(0)) + \chi_{\{1.5\}}(f(0.5)) +  \chi_{\{0\}}(f(1))}
		{\label{opt-catenary}}
		{}
		\addConstraint{0.5}{\le f(x),\,\forall \, x\in [0.2,0.8],}
	\end{mini*}
	which falls within the framework \eqref{opt-cons} with $Q=I=P=1$, $\Kcons=[0.2,0.8]$, $D(f)=f$, $f_0=0$, $b_0 = 0.5$, $\Gamma=0$ and $\Lcal(f, b) = \Lcal(f)+ \chi_{\{0\}}(b)$, in other words the bias term is zero ($b=0$). One of the advantages of this problem is that its solution can be computed analytically for some values of $\lambda$ (for an illustration, see the black solid curve in  Fig.~\ref{fig:catenary}(a)). This optimal solution can be thought of as the tilt of a circus tent, and is used as the ground truth. 
	
	In our experiments we chose the bandwidth parameter to be $\lambda = 5$.  We compared the efficiency (in terms of time and accuracy) of four different covering schemes which we detail in the following.
	\begin{enumerate}[labelindent=0em,leftmargin=1em,topsep=0cm,partopsep=0cm,parsep=0cm,itemsep=1mm]
		\item Covering with balls only: In this case the $M$ points of the covering of $\Kcons$ were equidistant over the interval $\Kcons$, i.e.\ $\tilde{x}_m =0.2+\frac{1}{2M}+(j-1)\frac{1}{M}$ and $\delta_m = \frac{0.8-0.2}{2M}$ with $m\in [M]$. The shape constraint
		$0.5\le f(x)$ for all $x \in \Kcons$ was tightened to the SOC one ($\eta_m\|f\|_K + 0.5 \le f\left(\tilde{x}_{m}\right)$ for all $ m\in[M]$) with $\eta_m=\sqrt{ \left|2- 2\tilde{\rho}_m\right|}=\sqrt{2- 2\tilde{\rho}_m}$ and $\tilde{\rho}_m:=e^{-\lambda\delta_{m}}$ according to \eqref{eq:eta_SDP_simple}. This choice corresponds to the ball covering 
		\begin{align}
			\bm \Phi(\BB_{|\cdot|}(\tilde{ x}_{m},\delta_m))&\subseteq \BB_{\K}\big(\underbrace{K\left(\cdot,\tilde x_m\right)}_{\b c_m}, \underbrace{\eta_m}_{r_m}\big)
			\label{eq:catenary:covering1}
		\end{align}
		in the RKHS $\FK$, with $J_{B,m}=1, J_{H,m}=0$ ($\forall m\in [M]$) in accordance with \eqref{eq:Omega}-\eqref{eq:Omega_m} and \eqref{eq:approach1:ball-only-specialization}. The resulting convex optimization problem was solved directly using the representer theorem (Proposition~\ref{prop:reproducing}).
		\item Covering with balls and half-spaces:  This method corresponds to the coverings \eqref{eq:Omega}-\eqref{eq:Omega_m} with $J_{B,m}=J_{H,m}=1$, as depicted on Fig.~\ref{Diagram_covering}(b).  The rationale behind this scheme is to provide a finer covering compared to the previous one, and thus a more accurate approximation. As mentioned in footnote \ref{footnote:tr-invariant_unitBall}, since for the Laplacian kernel $K(x,x)=1$ for all $x\in \X$, we have that $\bm{\Phi}(\Kcons)\subseteq \BB_{\K}(\b 0, 1)$. Moreover for $x \in \BB_{|\cdot|}(\tilde{ x}_{m},\delta_m)$, $K(\tilde{x}_{m},x)= e^{-\lambda\left|x-\tilde x_m\right|}\ge e^{-\lambda\delta_{m}}=\tilde{\rho}_m$. Hence  $-K(x,\tilde{x}_m) = \left<K(\cdot,x),-K(\cdot,\tilde{x}_m)\right>_K \le -\tilde{\rho}_m$, i.e.\ $K(\cdot,x) \in H^{-}_{\K} \left(-K(\cdot,\tilde{x}_{m}),-\tilde{\rho}_{m}\right)$, consequently 
		\begin{align}
			\bm \Phi(\BB_{|\cdot|}(\tilde{ x}_{m},\delta_m))&\subseteq \BB_{\K}\big(\underbrace{\b 0}_{\b c_m}, \underbrace{1}_{r_m}\big) \cap H^{-}_{\K} \big(\underbrace{-K(\cdot,\tilde{x}_{m})}_{\b v_m},\underbrace{-\tilde{\rho}_{m}}_{\rho_m}\big)
			\label{eq:catenary:covering2}
		\end{align}
		in line with \eqref{eq:construction1:1ball-1hyperplane}. This is indeed a covering at least as tight as \eqref{eq:catenary:covering1}, as, when considering an element $\b g$ in the r.h.s.\ of \eqref{eq:catenary:covering2}, then
		\begin{align*}
			\|\b g - K\left(\cdot,\tilde x_m\right)\|_\K^2=\|\b g\|_K^2+K\left(\tilde x_m\,\tilde x_m\right)-2\langle\b g,K\left(\cdot,\tilde x_m\right)\rangle_\K\le 2-2\tilde{\rho}_{m}=\eta_m^2,
		\end{align*}
		which gives that
		\begin{align*}
			\BB_{\K}\big(\b 0, 1\big) \cap H^{-}_{\K} \big(-K(\cdot,\tilde{x}_{m}),-\tilde{\rho}_{m}\big) \subseteq \BB_{\K}\big(K\left(\cdot,\tilde x_m\right), \eta_m\big).
		\end{align*} 
		The values of $\tilde{x}_m$, $\delta_m$, $\tilde{\rho}_m$ and $\eta_m$ were chosen similarly as in the previous point. 
		
		\item Soap bubble covering with balls only: In contrast to the direct solution with a fine covering, our first soap bubble scheme using balls (Alg.~\ref{alg:soap}) is initialized with a coarser uniform covering with an initial covering radius  $\delta_{\text{max}}^{(0)}=0.01$; the latter results in $M=\frac{0.8-0.2}{2\times 0.01} = 30$ anchor points at the beginning. This initial covering is then iteratively refined in our experiments using a rate $\gamma = 0.8$. The shape constraint were considered to be saturated when the condition $\left|\eta_m\|f\|_k + 0.5 - f(\tilde{x}_{m})\right|\le 10^{-8}$ held, determining the bursting condition of the balls in Alg.~\ref{alg:soap}.
		\item Soap bubble covering with balls and half-spaces: A combination of balls and half spaces were considered as in the second covering scheme, to which the soap bubble algorithm (Alg.~\ref{alg:soap_Omega}) was applied. The initialization was the same as in the third scheme.\\
	\end{enumerate}
	
	Our results are summarized in Fig.~\ref{fig:catenary}. The figure shows that the adaptive soap bubble technique (i)  converges to the optimal solution as the iteration proceeds (in accordance with Theorem~\ref{thm:soap_bubble}; see Fig.~\ref{fig:catenary}(a)) with illustration of the bursts in Fig.~\ref{fig:catenary}(c). (ii) It achieves the same accuracy with smaller number of covering points (Fig.~\ref{fig:catenary}(b)) and faster (Fig.~\ref{fig:catenary}(d)) compared to the non-adaptive schemes. (iii) Considering half-spaces additionally to balls results in a small performance gain. These experiments demonstrate the efficiency of the adaptive soap bubble algorithm in the context of a simple shape optimization problem.
	
	\subsection{Experiment-2: Safety-Critical Control}\label{sec:app:safety-crit-control}
	In our \tb{second experiment} we focus on a constrained path-planning problem. Particularly, in this task the trajectory of an underwater vehicle navigating in a two-dimensional cavern is described by a curve $t \in \Tiv:=[0,1] \mapsto [x(t);z(t)]\in \R^2$ corresponding to its lateral ($x$) and depth ($z$) coordinates at time $t\in \Tiv$. For simplicity, we assume that the lateral component satisfies $x(0)=0$ and $\dot{x}(t)=1$ for all $t\in \Tiv$. In this case,  $x(t)=t$ for all $t \in \Tiv$ and the control problem reduces to that of ensuring that the depth $z(t)$ stays between the floor and ceiling of the cavern ($z(t)\in \left[z_{\text{low}}(t),z_{\text{up}}(t)\right]$ for all $ t\in \Tiv$). We take as initial conditions $z(0)=0$ and $\dot{ z}(0)=0$. By denoting the control with $u\in L^2(\Tiv,\R)$ where $L^2(\Tiv,\R)$ is the set of square-integrable real-valued functions on $\Tiv$, our control task can be formulated as  
	\begin{mini}|s|
		{\substack{u(\cdot) \in L^2(\Tiv, \R)}}{\int_\Tiv |u(t)|^2 \d t}
		{\label{opt-cave} \tag{$\Psc_{\text{cave}}$}}
		{}
		\addConstraint{z(0)}{=0, \quad \dot{ z}(0)=0}
		\addConstraint{\ddot{ z}(t)}{=-  \dot{ z}(t) + u(t), \, \forall \, t \in \Tiv}
		\addConstraint{z_{\text{low}}(t)}{\le z(t) \le z_{\text{up}}(t),\,\forall \, t\in \Tiv.}
	\end{mini}
	The task \eqref{opt-cave} belongs to the class of 
	linearly-constrained linear quadratic regulator problems. As shown by \citet{aubin2020hard_control}, these tasks can be rephrased as a shape-constrained kernel regression for a kernel $K$ defined by the objective and the dynamics. By defining the full state of the vehicle as $\b f(t):=\left[z(t);\dot{ z}(t)\right]\in\R^2$, $\b f$ evolves according to the linear dynamics 
	\begin{align*}
	\dot{\b f}(t) &=\b A \b f(t) + \b B u(t)\in \R^2,&
	\b f(0)&=\b 0,&
	\b A&=\begin{bmatrix}0&1\\ 0 & -1\end{bmatrix}\in \R^{2\times 2},&
	\b B&=\begin{bmatrix}0 \\ 1 \end{bmatrix} \in\R^2.
	\end{align*}
	Using that $\b f(0)=\b 0$ the controlled trajectories $\b f$ belong to a $\R^2$-valued RKHS $\FK$ defined over $\Tiv$ with the matrix-valued kernel\footnote{\label{footnote:control:FK}The Hilbert space $\FK$ corresponding to \eqref{def_K1} is the one of controlled trajectories with zero initial condition ($\b f(0)=\b 0$) such that $\|\b f\|_K=\|u\|_{L^2(\Tiv,\R)}$.} 
	\begin{align}
	K(s,t) &:= \int_{0}^{\min(s,t)}e^{(s-\tau)\b A} \b B\b B^{\top}e^{(t-\tau)\b A^{\top}}  \d \tau, \quad s,t\in \Tiv,\label{def_K1}
	\end{align}
	where $e^{\b M}$ denotes the matrix exponential. 
	With our kernel-based formulation, the problem \eqref{opt-cave} can be rewritten as an optimization problem over full-state trajectories
	\begin{mini*}|s|
		{\substack{\b f=[f_1,f_2]\in \FK}}{\left\|\b f\right\|_{K}^2}{}{}
		\addConstraint{z_{\text{low}}(t)}{\le f_1(t) \le z_{\text{up}}(t),\,\forall \, t\in \Tiv}.
	\end{mini*}
	In our experiment we assume that the given bounds $z_{\text{low}}$ and $z_{\text{up}}$ are piecewise constant: taking a uniform $\delta$-covering $\Tiv =\cup_{m\in [M]}\Tiv_m$ with $\Tiv_m:=[t_{m}-\delta,t_{m}+\delta]$ and $t_{m+1} = t_m +2\delta$ for $m\in [M-1]$, this means that $z_{\text{low}}(t)=z_{\text{low},m}$ for all $t\in \Tiv_m$; similarly $z_{\text{up}}(t)=z_{\text{up},m}$ for all $t\in \Tiv_m$.  Hence, with the piecewise constant assumption, the control task \eqref{opt-cave} reduces to
	\begin{mini*}|s|
		{\substack{\b f\in \FK}}{\left\|\b f\right\|_{K}^2}{}{}
		\addConstraint{z_{\text{low},m}}{\le f_1(t) \le z_{\text{up},m},\,\,\forall \, t\in \Tiv_m,\, \forall m\in [M]}.
	\end{mini*}
	This optimization problem belongs to the family \eqref{opt-cons} with $N=0$, $Q=2$, $P=1$, $D_m(\b f)=f_1$ and $b_{0,m}=z_{\text{low},m}$ for $m\in[M]$ ($z_{\text{low},m} \le f_1(t)$ for $t\in \Tiv_m$ and $m\in [M]$), $D_{M+m}(\b f)=-f_1$ and $b_{0,M+m}=-z_{\text{up},m}$ for $m\in[M]$ ($-z_{\text{up},m} \le -f_1(t)$ for $t\in \Tiv_m$ and $m\in [M]$), $ \b f_{0,m}=\b 0$ and $\Gamma_m = 0$ for $m\in [2M]$, $I=2M$ and $\Bsc = \{0\}$.
	
	In Fig.~\ref{fig:cave} we compare the optimal trajectory obtained with the proposed SOC tightening (using ball covering) to the one derived when applying discretized constraints (formally corresponding to taking $\eta_m=0$). Here the piecewise constant bounds were obtained as piecewise approximations of random functions drawn in a Gaussian RKHS. As illustrated in Fig.~\ref{fig:cave}(a), the vehicle guided with discretized constraints crashes into the blue wall at multiple locations, whereas the trajectory resulting from the SOC-based tightening stays within the bounds at all times. The SOC trajectory can be described as solving a problem where $z_{\text{low},m}$ (resp.\ $z_{\text{up},m}$) was replaced by $z_{\text{low},m}+\eta_{m,1}\left\|\bar{\b f}_{\app}\right\|_{K}$ (resp.\ $z_{\text{up},m}-\eta_{m,1}\left\|\bar{\b f}_{\app}\right\|_{K}$). This acts as a supplementary buffer which we illustrate in  Fig.~\ref{fig:cave}(b) (green solid line). Even though the SOC trajectory intersects the green boundary, the buffer $\eta_{m,1}\left\|\bar{\b f}_{\app}\right\|_{K}$ is guaranteed to be large enough for the SOC trajectory to never collide with the blue boundary. This experiment demonstrates the efficiency of the SOC approach in a safety-critical application where the constraints have to be met at all times.
	\begin{figure*}
		\centering
		\begin{minipage}[c]{.48\textwidth}
			\begin{center}			
				\subfloat[][]{\resizebox{\textwidth}{!}{\label{fig:cave1}\includegraphics{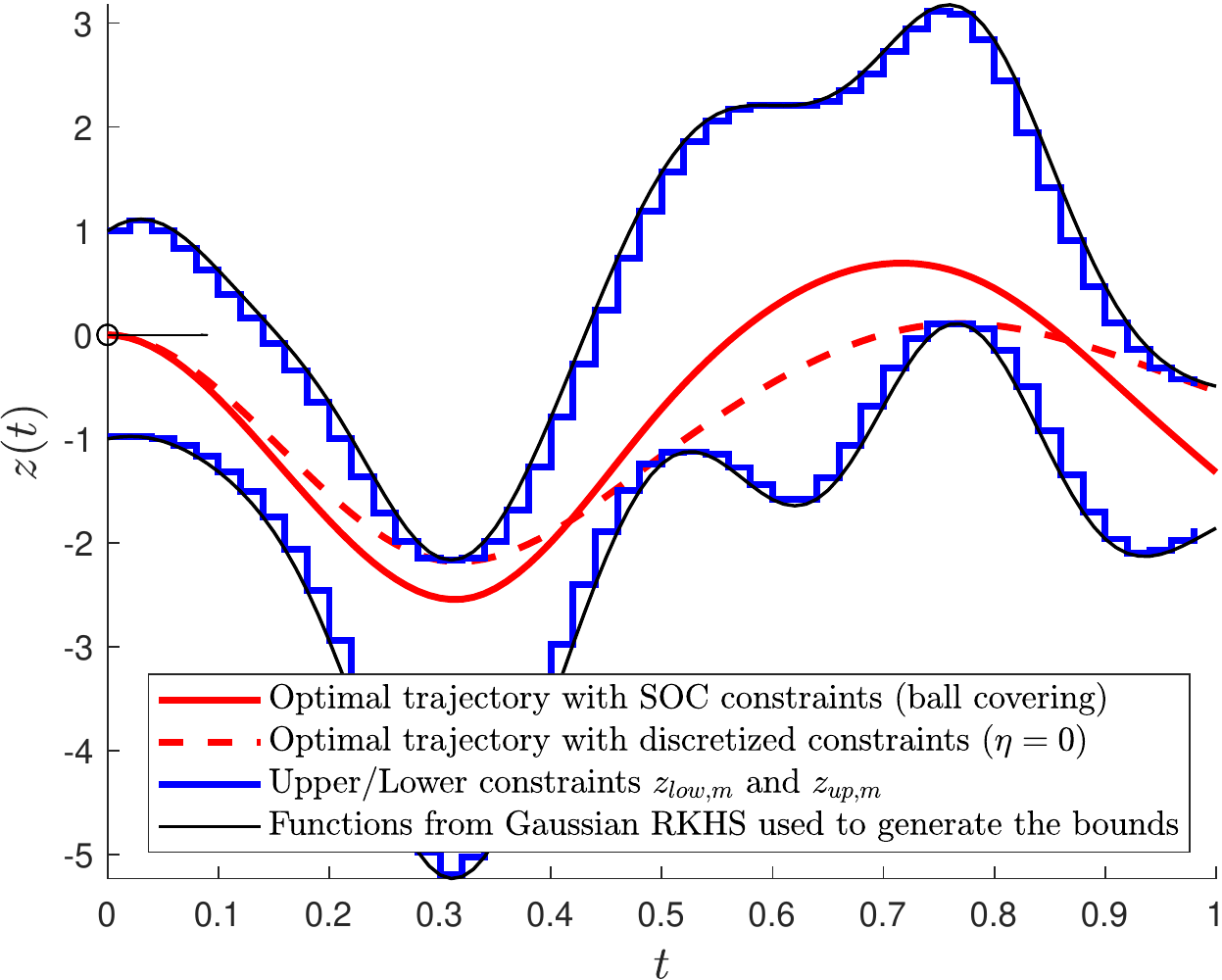}}}
			\end{center}
		\end{minipage}
		\begin{minipage}[c]{.48\textwidth}
			\begin{center}			
				\subfloat[][]{\resizebox{\textwidth}{!}{\label{fig:cave2}\includegraphics{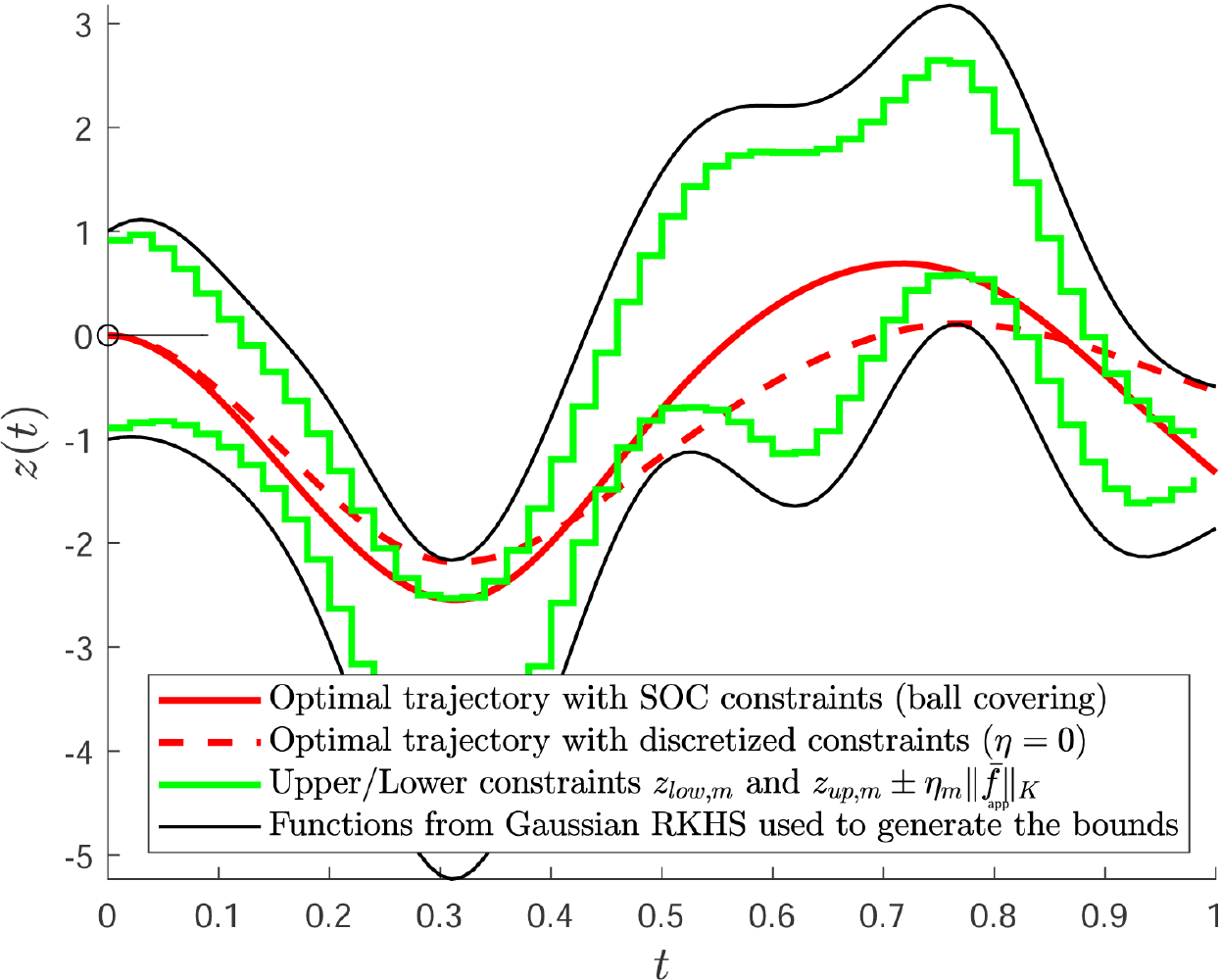}}}
			\end{center}
		\end{minipage}
		\caption{Illustration of the optimal control problem \eqref{opt-cave} of  piloting a vehicle staying between the ceiling ($z_{\text{up}}$) and the floor ($z_{\text{low}}$) of a cavern. Red solid line: SOC-based approach. Red dashed line: solution based on a discretization (formally setting $\eta=0$). Blue solid lines: constraints $(z_{\text{low},m}$ and $z_{\text{up},m})$. Black solid lines: functions used to generate the constraints.
			Green solid lines: constraints with buffer $\pm\eta_{m,1}\left\|\bar{\b f}_{\app}\right\|_{K}$.}
		\label{fig:cave}
	\end{figure*}
	
	\noindent\tb{Remark} (encoding of the bounds $z_{\text{low}}$ and $z_{\text{high}}$): In this control application we assumed that the prescribed bounds are piecewise constant and we generated them using functions which do not necessarily belong to $\FK$. If one faces instead a boundary $z_{\text{low}}$ (resp.\ $z_{\text{high}}$) which can be written as $\b e_1^\top \b z_{\text{low}}$ for some $\b z_{\text{low}}\in\FK$, then it could be treated as a bias $\b f_{0,i}=\b z_{\text{low}}$ (resp.\ $\b z_{\text{high}}$). 
	While this would reduce the number of shape constraints from $I=2M$ to $I=2$, our current choice allows us to investigate the efficiency of the proposed approach in a complementary setting. Indeed, in contrast to the considered shape optimization task with one shape constraint ($I=1$) on a large $\Kcons$ which is refined by the soap bubble algorithm, the path-planning task involves $I=2M$ constraints on an already refined grid.

	\subsection{Experiment-3: Estimation of Robotic Arm Position}\label{sec:app:robotic-arm}
	In our \tb{third experiment} we consider a robotic arm with $\NS\in \N^*$ segments moving in a two-dimensional plane for which we want to estimate the ``tool tip pose'', i.e.\ its 2D-position and its 1D-orientation, depending on the length of the links $[L_i]_{i\in\left[\NS\right]}\in \R^{\NS}$ and angle of the joints $[\theta_i]_{i\in\left[\NS\right]}\in \R^{\NS}$. This means that 	the input is $\b x=[L_1;\dots,L_{\NS};\theta_1;\dots;\theta_{\NS}]\in\R^{d}$ with $d=2\NS$. We follow the experimental protocol of \citet{agrell19gaussian} where the author considered a $4$-dimensional input ($\NS=2$) and we also extend it to $6$-dimensional input ($\NS=3$). Since \citet{agrell19gaussian} focuses on Gaussian processes (a Bayesian estimate), we chose to compare our method against the closer, frequentist and recent, kernel sum-of-squares technique (kSoS; \citealt{muzellec22learning}). In this experiment the underlying ``tool tip pose'' function---which we are aiming to estimate---takes the form 
	\begin{equation}\label{eq:robot-arm-function}
		\b f^{ref}(\b x)=\Bigg[\sum_{i \in \left[\NS\right]}L_i \underbrace{\cos\Bigg(2\pi\sum_{j=1}^{i}\theta_j\Bigg)}_{=:c_i^1(\b x)}; \sum_{i \in \left[\NS\right]}L_i \underbrace{\sin\Bigg(2\pi\sum_{j=1}^{i}\theta_j\Bigg)}_{=:c_i^2(\b x)};\sin\Bigg(2\pi\sum_{j=1}^{\NS}\theta_j\Bigg)\Bigg]\in\R^3.
	\end{equation}
	The observations $S=(\b x_n,\b y_n)_{n\in[N]}$ ($N=40$) are noisy measurements of the functional relation \eqref{eq:robot-arm-function}: 
	\begin{align*}
	  \b y_n=\b f^{ref}(\b x_n)+\bm{\epsilon}_n   
	\end{align*}
    with $(\bm{\epsilon}_n)_{n\in [N]}\stackrel{\text{i.i.d.}}{\sim}\mathcal{N}\left(\b 0_3,0.2^2\b I_3\right)$, and inputs $\b x_n$  generated according to Latin hypercube sampling of $\X=[0,1]^{d}$. As the output values are clearly not independent, we approximate the $\b x\mapsto \b y$ relation using a vRKHS associated to a decomposable matrix-valued kernel $K(\b x, \b x')=k(\b x, \b x')\bm{\Sigma}\in\R^{3\times 3}$, where $\bm{\Sigma}\in \R^{3\times 3}$ is the covariance matrix of the outputs of \eqref{eq:robot-arm-function} estimated over $1000$ samples. The kernel $k$ was chosen to be the Gaussian:
    \begin{align*}
     k(\b x, \b x')=e^{-(\b x-\b x')^\top \diag\left(\left(\frac{1}{2\sigma_i^2}\right)_{i\in [d]}\right)(\b x-\b x')}.   
    \end{align*}
   The objective function is a regularized empirical mean square error
	\begin{equation}\label{eq:robot-arm-objective}
		\Lcal_{S}(\b f)=\frac{1}{N}\sum_{n\in[N]} \|\b y_n-\b f(\b x_n)\|_{2}^2+\lambda \|\b f\|_K^2.
	\end{equation}
	 Following \citet{agrell19gaussian}, we add some extra side information, assuming we know whether or not the arm will move further away from the x-axis or y-axis when changing the link lengths, given any joint configuration. The considered constraints are thus 
	 \begin{align*}
	  (\partial_i f^{ref}_l(\b x))_l (\partial_i f_l(\b x))_l\ge 0 \text{ for $l\in[2]$, $i\in \left[\NS\right]$, }    
	 \end{align*}
	 expressing that the estimate and true derivatives point in the same direction component-wise. By \eqref{eq:robot-arm-function}, the linearity w.r.t.\ $L_i$ entails that $\partial_i f^{ref}_l(\b x)=c_i^l(\b x)$. Consequently we consider five constraints: the original one \eqref{eq:robot-arm-constraints-orig}; its relaxation through discretization \eqref{eq:robot-arm-constraints-disc}; two SOC tightenings, obtained through a ball covering \eqref{eq:robot-arm-constraints-SOCball}, and ball and hyperplanes \eqref{eq:robot-arm-constraints-SOChyp}, with notations consistent with those of \Cref{sec:constr1} and \Cref{sec:app:catenary}; and finally a kSoS approximation \eqref{eq:robot-arm-constraints-kSoS} as per \citet{muzellec22learning} with an extra positive semidefinite matrix-valued variable $\b A=[a_{m_1,m_2}]_{m_1,m_2\in [M]}\succcurlyeq \b 0$. These constraints are as follows:
	\begin{align}\label{eq:robot-arm-constraints-orig}
		c_i^l(\b x) \partial_i f_l(\b x) &\ge 0,\, &&\forall \, x\in\X, \\
		\label{eq:robot-arm-constraints-disc}
		c_i^l(\tilde{\b x}_{m})\partial_i f_l(\tilde{\b x}_{m}) &\ge 0,\, &&\forall \, m\in[M],\\
		\label{eq:robot-arm-constraints-SOCball}
		c_i^l(\tilde{\b x}_{m}) \partial_i f_l(\tilde{\b x}_{m}) &\ge \eta_i\|\b f\|_K,\, &&\forall \, m\in[M],\\
		\label{eq:robot-arm-constraints-SOChyp}
		\exists\, \xi_m \ge 0, \,\, \xi_m \tilde{\rho}_i^l &\ge r_i^l \|\b f - \xi_m \partial_{i,2}K(\cdot,\b x_{m})\b e_l\|_K,\, &&\forall \, m\in[M],\\
		\label{eq:robot-arm-constraints-kSoS}
		c_i^l(\tilde{\b x}_{m}) \partial_i f_l(\tilde{\b x}_{m}) & = \sum_{m_1,m_2\in [M]} a_{m_1,m_2} k_{SoS}(\tilde{\b x}_{m},\tilde{\b x}_{m_1}) k_{SoS}(\tilde{\b x}_{m},\tilde{\b x}_{m_2}),\, &&\forall \, m\in[M].
	\end{align}
	Notice that the shape constraint \eqref{eq:robot-arm-constraints-disc} goes slightly beyond (and hence demonstrates the robustness of our approach) the analyzed affine SDP constraints on function derivatives \eqref{def:mixed-constraint:P>=1}  as $c_i^l(\b x)$ is $\b x$-dependent.
	We consider anchor points $\tilde{\b x}_{m}\in \R^d$ belonging to regular grids with varying stepsize $\Delta x$. As the kernel $K$ is translation invariant, the coefficients $\eta_{i,m},\rho_{i,m},r_{i,m}$ of the SOC methods do not depend on the samples $\tilde{\b x}_{m}$, and we can thus remove the subscript $m$. However this experiment is especially challenging for tightenings. While $\b f = \b 0$ is always an admissible solution, the functions $c_i^l(\cdot)$ frequently change signs since they are either $\sin$ or $\cos$ function; therefore,  tightening the constraint could force the function derivative to be both non-negative and non-positive on some subset. To mitigate this difficulty, we enforce the SOC constraints \eqref{eq:robot-arm-constraints-SOCball}-\eqref{eq:robot-arm-constraints-SOChyp} only on Euclidean balls $\BB_{\X}(\tilde{\b x}_{m},\delta)$ with $\delta=\frac{1}{100}\Delta x$, thus only partially covering the set $\X$. As the whole set $\X$ is not covered, the soap bubble algorithm is not applicable. We also remove a few points for which $\left|c_i^l(\tilde{\b x}_{m})\right|<10^{-1}$ to avoid numerical instabilities. Using a similar derivation as in Section~\ref{sec:app:catenary}, 
	one has $\tilde{\rho}_i^l=\b e_l^\top \bm{\Sigma} \b e_l \min_{\b y\in \BB_{\X}(\b 0,\delta)} \partial_{i,1}\partial_{i,2}k(\b 0, \b y)$, $r_l=\sqrt{\b e_l^\top \Sigma \b e_l \partial_{i,1}\partial_{i,2}k(\b 0, \b 0)}$ and $\eta_i^l=\sqrt{2\left((r_l)^2-\tilde{\rho}_i^l\right)}$, where $\tilde{\rho}_i^l$ was estimated by taking the minimum over $\{\b y_i\}_{i \in [1000]}$ uniformly drawn samples  in the Euclidean ball $\BB_{\X}(\b 0,\delta)$. The hyperparameters $(\sigma_i)_{i\in[d]}$ and the regularization $\lambda>0$ were optimized using 5-fold cross-validation. 
	
	We compared our method with \eqref{eq:robot-arm-constraints-kSoS} obtained from the kSoS approach \citep{muzellec22learning},  where the auxiliary kernel is a real-valued Cauchy kernel $k_{SoS}(\b x,\b x')=1/\left(1+\|\b x-\b x'\|_{2}^2/\sigma_{kSoS}^2\right)$ with $\sigma_{kSoS}=0.2$. We also tested the Gaussian kernel for kSoS but it gave slightly inferior results.
	For this kSoS approach, one has to add a term $\lambda_{kSoS}\tr(\b A)$, where $\tr(\cdot)$ denotes trace, to the objective \eqref{eq:robot-arm-objective} to penalize $\b A$; $\lambda_{kSoS}=10^{-8}$ was chosen. The main drawback of the kSoS method is its reliance on SDP optimization which is considerably slower than SOC or quadratic programming, in additional to its memory requirements which prevent considering more than a few hundred constraints. Note that one cannot apply the tightening framework of \citealt{marteauferey20nonparametric} since it is not $\b f$ but its derivative which has to satisfy a nonnegativity constraint, whence $\b f$ is not itself a kernel sum-of-squares. \citet[Section 5]{muzellec22learning}, which can be seen as the extension of \citealt{marteauferey20nonparametric}, discusses this aspect.
	
	In all cases \eqref{eq:robot-arm-constraints-disc}-\eqref{eq:robot-arm-constraints-kSoS}, we apply the formula for $\bar{\b f}_{\app}$ given by the representer theorem (Proposition~ \ref{prop:reproducing}) and evaluate our methods by computing the following performance measures:
	\begin{align*}
		L^2_{err}&=\frac{1}{N^{test}_{err}}\sum_{n \in \left[N^{test}_{err}\right]}\left\|\b f^{ref}(\b z_n)-\bar{\b f}_{\app}(\b z_n)\right\|^2_2, \\
		L^1_{cons}&=\frac{1}{N^{test}_{cons}}\sum_{n \in [N^{test}_{cons}]}
		\sum_{i\in [d]}\sum_{l\in [2]}\max\left(0,-c_i^l(\tilde{\b z}_n) \partial_i f_l(\tilde{\b z}_n)\right).
	\end{align*}
	 
	 The estimated reconstruction error is designated by $L^2_{err}$ and is assessed over a fine regular grid $(\b z_n)_{n\in \left[N^{test}_{err}\right]}$ with $N^{test}_{err}=5^d$. The estimated violation of the constraints is denoted by $L^1_{cons}$ and is computed over $N^{test}_{cons}=400$ Latin hypercube samples $(\tilde{\b z}_n)_{n\in \left[N^{test}_{cons}\right]}$ over the input space $\X=[0,1]^{d}$. We report the obtained performance values in Table~\ref{table:MyTableLabel} along with the computational time $T_s$. In our experiments we used an i5-CPU 16GB-RAM computer and the YALMIP solver \citep{lofberg04yalmip} to solve the optimization problem \eqref{eq:robot-arm-objective} with each of the constraints \eqref{eq:robot-arm-constraints-disc}-\eqref{eq:robot-arm-constraints-kSoS}.  Missing values in the table correspond to memory outflows or when YALMIP does not converge. These events occur due to the amount of constraints considered: off-the-shelf SDP solvers struggle beyond $250$ SDP constraints as in \eqref{eq:robot-arm-constraints-kSoS}, and hyperplane constraints  as in \eqref{eq:robot-arm-constraints-SOChyp} require many cones.

	We notice that enforcing constraints always improves both the reconstruction error $L^2_{err}$ and violation of constraints $L^1_{cons}$ except for kSoS, the SOC techniques consistently giving the best results. The more constraint points used, the better the results and the more expensive the computations are. The performance of both SOC (ball) and SOC (hyp.) is almost identical for this experiment, with SOC (hyp.) being about $10$ times more expensive to run time-wise. We notice that, despite being solved through SOC programming, SOC (ball) takes a very comparable time, about $+25\%$ more, w.r.t.\ the quadratic programming used for the discretized constraints, making it a competitive alternative. Removing points with too small $\left|c_i^l(\tilde{\b x}_{m})\right|$ resulted in considering $N_C$ constraints instead of $d\times M$ but the two numbers are still quite close.
	
	These experiments demonstrate the efficiency of our proposed method in the vector-valued setting in moderate input dimensions.
	\begin{table}
		\footnotesize
		\centering
		\begin{tabular}{rrrrr@{\hspace{0.05cm}}l@{\hspace{0.05cm}}lr@{\hspace{0.05cm}}l@{\hspace{0.05cm}}lr@{\hspace{0.05cm}}l@{\hspace{0.05cm}}lr@{\hspace{0.05cm}}l@{\hspace{0.05cm}}lr@{\hspace{0.05cm}}l@{\hspace{0.05cm}}l}\toprule
		     &&&& \multicolumn{15}{c}{Handling of shape constraints}\\\cmidrule(r){5-19}
			 Perf. & $d$  & $M$ & $N_C$ & \multicolumn{3}{c}{Unconstrained} & \multicolumn{3}{c}{Discretized} & \multicolumn{3}{c}{SOC (ball)} & \multicolumn{3}{c}{SOC (hyp.)}  & \multicolumn{3}{c}{kSoS} \\
			\midrule
			$L_{\text{err}}^2$ & 4 & 16 & 61 & 0.608 &$\pm$& 9e-2 & 0.559 &$\pm$& 9e-2 & 0.542 &$\pm$& 9e-2 & \tb{0.541} &$\pm$& 9e-2 & 0.683 &$\pm$& 1e-1 \\
			&   & 81 & 303 & 0.588 &$\pm$& 9e-2 & 0.489 &$\pm$& 8e-2 & \tb{0.467} &$\pm$& 9e-2 & 0.476 &$\pm$& 1e-1 & --&& \\
			&   & 256 & 961 & 0.611 &$\pm$& 8e-2 & 0.486 &$\pm$& 6e-2 & \tb{0.484} &$\pm$& 7e-2 & --&& & --&& \\
			$L_{\mathrm{cons}}^1$ &   & 16 & 61 & 0.039 &$\pm$& 1e-2 & 0.026 &$\pm$& 7e-3 & \tb{0.020} &$\pm$& 6e-3 & \tb{0.020} &$\pm$& 6e-3 & 0.042 & $\pm$ & 1e-2 \\
			&   & 81 & 303 & 0.033 &$\pm$& 1e-2 & 0.009 &$\pm$& 3e-3 & \tb{0.005} &$\pm$& 2e-3 & \tb{0.005} &$\pm$& 2e-3 & --&& \\
			&   & 256 & 961 & 0.037 &$\pm$& 1e-2 & 0.003 &$\pm$& 1e-3 & \tb{0.002} &$\pm$& 1e-3 & --&& & --&& \\
			$T_s$ &   & 16 & 61 & \tb{$<$0.01} && & 0.081 &$\pm$& 5e-3 & 0.103 &$\pm$& 5e-3 & 2.135 &$\pm$& 5e-1 & 1.465 &$\pm$& 3e-1 \\
			&   & 81 & 303 & \tb{$<$0.01}&& & 0.287 &$\pm$& 5e-2 & 0.369 &$\pm$& 2e-2 & 37.150 &$\pm$& 9 & --&& \\
			&   & 256 & 961 & \tb{$<$0.01}&& & 2.430 &$\pm$& 3e-1 & 3.125 &$\pm$& 5e-1 & --&& & --&& \\
			\midrule
			$L_{\text{err}}^2$ & 6 & 64 & 360 & 1.621 &$\pm$& 5e-2 & 1.580 &$\pm$& 5e-2 & \tb{1.520} &$\pm$& 5e-2 & \tb{1.520} &$\pm$& 5e-2 & --&& \\
			&   & 729 & 4097 & 1.636 &$\pm$& 4e-2 & 1.511 &$\pm$& 5e-2 & \tb{1.345} &$\pm$& 9e-2 & --&& & --&& \\
			$L_{\mathrm{cons}}^1$ &   & 64 & 360 & 0.039 &$\pm$& 4e-3 & 0.021 &$\pm$& 2e-3 & \tb{0.013} &$\pm$& 1e-3 & \tb{0.013} &$\pm$& 1e-3 & --&& \\
			&   & 729 & 4097 & 0.040 &$\pm$& 4e-3 & 0.003 &$\pm$& 3e-4 & \tb{0.001} &$\pm$& 3e-4 & --&& & --&& \\
			$T_s$ &   & 64 & 360 & \tb{$<$0.01}&& & 0.447 &$\pm$& 2e-1 & 0.588 & $\pm$& 1e-1 & 71.250 &$\pm$& 2e1 & --&& \\
			&   & 729 & 4097 & \tb{$<$0.01}&& & 54.700 &$\pm$& 5 & 70.000 &$\pm$& 6 & --&& & --&& \\
			\bottomrule
		\end{tabular}
		\caption{Illustration in the robotic arm position estimation problem. Performance values (column $5$-$9$) under different handlings of the shape constraint: mean $\pm$ std (smaller is better). Columns from left to right: performance measure (reconstruction error: $L_{\text{err}}^2$, constraint violation: $L_{\mathrm{cons}}^1$, running time: $T_s$), dimension ($d$), number of anchor points ($M$), number of constraints ($N_C\le d\times M$), unconstrained solver, discretized constraints, SOC constraints (with balls), SOC constraints (with ball+hyperplanes), kSOS constraints. '--': lack of convergence or memory overflow. The best performance values are indicated in each row by boldface.}
		\label{table:MyTableLabel}
	\end{table}

	\color{black}
	\subsection{Experiment-4: Econometrics}
	Our \tb{fourth} example belongs to econometrics; our goal is to estimate production functions based on very few samples and additional side information. Particularly, let us consider a firm which produces an output from $d$ different goods/inputs/factors. Let the quantity corresponding to the $i^{th}$ input be written as $x_i\in \Rnn$ ($i\in [d]$). Then the corresponding output can be modelled by a production function $f:\X \subseteq \Rnn^d \rightarrow \Rnn$. Classical assumptions on the production function \citep{varian84nonparametric,allon07nonparametric} are (i) non-negativity ($f(\b x)\ge 0$ $\forall \b x$), (ii) monotonically increasing property (i.e., more inputs gives rise to more output; $\p^{\b e_i} f(\b x) \ge 0$ $\forall \b x $ and $\forall i \in [d]$), (iii) $f(\b 0)=0$ (zero input gives no output) and (iv) concavity (also called diminishing marginal returns; $\left[\left(\p^{\b e_i+\b e_j}f\right)(\b x) \right]_{i,j\in [d]} \preccurlyeq \b 0_{d\times d}$ $\forall \b x$). Having access to $N$ input-output samples $S=(\b x_n, y_n)_{n\in [N]}$, the learning of a production function can be addressed by solving
	\begin{mini*}|s|
		{\substack{f\in \FK}}{\Lcal_{S}(f) := \frac{1}{N} \sum_{n\in [N]}[y_n-f(\b x_ n)]^2 + \lambda \left\|f\right\|_{K}^2, \quad (\lambda>0)}
		{}
		{}
		\addConstraint{0}{\le f(\b x)\quad \forall \b x\in \Kcons}
		\addConstraint{0}{\le \p^{\b e_i} f(\b x) \quad \forall \b x\in \Kcons, \forall i\in [d]}
		\addConstraint{0}{= f(\b 0)}
		\addConstraint{\b 0_{d\times d} }{\preccurlyeq -\left[\left(\p^{\b e_i+\b e_j}f\right)(\b x) \right]_{i,j\in [d]}\quad \forall \b x\in \Kcons,}
	\end{mini*}
	where $\Kcons\subset \left(\Rnn\right)^d$ is a compact set containing the samples. This problem belongs to the family \eqref{opt-cons} with the choice $s=2$, $I=d+2$, $D_1(f)=f$ and $P_1=1$, $D_i=\p^{\b e_{i-1}}$ and $P_i = 1$ for $i\in\{2,3,\ldots,d+1\}$, $\b D_{d+2} = -\left[\p^{\b e_i+\b e_j} \right]_{i,j\in [d]}$, $P_{d+2}=d$, $\b \Gamma_i=\b 0$, $\b b_{0,i}=\b 0$ and $f_{0,i}=0$ for all $i \in [d+2]$. The  requirement $f(\b 0)=0$ can be encoded by incorporating an indicator function to the loss function.  
	
	For our experiment, we considered a benchmark dataset containing the production data of $569$ Belgian firms.\footnote{The dataset is available at \url{https://vincentarelbundock.github.io/Rdatasets/doc/Ecdat/Labour.html}.}  The input $\b x$ is two-dimensional ($d=2$), describing the capital expressed in euros ($x_1$) and the labour involved, interpreted as the number of workers ($x_2$). The output $y$ is one-dimensional ($Q=1$), and is the added value in euros. We applied a standard pre-processing of the data \citep{mazumder19computational} by (i) considering the negative logarithm of the output\footnote{Taking the logarithm of the output improves the numerical stability at the price of discarding the constraint $f(\b 0) = 0$.}, (ii)  mean-centering and standardizing each component of the input and of the output to have zero mean and unit variance, and (iii) removing some outliers, resulting in $N_{tot}=543$ points kept. The final optimization problem\footnote{Imposing the quadratic regularization $\lambda \left\|g\right\|_K^2$ as an equivalent constraint $\left\|g\right\|_K\le \tilde{\lambda}$ is in line with the implementation of conic convex problems through interior point methods.} contains two monotonicity and one joint convexity constraint:
	\begin{mini!}|s|
		{\substack{g\in \FK}}{\Lcal_{S}(g) := \frac{1}{N} \sum_{n\in [N]}[y_n-g(\b x_ n)]^2}
		{}
		{}
		\addConstraint{\|g\|_{K}}{\le \tilde{\lambda}\label{eq:reg}}
		\addConstraint{0}{\le -\p^{\b e_1} g(\b x) \quad \forall \b x\in \Kcons \label{eq:m1}}
		\addConstraint{0}{\le -\p^{\b e_2} g(\b x) \quad \forall \b x\in \Kcons\label{eq:m2}}
		\addConstraint{\b 0_{2\times 2} }{\preccurlyeq \left[\left(\p^{\b e_i+\b e_j}g\right)(\b x) \right]_{i,j\in [2]}\quad \forall \b x\in \Kcons.\label{eq:c}}
	\end{mini!}
	To demonstrate the importance of imposing the shape constraints we made the problem even more challenging and fixed $\Kcons=\prod_{j\in[2]}\left[\min_{n\in [N_{tot}]}(\b x_{n})_j,2\right]$. This choice allows us to illustrate how the imposed shape constraints are satisfied outside of $\Kcons$, which here does not contain all the points. The covering of $\Kcons$ was uniform, performed through rectangles of size $\delta_1 \times \delta_2$. The values of $\delta_1$ and $\delta_2$ were chosen to have $15$ added points per dimension, resulting in $M=225$ $\tilde{\b x}_m$-s. The chosen kernel was Gaussian with bandwidth $\sigma$ set to the  square root of the eighth decile of the squared pairwise distances of the points $(\b x_n)_{n\in[N_{tot}]}$. As discussed after \eqref{eq:eta_SDP_simple}, the Gaussian kernel being translation-invariant, the computation of $\eta_{m,P_i}$ can be centered at the origin, and it is sufficient to evaluate $\eta(\b 0, \b \delta; D_i)$ defined in \eqref{eq:eta-function}. These values were approximated numerically by taking
	$50$ $\X$-points uniformly at random in $[-\delta_1,\delta_1]\times[-\delta_2,\delta_2]$. For the convexity constraint ($P_i=2$), we applied additionally a $\b u=[\cos(\theta);\sin(\theta)]$ parameterization with $20$ equidistant values of $\theta$ from $[0,\pi)$, owing to the invariance of $\eta(\b 0, \b \delta; D_i)$ when replacing $\b u$ by $-\b u$.
	We considered four scenarios in terms of the shape constraints imposed: (i) no shape constraint [\eqref{eq:reg}], (ii) two SOC-based monotonicity constraints [\eqref{eq:reg}-\eqref{eq:m2}], (iii) one SOC-based convexity constraint [\eqref{eq:reg}, \eqref{eq:c}], (iv) two SOC-based monotonicity and one SOC-based convexity constraint [\eqref{eq:reg}-\eqref{eq:c}], where turning \eqref{eq:c} into \eqref{eq:pointwise-constraints_SOC} still leads to and SDP constraint, adding an extra variable.
	In our experiments, we partitioned randomly the dataset	$(\b x_n,y_n)_{n\in [N_{tot}]}$
	into a validation set $\D_{val}$ and a test set $\D_{test}$ of approximately equal size ($\#\D_{val}=271$, $\#\D_{test}=272$) corresponding each to $50\%$ of the total dataset. $20$-fold cross-validation  was performed on	$\D_{val}$ to estimate the optimal value of $\tilde{\lambda}$ on a logarithmic grid. We then selected randomly $10\%$ of $\D_{val}$ (referred to as $\D_{val}^{'}$) to optimize $\Lcal$ over this small training set using one of the four constraint settings detailed above for the estimated $\tilde{\lambda}$.\footnote{The rationale behind selecting only $10\%$ of $\D_{val}$ is to make the problem more challenging and to illustrate the usefulness of considering shape constraints for small sample size.} The efficiency of the resulting estimate for $g$ was evaluated by the mean-squared error (MSE) over $\D_{val}^{'}$ and $\D_{test}$.
	The whole experiment was repeated $20$ times. The resulting statistics on the MSE values are summarized in Fig.~\ref{fig:labour_boxplot_small}, with a visual illustration of the underlying curves in Fig.~\ref{fig:labour}. As it can be observed, adding shape constraints gradually improves the generalization performance (Fig.~\ref{fig:labour_boxplot_small}) while mitigating overfitting on the training set, and also helps satisfying the shape requirements outside of the constraint set $\Kcons$ (Fig.~\ref{fig:labour}). Similar improvements on the correspondence between test and train RMSE when incorporating shape constraints have been observed in an early version\footnote{See \url{https://github.com/mcurmei627/dantzig/tree/master/Experiments/Synthetic}} of \citet{curmei21shape}.
	
	These four applications demonstrate the efficiency of the proposed SOC approach in the context of shape optimization tasks, safety-critical control, robotics, and econometrics.
	
	\begin{figure}
		\centering	
		\resizebox{\linewidth}{!}{\includegraphics{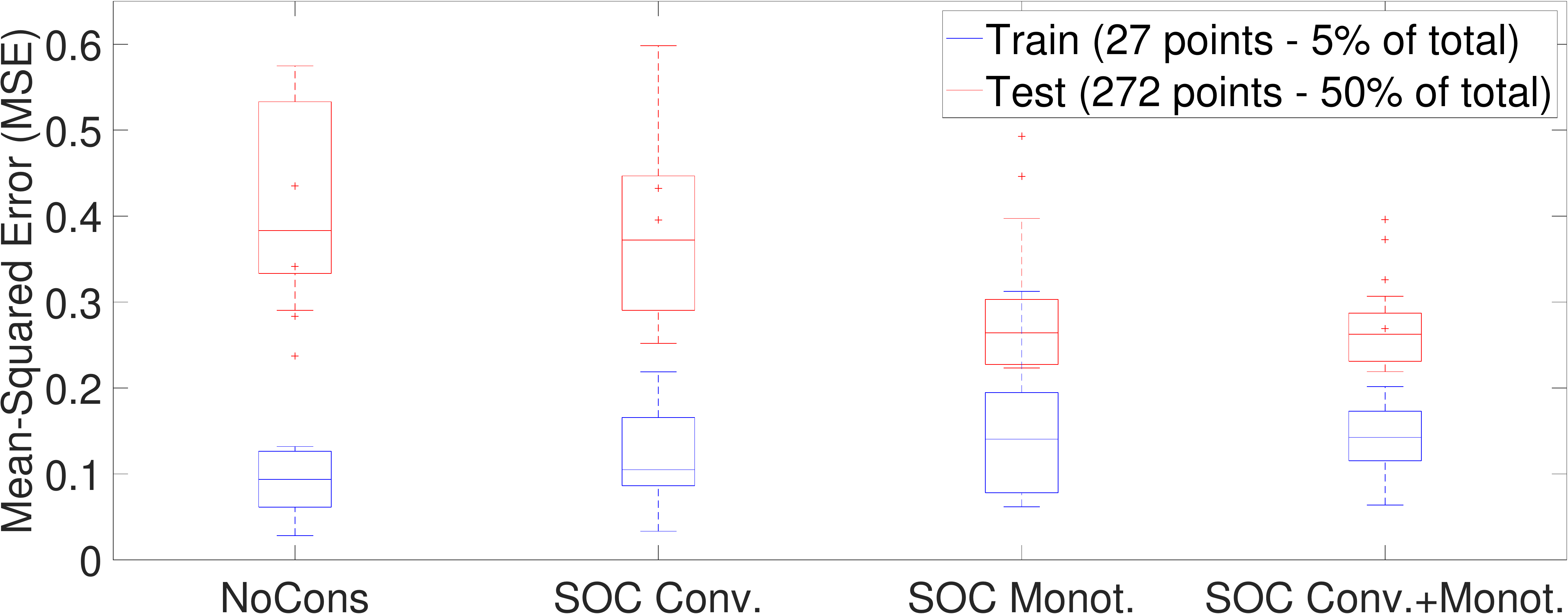}}	
		\caption{MSE as a function of incorporating shape constraints with the proposed SOC technique. NoCons: no constraint. SOC Monot.: two monotonicity constraints. SOC Conv.: one convexity constraint. SOC Conv.+Monot.:  one convexity and two monotonicity constraints.}
		\label{fig:labour_boxplot_small}
	\end{figure}
	\begin{figure}
		\centering	
		\begin{minipage}[c]{.48\linewidth}
			\begin{center}			
				\subfloat[][NoCons]{\resizebox{\linewidth}{!}{\label{Labour_27Pts_wMSCons}\includegraphics{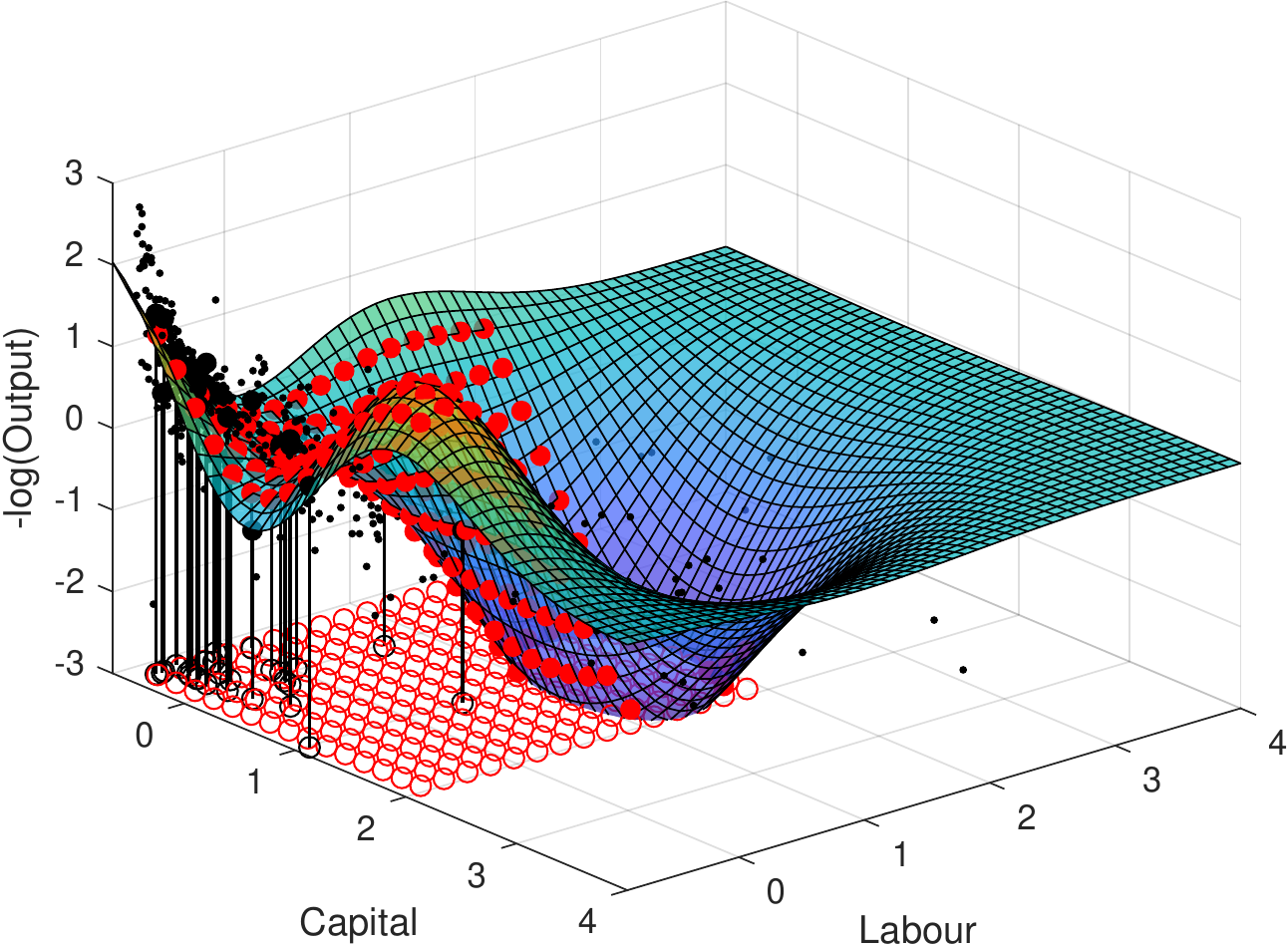}}}
			\end{center}				
		\end{minipage}
		\begin{minipage}[c]{.48\linewidth}
			\begin{center}			
				\subfloat[][SOC Monot.]{\resizebox{\linewidth}{!}{\label{Labour_27Pts_wCSCons}\includegraphics{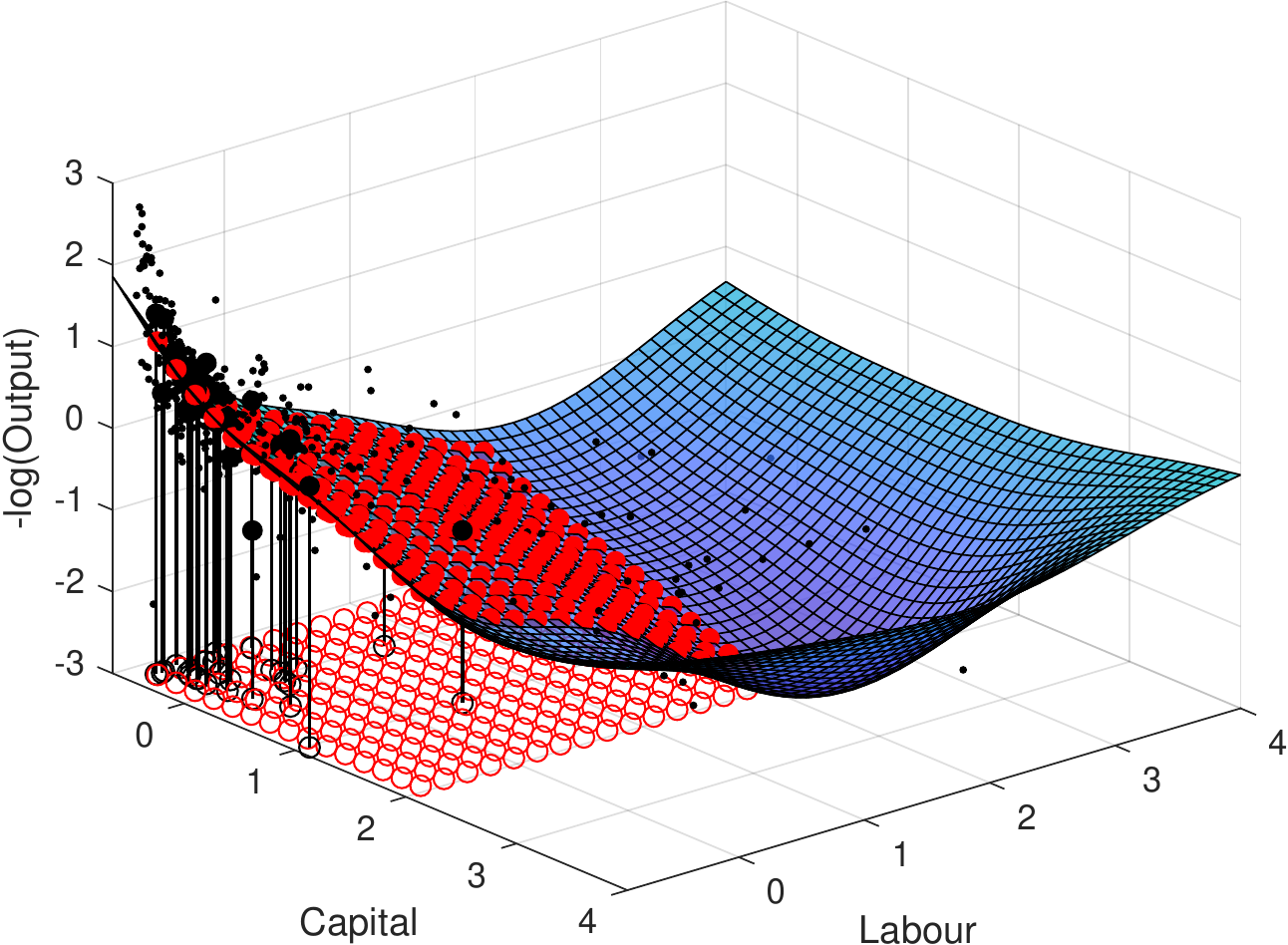}}}
			\end{center}				
		\end{minipage}
		
		\begin{minipage}[c]{.48\linewidth}
			\begin{center}			
				\subfloat[][SOC Conv.]{\resizebox{\linewidth}{!}{\label{Labour_27Pts_wMSCSCons}\includegraphics{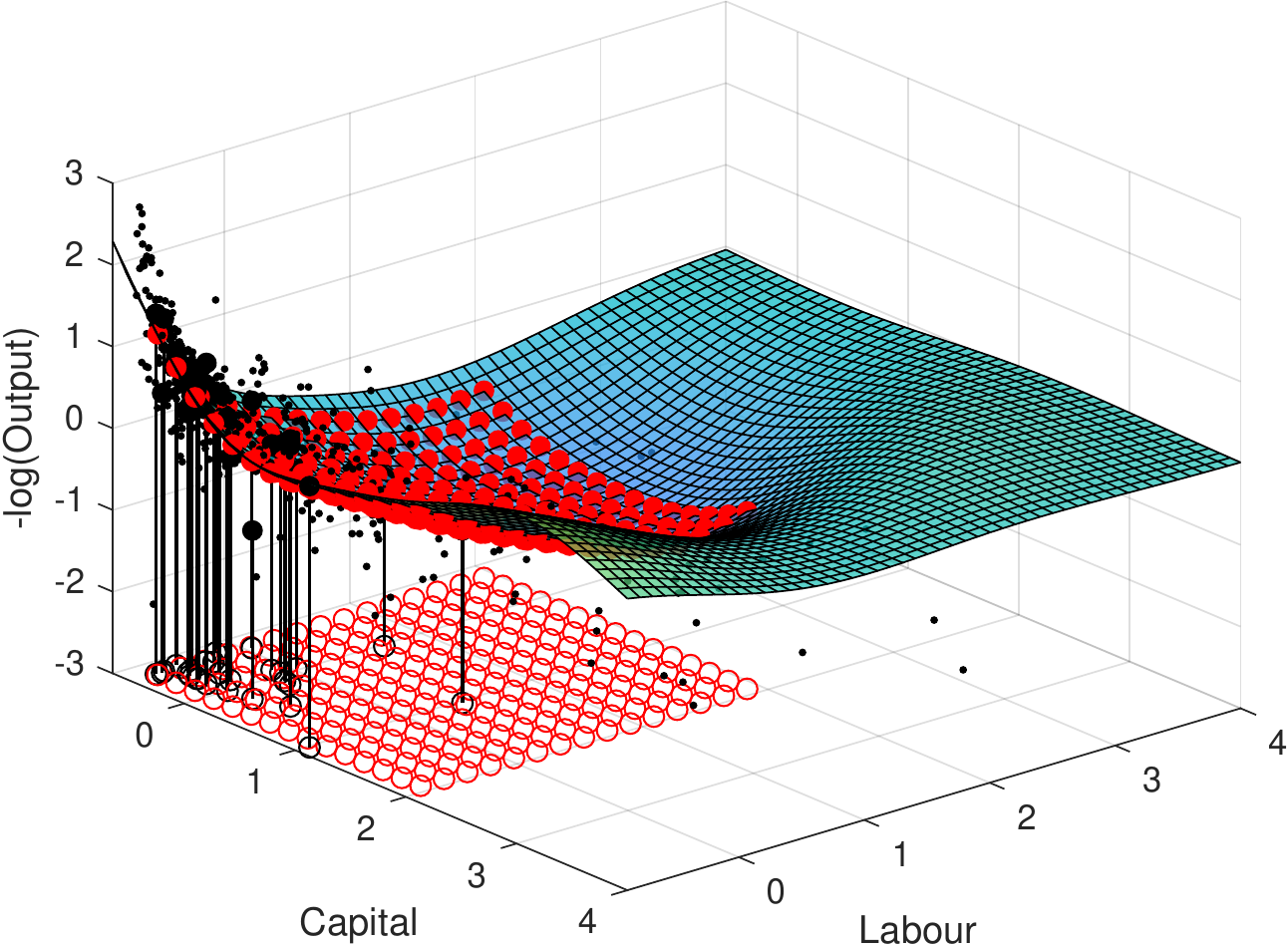}}}
			\end{center}				
		\end{minipage}
		\begin{minipage}[c]{.48\linewidth}
			\begin{center}			
				\subfloat[][SOC Conv.+Monot.]{\resizebox{\linewidth}{!}{\label{Labour_27Pts_wMCCons}\includegraphics{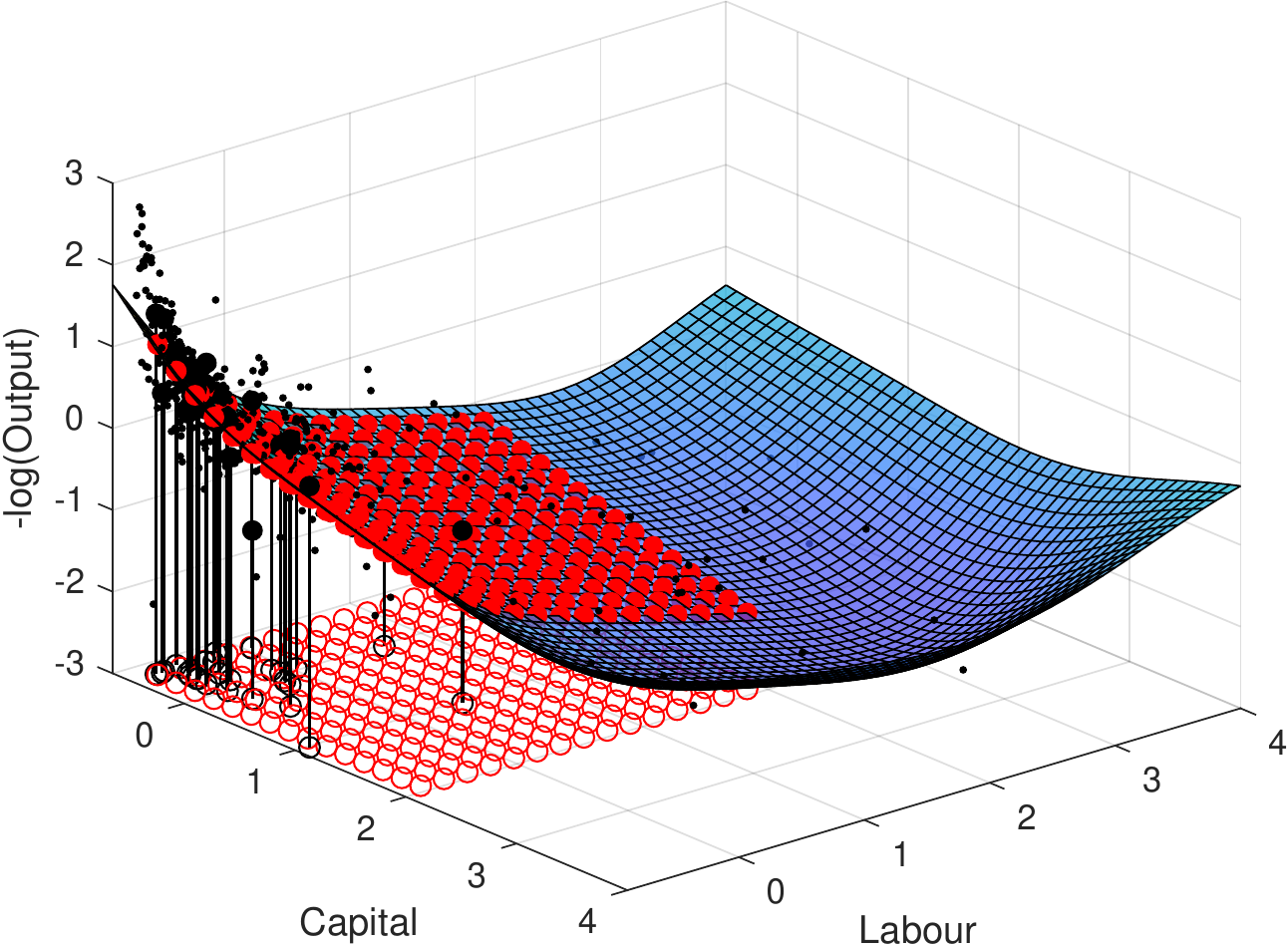}}}
			\end{center}				
		\end{minipage}
		\caption{Illustration of the production function estimates with different shape constraints. Notation of the methods: as in Fig.~\ref{fig:labour_boxplot_small}. Red circles: covering points of $\Kcons$. Red points on the surface: resulting $y$ values. Black circles with vertical lines: $N_{te}$ test points. Black circles without vertical lines: remaining  $(\b x, y)$ points.}
		\label{fig:labour}
	\end{figure}
	
	\section{Conclusions} \label{sec:conclusions}
	In this paper we focused on the problem of incorporating hard affine SDP shape constraints on function  derivatives into optimization problems over vector-valued reproducing kernel Hilbert spaces. We proposed a unified and modular second-order cone (SOC) based convex optimization framework to tackle this task. We designed and analysed two complementary approaches to derive SOC-based tightenings; they build upon a convex separation theorem in RKHSs (Theorem~\ref{thm:inclusion}) and on an upper bound of the modulus of continuity (Theorem~\ref{thm:SDP}). We established the existence and certificate of optimality of the tightenings (Theorem~\ref{thm:certificate}) alongside with a convergence guarantee (Proposition~\ref{thm:apriori_bound}) in terms of the refinement of the underlying covering. In addition, we proposed the soap bubble algorithm which guarantees hard shape constraints while adaptively refining the covering, and proved its convergence (Theorem~\ref{thm:soap_bubble}). The efficiency of the approach was demonstrated in four applications (Section~\ref{sec:numerical-demos}): in the context of shape optimization, safety-critical control, robotics and econometrics.
	
	\acks{We thank the anonymous referee for his positive comments and for pointing out the useful reference \citet{attouch14variational}. ZSz benefited from the support of the Europlace Institute of Finance and that of the \href{http://www.cmap.polytechnique.fr/~stresstest/}{Chair Stress Test}, RISK Management and Financial Steering, led by the French École Polytechnique and its Foundation and sponsored by BNP Paribas.}
	
	\appendix
	\section{Proofs}\label{sec:proofs}
	Section~\ref{sec:proof:main} contains the proofs of our results (detailed in Section~\ref{sec:constraints} -- Section~\ref{sec:covering_algorithms}).  Section~\ref{sec:proof:aux} is dedicated to auxiliary lemmas used in Section~\ref{sec:proof:main}.
	
	\subsection{Auxiliary Lemmas} \label{sec:proof:aux}
	In this section we provide auxiliary lemmas with their proofs.
	\begin{lemma}[Infimum over balls]\label{lemma:inf:balls}
		Let $\F$ be a Hilbert space, $\b g, \b c\in \F$ and $r>0$. Then 
		\begin{align*}
			\inf\limits_{\b w \in \dBB_{\F}(\b c,r)} \left<\b g,\b w\right>_{\F}=\left<\b g,\b c\right>_{\F}-r\|\b g\|_\F.
		\end{align*}
	\end{lemma}
	
	\begin{proof}{(Lemma~\ref{lemma:inf:balls})}
		The statement follows by noting that
		\begin{align*}
			\inf\limits_{\b w \in \dBB_{\F}(\b c,r)} \left<\b g,\b w\right>_{\F}=
			\left<\b g,\b c\right>_{\F}+\inf\limits_{\b w \in \BB_{\F}(\b 0,r)} \left<\b g,\b w\right>_{\F}=\left<\b g,\b c\right>_{\F}-r\|\b g\|_\F.
		\end{align*}
	\end{proof}
	
	\begin{lemma}[Infimum over half-spaces] \label{lemma:inf:half-spaces}
		Let $\F$ be a Hilbert space, $\b g,\b v\in\F$, $\b v\neq \b 0$, $\rho>0$ and assume that $\inf\limits_{\b w\in \dH^-_{\F}(\b v,\rho)}\left<\b g,\b w\right>_{\F}$ is finite. Then there exists $\xi\in \R_+$ such that
		\begin{align*}
			\b g=-\xi \b v \text{ and } -\xi \rho=\inf\limits_{\b w\in H^-_{\F}(\b v,\rho)}\left<\b g,\b w\right>_{\F}.
		\end{align*}
	\end{lemma}
	
	\begin{proof}{(Lemma~\ref{lemma:inf:half-spaces})}
		Let us decompose $\b g$ along the one-dimensional subspace spanned by $\b v$: 
		$\b g=-\xi \b v +\b u$ where  $\xi\in\RR$, $\b u\in \F$ and $\left<\b u,\b v\right>_{\F}=0$. We show that a finite infimum implies that in this decomposition $\xi\ge 0$ and $\b u=\b 0$. Indeed,
		\begin{itemize}[labelindent=0em,leftmargin=1em,topsep=0.2cm,partopsep=0cm,parsep=0cm,itemsep=2mm]
			\item $\xi\ge 0$:  
			\begin{align}
				-\infty & \stackrel{(a)}{<} \hspace{-0.15cm} \inf_{\b w\in H^-_\F(\b v,\rho)}\left<\b g,\b w\right>_{\F} \stackrel{(b)}{\le} \hspace{-0.1cm} \inf_{\tau\in\RR_+}\left<-\xi \b v +\b u,\frac{\rho}{\|\b v\|_\F^2}\b v-\tau \b v\right>_{\F} \stackrel{(c)}{=} -\xi\rho+\inf_{\tau\in\RR_+}\xi\tau\|\b v\|^2_\F, \label{eq:xi}
			\end{align}
			where (a) holds by our assumption on the finiteness of the infimum. (b) is implied by the fact that for all $\tau \ge 0$, $\rho - \tau \left\| \b v \right\|_{\F}^2 \le \rho$, so $\left\langle\frac{\rho}{\|\b v\|_\F^2}\b v-\tau \b v, \b v\right\rangle_\F \le \rho$, hence $\frac{\rho}{\|\b v\|^2}\b v-\tau \b v \in H^-_\F(\b v,\rho)$. 
			(c) follows from $\left<\b u,\b v\right>_{\F}=0$. As $\b v\ne \b 0$,  \eqref{eq:xi} implies that $\xi\ge 0$.        
			\item $\b u=\b 0$:
			\begin{align}
				-\infty & \stackrel{(a)}{<} \hspace{-0.19cm} \inf_{\b w\in H^-_\F(\b v,\rho)}\left<\b g,\b w\right>_{\F}\stackrel{(b)}{\le} \hspace{-0.1cm} \inf_{\tau\in\RR_+}\left<-\xi \b v +\b u,\frac{\rho}{\|\b v\|_\F^2}\b v-\tau \b u\right>_{\F}  \stackrel{(c)}{=} \hspace{-0.1cm} -\xi\rho+\inf_{\tau\in\RR_+}-\tau\|\b u\|^2_\F. \label{eq:u}
			\end{align}
			Our assumption on the finiteness of the infimum implies (a). (b) follows from $\frac{\rho}{\|\b v\|_\F^2}\b v-\tau \b u \in H^-_\F(\b v,\rho)$ since
			$\left\langle \frac{\rho}{\|\b v\|_\F^2}\b v-\tau \b u, \b v\right\rangle_\F = \rho$ for any $\tau \in \R$.
			(c) is again a consequence of   $\left<\b u,\b v\right>_{\F}=0$. Hence \eqref{eq:u} means that $\b u = \b 0$.
		\end{itemize}
		Applying the obtained $\b g=-\xi \b v$ relation ($\xi \ge 0$),  we conclude that 
		\begin{align*}
			\inf_{\b w\in H^-_\F(\b v,\rho)}\left<\b g,\b w\right>_{\F} 
			&= \inf_{\b w\in H^-_\F(\b v,0)}\left<-\xi \b v,\frac{\rho}{\|\b v\|_\F^2}\b v+\b w\right>_{\F} =  \inf_{\b w\in H^-_\F(\b v,0)} \left( -\xi \rho  + \left<-\xi \b v, \b w\right>_\F\right) = -\xi \rho
		\end{align*}
		using that $\left<-\xi \b v, \b w\right>_\F\ge 0$ since $\xi\ge 0$ and $\b w \in H^-_\F(\b v,0)$, with the infimum attained at $\b w =\b 0$.
	\end{proof}
	
	\begin{lemma}[Closed convex constraints]\label{lemma:closed_constraints}
		Let $C_{\app} = \left(\bigcap_{i\in\mathcal{I}_{\text{SOC}}} C^i_{P_i,\text{SOC}}\right)\cap \left(\bigcap_{i\in\mathcal{I}_{\Omega}} C^i_{1,\Omega}\right)$ with $C^i_{P_i,\text{SOC}}$ and $C^i_{1,\Omega}$ defined as in \eqref{def:mixed-constraint:P>=1_SOC} and \eqref{def:mixed-constraint:P=1_SOC}, then $C_{\app}$ is a  closed convex set of $\FK\times \R^B$. So is $C$, defined in \eqref{def_mixded_constraint}.
	\end{lemma}
	
	\begin{proof}{(Lemma~\ref{lemma:closed_constraints})}
		The set $C_{\app}$ is closed and convex as it is the intersection of the closed convex sets $\left\{C^i_{P_i,\text{SOC}}\right\}_{i\in\mathcal{I}_{\text{SOC}}}$ and $\left\{C^i_{1,\Omega}\right\}_{i\in \mathcal{I}_{\Omega}}$. The closedness of the latter sets can be proved as follows.
		\begin{itemize}[labelindent=0em,leftmargin=1.2em,topsep=0.2cm,partopsep=0cm,parsep=0cm,itemsep=0mm]
			\item Closedness of $C^i_{P_i,\text{SOC}}$: Since $\|\cdot\|_K$ is lower semicontinuous and $\FK$ is a vRKHS, the evaluation of the functions being continuous, the SOC constraints define closed sets, thus any $C^i_{P_i,\text{SOC}}$ is closed.
			\item Closedness of $C^i_{1,\Omega}$: Let $\left(\b f^{(k)}, \b b^{(k)}\right)_{k\in\NN}\in \left(C^i_{1,\Omega}\right)^\NN$ converge to some $(\b f, \b b)\in \FK\times \R^B$. We show that $(\b f, \b b)\in C^i_{1,\Omega}$ which is equivalent to  $\Omega_{i,m}\subseteq H^{+}_{\K}\left(\b f - \b f_{0,i}, b_{0,i}-\bm{\Gamma}_i\b b\right)$ for every $m\in[M_i]$ by \eqref{eq:inclusion:m-fixed}.  Since the sequence $\left(\b f^{(k)}, \b b^{(k)}\right)_{k\in\NN}\in \left(C^i_{1,\Omega}\right)^\NN$, one has $\Omega_{i,m}\subseteq H^{+}_{\K}\left(\b f^{(k)} - \b f_{0,i}, b_{0,i}-\bm{\Gamma}_i\b b^{(k)}\right)$ for all $k\in\NN$ and $m\in[M_i]$ by \eqref{eq:inclusion:m-fixed}, i.e.\ $\langle \b g, \b f^{(k)} - \b f_{0,i}\rangle_K \ge b_{0,i}-\bm{\Gamma}_i\b b^{(k)}$
			for all $k\in\NN$ and any $\b g \in \Omega_{i,m}$. This implies   that $\langle \b g, \b f - \b f_{0,i}\rangle_K \ge b_{0,i}-\bm{\Gamma}_i\b b$ also holds for all $\b g \in \Omega_{i,m}$ by continuity. Hence $\Omega_{i,m}\subseteq H^{+}_{\K}(\b f - \b f_{0,i}, b_{0,i}-\bm{\Gamma}_i\b b)$ for all $m\in [M_i]$ which means that $(\b f, \b b)\in C^i_{1,\Omega}$ by \eqref{eq:inclusion:m-fixed}.
		\end{itemize}
		Similarly $C$ is the intersection of closed convex sets as per \eqref{def_mixded_constraint}.
	\end{proof}	
	
	\subsection{Proofs of Our Results} \label{sec:proof:main}
	This section contains the proofs of the results presented in Section~\ref{sec:constraints}, Section~\ref{sec:optimization}, and Section~\ref{sec:covering_algorithms}: Lemma~\ref{lemma:reproducing} (Section~\ref{sec:proof:lemma:reproducing}), Theorem~\ref{thm:inclusion} (Section~\ref{sec:proof:thm:inclusion}), Theorem~\ref{thm:SDP} (Section~\ref{sec:proof:thm:SDP}), Lemma~\ref{lemma:eta_SDP_4Dtensor} (Section~\ref{sec:proof:lemma:eta_SDP_4Dtensor}),
	Theorem~\ref{thm:certificate} (Section~\ref{sec:proof:thm:certificate}), 
    Corollary~\ref{thm:aposteriori_bound} (Section~\ref{sec:proof:thm:aposteriori_bound}), 
    Proposition~\ref{thm:apriori_bound} (Section~\ref{sec:proof:thm:apriori_bound}),
    Proposition~\ref{prop:reproducing} (Section~\ref{sec:proof:prop:reproducing}), and 
    Theorem~\ref{thm:soap_bubble} (Section~\ref{sec:proof:thm:soap_bubble}). 
	
	\subsubsection{Proof of Lemma~\ref{lemma:reproducing}} \label{sec:proof:lemma:reproducing}
	By the reproducing property of matrix-valued kernels 
	\begin{align*}
		f_q(\b x) &= \b e_q^\T \b f(\b x)= \left<\b f,K(\cdot,\b x)\b e_q\right>_\K \stackrel{(*)}{\Rightarrow} 
		D_q f_q(\b x)  = \left<\b f,D_q K(\cdot,\b x)\b e_q\right>_\K 
		\intertext{provided that the terms on the r.h.s.\ of the  implication $(*)$ exist, which is proved below. Hence}
		D(\b f)(\b x) &= \sum_{q\in [Q]} \beta_q D_q f_q(\b x)
		= \sum_{q\in [Q]} \beta_q \left<\b f,D_q K(\cdot,\b x)\b e_q\right>_\K =  \left<\b f,\sum_{q\in [Q]}  D_q K(\cdot,\b x)\beta_q\b e_q\right>_\K.
	\end{align*}
	For $(*)$ to be valid, one has to show that for any $\b r\in \N^d$ satisfying $|\b r| \le s$ and any $q\in[Q]$ we have
	\begin{subequations}
		\begin{gather}
			\b f\in \C^s(\X, \R^Q), \label{eq:repr-prop_proof:1}\\
			\p_2^{\b r}K(\cdot,\b x)\b e_q\in\FK,\label{eq:repr-prop_proof:2}\\
			\p^{\b r} (f_q)(\b x) = \left<\b f, \p_2^{\b r}K(\cdot,\b x)\b e_q\right>_\K \; (\forall\,\b f \in \FK,\,\b x \in \X), \label{eq:repr-prop_proof:3}
		\end{gather}
	\end{subequations}
	where $\p_2^{\b r}K(\b x',\b x):=\p^{\b r}[\b x \mapsto K(\b x',\b x)] \in \R^{Q\times Q}$; this extends to general $D_q$ by taking linear combinations. 
	We prove \eqref{eq:repr-prop_proof:1}, \eqref{eq:repr-prop_proof:2}, \eqref{eq:repr-prop_proof:3} by induction over $s_0\in\iv{0}{s-1}$, assuming the property to be satisfied for all $\b r$ such that $|\b r| \le s_0$. For $s_0=0$, the assertion is true. Fix $p,q\in[Q]$ and $\b r$ satisfying $|\b r| = s_0$.
	Let $\b r':=\b r + \b e_p$, $|\b r'| = |\b r|+1 = s_0+1$ where $p,q\in[Q]$ are fixed and $\b e_p\in \R^d$ is the $p^{th}$ canonical basis vector. We show the statement first for the interior $\mathring{\X}$, then for the whole $\X$, extending by continuity. For all $h\neq 0$ and $\b x \in \X$, let us introduce the difference quotient $\Delta_{h, \b x}$, the limits of which shall give \eqref{eq:repr-prop_proof:1}, \eqref{eq:repr-prop_proof:2}, \eqref{eq:repr-prop_proof:3}
	\begin{align}
		\Delta_{h, \b x} := \frac{\partial_2^{\b r}K(\cdot,\b x+h\b e_p)\b e_q-\partial_2^{\b r}K(\cdot,\b x)\b e_q}{h}. \label{def:diff-quotient}
	\end{align}
	
	\noindent $\bullet$ Case of $\mathring{\X}$: 
	Take $\b x_1, \b x_2 \in \mathring{\X}$ and $\rho>0$ such that $\BB_\X(\b x_1,\rho)\cup \BB_\X(\b x_2,\rho)\subset \mathring{\X}$. Let $h_1, h_2 \in[-\rho,\rho]\backslash\{0\}$. By induction, for any $\b x, \b x'\in\X$,
	\begin{align}
		\b e_q^\top \partial_1^{\b r} \partial_2^{\b r}K(\b x',\b x)\b e_q=\partial^{\b r} (\b e_q^\top \partial_2^{\b r}K(\cdot,\b x)\b e_q)(\b x') = \left<\partial_2^{\b r}K(\cdot,\b x)\b e_q, \partial_2^{\b r}K(\cdot,\b x')\b e_q\right>_\K, \label{eq:deriv-prop_DD}
	\end{align}
	where $\p_1^{\b r}$ is defined analogously to $\p_2^{\b r}$. Let us derive Cauchy sequences based on $\Delta_{h, \b x}$. Since 
	\begin{align}
		\|\Delta_{h_1, \b x_1} - \Delta_{h_2, \b x_2}\|_\K^2 &= 
		\|\Delta_{h_1, \b x_1}\|_\K^2 + \|\Delta_{h_2, \b x_2}\|_\K^2 -2 \left<\Delta_{h_1, \b x_1}, \Delta_{h_2, \b x_2}\right>_\K, \label{eq:diff-decomposition}
	\end{align}
	it is sufficient to consider quantities of the form
	\begin{eqnarray}
		\lefteqn{\left<\Delta_{h_1, \b x_1}, \Delta_{h_2, \b x_2}\right>_\K\nonumber}\\
		&&\hspace{-0.9cm}= \frac{1}{h_1 h_2}\left[
		\left<\p_2^{\b r}K(\cdot,\b x_1+h_1\b e_p)\b e_q,\p_2^{\b r}K(\cdot,\b x_2+h_2\b e_p)\b e_q\right>_K
		+  \left<\p_2^{\b r}K(\cdot,\b x_1)\b e_q,\p_2^{\b r}K(\cdot,\b x_2)\b e_q\right>_K \right.\nonumber\\
		&&\hspace{-0.9cm} \hspace{1.4cm}  - \left. \left<\p_2^{\b r}K(\cdot,\b x_1)\b e_q, \p_2^{\b r}K(\cdot,\b x_2+h_2\b e_p)\b e_q\right>_K -
		\left<\p_2^{\b r}K(\cdot,\b x_1+h_1\b e_p)\b e_q,\p_2^{\b r}K(\cdot,\b x_2)\b e_q\right>_K\right].\nonumber\\
		&&\hspace{-0.9cm}=\frac{1}{h_1 h_2}\b e_q^\top \left[ 
		\p_1^{\b r} \p_2^{\b r}K(\b x_2+h_2\b e_p,\b x_1+h_1\b e_p) +
		\p_1^{\b r} \p_2^{\b r}K(\b x_2,\b x_1) \right.\nonumber\\
		&&\quad \hspace{1.5cm} \left. - 
		\p_1^{\b r} \p_2^{\b r}K(\b x_2 + h_2 \b e_p,\b x_1) -
		\p_1^{\b r} \p_2^{\b r}K( \b x_2,\b x_1 + h_1 \b e_p)
		\right] \b e_q\nonumber\\
		&& \hspace{-0.9cm}= \int_{0}^{1}\int_{0}^{1} \b e_q^\top \p_1^{\b r'} \p_2^{\b r'}K(\b x_1+\alpha h_1\b e_p,\b x_2+\beta h_2\b e_p)\b e_q \d\alpha \d\beta, \label{eq:int-by-part} 
	\end{eqnarray}
	where \eqref{eq:int-by-part} follows from integration by parts and by the fact that $\K \in \C^{s,s}\left(\X\times \X,\R^{Q\times Q}\right)$.
	Applying the resulting expression \eqref{eq:int-by-part} in \eqref{eq:diff-decomposition}, we obtain that
	\begin{align}
		\|\Delta_{h_1, \b x_1} - \Delta_{h_2, \b x_2}\|_\K^2 &= 
		\sum_{i,j \in [2]} (-1)^{i+j}\int_{0}^{1}\int_{0}^{1} \b e_q^\top \p_1^{\b r'} \p_2^{\b r'}K(\b x_i+\alpha h_i\b e_p,\b x_j+\beta h_j\b e_p)\b e_q \d\alpha \d\beta. \label{eq:target:before-modulus-of-cont}
	\end{align}
	To upper bound \eqref{eq:target:before-modulus-of-cont}, 	since $\K \in \C^{s,s}\left(\X\times \X,\R^{Q\times Q}\right)$, one can define the modulus of continuity for any $\delta \geq 0$
	\begin{align*}
		\omega\left(\b e_q^\top \p_1^{\b r'} \p_2^{\b r'}K(\cdot,\cdot)\b e_q, \delta\right) & = \hspace{-0.1cm} \sup_{\substack{\b x,\, \b x',\,\b y,\, \b y'\in\X,\\ \|\b x-\b x' \|_2\le \delta,\,\|\b y-\b y'\|_2\le \delta}} \hspace{-0.2cm} \left|\b e_q^\top \p_1^{\b r'} \p_2^{\b r'}K(\b x',\b x)\b e_q-\b e_q^\top \p_1^{\b r'} \p_2^{\b r'}K(\b y',\b y)\b e_q\right|
	\end{align*}
	which is a continuous function of $\delta$, with limit $0$ at $0$. Forming two groups in \eqref{eq:target:before-modulus-of-cont} with $(i,j) \in \{(1,1), (1,2)\}$ and 
	$(i,j) \in \{(2,2),(2,1)\}$ 
	one gets the bound 
	\begin{align}\label{eq:bound_Delta}
		\|\Delta_{h_1, \b x_1} - \Delta_{h_2, \b x_2}\|_\K\le \sqrt{2\omega\left(\b e_q^\top \p_1^{\b r} \p_2^{\b r}K(\cdot,\cdot)\b e_q, \left\| \b x_1 - \b x_2\right\|_2 + |h_1| + |h_2|\right)}
	\end{align}
	depending on $\omega(\b e_q^\top \p_1^{\b r} \p_2^{\b r}K(\cdot,\cdot)\b e_q,\cdot)$, since
	\begin{align*}
		\left\| (\b x_1 + \beta h_1 \b e_p) - (\b x_2 + \beta h_2 \b e_p)\right\|_2 &\le \left\| \b x_1 - \b x_2\right\|_2 + \underbrace{\beta}_{\in [0,1]} |h_1 - h_2| \underbrace{\left\|\b e_p\right\|_2}_{=1} \\
		&\le \left\| \b x_1 - \b x_2\right\|_2 + |h_1| + |h_2|.
	\end{align*}
	Having derived the upper bound  \eqref{eq:bound_Delta} to control $\|\Delta_{h_1, \b x_1} - \Delta_{h_2, \b x_2}\|_\K$, let us choose $\b x_1=\b x_2=\b x \in \mathring{\X}$ and any sequence $(h_n)_{n\in \N} \subset [-\rho,\rho]\backslash\{0\}$ such that $h_n\xrightarrow{n\rightarrow \infty} 0$. In this case, \eqref{eq:bound_Delta} shows that $(\Delta_{h_n, \b x})_{n\in\N}$ is a Cauchy sequence in the Hilbert space $\FK$ so it converges by the completeness of $\FK$. Moreover \eqref{eq:bound_Delta} ensures that all the sequences --- independently of the choice of $(h_n)_{n\in \N}$ --- have the same limit which we denote formally by $\Delta_{0, \b x} \in \FK$. Since strong convergence in $\FK$ implies weak convergence, for any $\b f\in\FK$, we have
	\begin{align}
		\lim\limits_{h\rightarrow 0}\frac{\partial^{\b r}f_q(\b x+h\b e_p)-\partial^{\b r}f_q(\b x)}{h}=\lim\limits_{h\rightarrow 0}	\left<\b f, \Delta_{h, \b x}\right>_\K=\left<\b f, \Delta_{0, \b x}\right>_\K. \label{eq:f:r'}
	\end{align}Consequently $\p^{\b r'}f_q(\b x)$ exists.
	Moreover, by choosing $\b f= K(\cdot,\b x')\b e_{q'}$ in \eqref{eq:f:r'}, we deduce that $\Delta_{0, \b x}\in\FK$ equals to $\p_2^{\b r'}K(\cdot,\b x)\b e_q$ which establishes \eqref{eq:repr-prop_proof:2} and \eqref{eq:repr-prop_proof:3} for $\b r'$. The continuity of $\p^{\b r'}f_q(\b x)$ on $\mathring{\X}$, hence \eqref{eq:repr-prop_proof:1} for $\b r'$ follows from the Cauchy-Schwarz inequality
	\begin{align*}
		\left|\p^{\b r'}f_q(\b x_1)-\p^{\b r'}f_q(\b x_2)\right| &\le \|\b f\|_\K \|\Delta_{0, \b x_1} - \Delta_{0, \b x_2}\|_\K
	\end{align*}
	combined with \eqref{eq:bound_Delta}.
	
	\noindent$\bullet$ Case of $\X$: Let us consider an arbitrary point $\b x \in \X$. Then there exists a sequence $(\b x_n')_{n\in\N} \in (\mathring{\X})^\N$ converging to $\b x$ since $\X$ is contained in the closure of its interior. For any such sequence $(\b x_n')_{n\in\N}$, $(\Delta_{0, \b x_n'})_{n\in\N}$ is a Cauchy sequence by \eqref{eq:bound_Delta} applied with $h_1=h_2=0$ (hence convergent by the completeness of $\FK$), with the same limit which we again denote formally by $\Delta_{0, \b x}\in \FK$. Consequently,
	\begin{align*}
		\lim\limits_{\b x'\rightarrow \b x}\p^{\b r'}f_q(\b x')=\lim\limits_{\b x'\rightarrow \b x}	\left<\b f, \Delta_{0, \b x'}\right>_\K=\left<\b f, \Delta_{0, \b x}\right>_\K,
	\end{align*}
	so $\partial^{\b r'}f_q(\b x)$ exists and $\Delta_{0, \b x}\in \FK$ can be identified with $ \p_2^{\b r'}K(\cdot,\b x)\b e_q $ which establishes 
	\eqref{eq:repr-prop_proof:2} and \eqref{eq:repr-prop_proof:3} for $\b r'$. 	Let $f_q^{[\b r']}(\b x'):=\left<\b f, \Delta_{0, \b x'}\right>_\K$ for $\b x'\in\X$. Again by the Cauchy-Schwarz inequality, we obtain that $f_q^{[\b r']}$ is continuous on $\X$, and it is the continuous extension of $\p^{\b r'}f_q$ from $\mathring{\X}$ to $\X$. This proves \eqref{eq:repr-prop_proof:1} for $\b r'$ and concludes the induction.
	
	\subsubsection{Proof of Theorem~\ref{thm:inclusion}}\label{sec:proof:thm:inclusion}
	By the convex separation formula of \cite{dubovitskii65extremum}, the first statement is equivalent to the existence of $\b g_{\b f}, (\b g_{B,j})_{j\in [J_B]}, (\b g_{H,j})_{j\in [J_H]}\in\FK$ not vanishing simultaneously and satisfying
	\begin{align*}
		\inf\limits_{\b w \in \dH^{-}_{\K}\left(\b f-\b f_0, b_0-\bm{\Gamma} \b b\right)}\hspace{-0.1cm}\left<\b g_{\b f}, \b w\right>_{\K}+ \sum_{j \in [J_B]} \inf\limits_{\b w \in \dBB_{\K}(\b c_{j},r_{j})} \hspace{-0.1cm} \left<\b g_{B,j},\b w\right>_{\K} + \sum_{j \in [J_H]} \inf\limits_{\b w \in \dH^{-}_{\K} (\b v_{j},\rho_{j})} \hspace{-0.1cm} \left<\b g_{H,j},\b w\right>_{\K} &\geq 0,\\
		\b g_{\b f}+\sum_{j\in [J_B]}\b g_{B,j} + \sum_{j\in [J_H]}\b g_{H,j} &= \b 0.
	\end{align*}
	Since the sum of the infima is nonnegative, each infimum is finite. Hence by 
	Lemma~\ref{lemma:inf:balls} and Lemma~\ref{lemma:inf:half-spaces} we get that 
	the inclusion $\Omega \subseteq H^{+}_{\K}\left(\b f-\b f_0, b_0-\bm{\Gamma} \b b\right)$ holds if and only if there exist $[\xi_{\b f};\xi_1;\dots; \xi_{J_H}]\in\R_+^{J_H+1}$ and $(\b g_{B,j})_{j\in [J_B]}\in\FK^{J_B}$ not vanishing simultaneously (since $\xi_{\b{f}}=0$ $\Leftrightarrow$ $\b g_{\b f}=\b 0$, and  $\xi_j = 0$ $\Leftrightarrow$ $\b g_{H,j}= \b 0$) such that
	\begin{equation} \label{ineq-eq_thm_DM_xi_f}
		\begin{split}
			-\xi_{\b f}   b_0-\bm{\Gamma} \b b  + \sum_{j \in [J_B]} \left<\b g_{B,j},\b c_j\right>_{\K} - \sum_{j\in [J_H]} \xi_j \rho_j - \sum_{j\in [J_B]} r_j \left\|\b g_{B,j}\right\|_{\K}&\geq 0,\\
			-\xi_{\b f}(\b f-\b f_0) + \sum_{j\in [J_B]}\b g_{B,j} - \sum_{j \in [J_H]}\xi_j \b v_j &= \b 0. 
		\end{split}
	\end{equation}
	Let $\mathscr{V}=\Sp\left(\b f-\b f_0, \left\{\b c_j\right\}_{j\in [J_B]},\left\{\b v_j\right\}_{j\in [J_H]}\right)$ and $\b g_j= \text{proj}_{\mathscr{V}}(\b g_{B,j})$ where $\text{proj}_{\mathscr{V}}$ denotes the projection onto the subspace  
	$\mathscr{V}$. Since $\left<\b g_{B,j},\b c_j\right>_{\K}=\left<\b g_{j},\b c_j\right>_{\K}$ and $\left\|\b g_{j}\right\|_{\K}\le\left\|\b g_{B,j}\right\|_{\K}$, this family also satisfies \eqref{ineq-eq_thm_DM_xi_f}.
	Here, again $[\xi_{\b f};\xi_1;\dots; \xi_{J_H}]\in\R_+^{J_H+1}$ and $(\b g_{j})_{j\in [J_B]}\in\FK^{J_B}$ cannot all vanish. Indeed, if it were the case, then by 
	$\left<\b g_{B,j},\b c_j\right>_{\K}=\left<\b g_{j},\b c_j\right>_{\K}=0$, \eqref{ineq-eq_thm_DM_xi_f} would give $- \sum_{j\in [J_B]} r_j \left\|\b g_{B,j}\right\|_{\K}\geq 0$, so, since $r_j>0$ ($\forall j\in [J_B]$), $(\b g_{B,j})_{j\in [J_B]}$ would all vanish too.
	
	The nonnegative number $\xi_{\b f}$ cannot be zero since in this case either $H^{-}_{\K}\left(\b f-\b f_0, b_0-\bm{\Gamma} \b b\right)$ or $\Omega$ would be empty by  \eqref{ineq-eq_thm_DM_xi_f} \citep{dubovitskii65extremum}, both cases being excluded by assumption.
	Hence, we can divide
	\eqref{ineq-eq_thm_DM_xi_f} by $\xi_{\b f}>0$; replacing $\xi_j$ with $\xi_j/\xi_{\b f}$ and $\b g_{j}$ with $\b g_{j}/\xi_{\b f}$, the claimed equation \eqref{ineq-eq_thm_DM} follows.    
	
	\subsubsection{Proof of Theorem~\ref{thm:SDP}}\label{sec:proof:thm:SDP}
	In accordance with the r.h.s.\ of \eqref{def:eta_SDP} let us define
	\begin{align}
		\b g_{\b x,\b u}(\cdot):=\b u^\top \b D K(\cdot,\b x) \b u:=\sum_{p_1,\,p_2\in[P]} u_{p_1}u_{p_2}D_{p_1,\,p_2}K(\cdot,\b x) \in\FK, \label{eq:gxu}
	\end{align}
	where $\b x\in \X$ and $\b u\in \S^{P-1}$. Since $\Kcons\subseteq \bigcup_{m\in[M]} \BB_{\X}(\tilde{\b x}_{m},\delta_{m})$, for any $\b x\in \Kcons$ let us take $\tilde{\b x}_{m}$ for which $\|\b x-\tilde{\b x}_{m}\|_{\X}\le \delta_{m}$. Applying the reproducing formula \eqref{eq:repr-prop} and the Cauchy-Schwartz inequality, for any $\b f\in\Fk$ one gets the lower bound   
	\begin{align}
		\b u^\top \b D (\b f - \b  f_{0}) (\b x) \b u&=\langle \b f - \b  f_{0}, \b u^\top \b D K(\cdot,\b x) \b u\rangle_K\nonumber\\
		&= \b u^\top \b D (\b f - \b  f_{0}) (\tilde{\b x}_{m}) \b u +\langle \b f - \b  f_{0}, \b u^\top \left[\b D K(\cdot,\b x)-\b D K(\cdot,\tilde{\b x}_{m})\right] \b u\rangle_K\nonumber\\
		&\ge \b u^\top \b D (\b f - \b  f_{0}) (\tilde{\b x}_{m}) \b u -\|\b f - \b  f_{0}\|_K \left\|\b u^\top \left[\b D K(\cdot,\b x)-\b  D K(\cdot,\tilde{\b x}_{m})\right] \b u\right\|_K\nonumber\\
		&\ge \b u^\top \b D (\b f - \b  f_{0}) (\tilde{\b x}_{m}) \b u-\eta_{m,P} \|\b f - \b  f_{0}\|_K \label{eq:SDP:lower-bound}.
	\end{align}
	This means that for $(\b f,\b b)\in C_{P,\text{SOC}}$ and for any $\b u\in \S^{P-1}$,
	\begin{align*}
		\eta_{m,P}\|\b f-\b f_0\|_K \underbrace{\b u^\top \b u}_{=1} &\le \b u^\top \b D(\b f-\b f_0)(\tilde{\b x}_{m}) \b u + \b u^\top\diag(\bm{\Gamma}\b b- \b b_{0})\b u \\
		0&\le \b u^\top \b D(\b f-\b f_0)(\tilde{\b x}_{m}) \b u - \eta_{m,P}\|\b f-\b f_0\|_K + \b u^\top\diag(\bm{\Gamma}\b b- \b b_{0})\b u \\
		0&\hspace{-0.1cm}\stackrel{\eqref{eq:SDP:lower-bound}}{\le} \b u^\top \b D (\b f - \b  f_{0}) (\b x) \b u + \b u^\top\diag(\bm{\Gamma}\b b- \b b_{0})\b u,
	\end{align*}
	in other words, $(\b f,\b b)\in C_{P}$; this proves Theorem~\ref{thm:SDP}.
	
	\subsubsection{Proof of Lemma~\ref{lemma:eta_SDP_4Dtensor}} \label{sec:proof:lemma:eta_SDP_4Dtensor}
	Taking the square of the argument of the supremum in  \eqref{def:eta_SDP}, by \eqref{eq:gxu} we have 
	\begin{eqnarray}
		\lefteqn{\left\|\sum_{p_1,\, p_2\in[P]} u_{p_1}u_{p_2} D_{p_1,p_2}K(\cdot,\tilde{\b x}_{m}) - 
			\sum_{p_1,\, p_2\in[P]} u_{p_1}u_{p_2} D_{p_1,p_2}K(\cdot,\b x) \right\|_K^2  =}\nonumber\\
		&&= \left\| g_{\tilde{\b x}_m,\b u}-g_{\b x,\b u}\right\|_K^2  =
		\left\| g_{\tilde{\b x}_m,\b u}\right\|_K^2 
		+  \left\|g_{\b x,\b u}\right\|_K^2 
		-2 \left<g_{\tilde{\b x}_m,\b u},g_{\b x,\b u}\right>_K^2.\label{eq:eta-squared}
	\end{eqnarray}
	This means that it is sufficient to compute expressions of the form  $\left<g_{\b x',\b u},g_{\b x,\b u}\right>_K$ where $\b x', \b x \in \X$.
	\begin{align}
		\left<g_{\b x',\b u},g_{\b x,\b u}\right>_K & = \left< \sum_{p_1,\,p_2\in[P]} u_{p_1}u_{p_2}D_{p_1,\,p_2}K(\cdot,\b x'), 
		\sum_{p_1,\,p_2\in[P]} u_{p_1}u_{p_2}D_{p_1,\,p_2}K(\cdot,\b x)\right>_K \nonumber\\
		&= \sum_{p_1',\,p_2',\,p_1,\,p_2\in[P]} u_{p_1'}u_{p_2'} u_{p_1} u_{p_2} \underbrace{\left<D_{p_1',\,p_2'}K(\cdot,\b x'), D_{p_1,\,p_2}K(\cdot,\b x)\right>_K}_{\stackrel{(*)}{=} D_{p_1',p_2'}^\T D_{p_1,p_2}K(\b x',\b x)=\Ktens(\b x',\b x)_{p_1,p_2,p'_1,p'_2}}\nonumber\\
		&  = \left<\b u \otimes \b u,\Ktens(\b x',\b x)(\b u \otimes \b u)\right>_F.\label{eq:gx'uxu}
	\end{align}
	$(*)$ follows from the fact that for any point $\b x', \b x \in \X$ and differential operator $\tilde{D}, D\in O_{Q,s}$ with parameterization $D(\b f)(\b x) = \sum_{q\in [Q]}\beta_q D_{q,\b x} (f_q)(\b x)$ and $\tilde{D}(\b f)(\b x') = \sum_{q\in [Q]}\tilde{\beta}_q \tilde{D}_{q,\b x'} (f_q)(\b x')$
	\begin{align*}
		\left< \tilde{D} K(\cdot,\b x'), DK(\cdot,\b x)\right>_\K =\sum_{q,q'\in[Q]} \beta_{q}\tilde{\beta}_{q'} \b e_{q'}^\top \tilde{D}_{q',\b x'} D_{q,\b x}K(\b x', \b x) \b e_{q}=\tilde{D}^\top D K(\b x', \b x)
	\end{align*}
	as implied by the reproducing property \eqref{eq:repr-prop}. Combining
	\eqref{eq:eta-squared} and \eqref{eq:gx'uxu}  concludes the proof.
	
	\subsubsection{Proof of Theorem~\ref{thm:certificate}} \label{sec:proof:thm:certificate}
	
	\begin{itemize}[labelindent=0em,leftmargin=1em,topsep=0cm,partopsep=0cm,parsep=0cm,itemsep=2mm]
	    \item Admissible pair for \eqref{opt-cons}: By construction $C_{\app}\subseteq C$ (see \Cref{sec:constraints}) and $\hatFK\subseteq \FK$, so the admissible pair of \eqref{opt_cons_SOC} in (ii) also yields an admissible pair for \eqref{opt-cons}.
	    \item Existence of minimizers for \eqref{opt_cons_SOC} and  \eqref{opt-cons}:
	    We apply Theorem 3.2.5 by \citet{attouch14variational} which states that  coercive w-l.s.c.\ $\R\cup \{\infty\}$-valued functions on a reflexive Banach space (specifically on a Hilbert space) have a minimum point. Indeed,
	    $C_{\app}$ and $C$ are strongly closed convex subsets of $\FK\times \R^B$ by Lemma~\ref{lemma:closed_constraints}, so is $\hatFK$, hence all these sets are weakly closed \citep[Theorem 3.3.2]{attouch14variational}. Consequently their indicator functions are w-l.s.c.\ and, by (ii), the intersection of their domains is non-empty. The w-l.s.c.\ and coercive property of $\Lcal$ is preserved when adding indicator functions by the closedness of w-l.s.c.\ functions  w.r.t.\ addition and by the non-negativity of indicator functions, respectively. The proof concludes by noting that $\hatFK\times \R^B$ is a Hilbert space.
	    \item Certificate of optimality: Since $C_{\app}\subseteq C$, $\hatFK\subseteq \FK$ and $C \subseteq C_{\rel}$,\textsuperscript{\ref{footnote:relax-tighten}}  the certificate of optimality $v_{\text{relax}}\leq \bar{v} \leq v_{\app}$ follows.
	\end{itemize}
	 
	\subsubsection{Proof of Corollary~\ref{thm:aposteriori_bound}}\label{sec:proof:thm:aposteriori_bound}
		\begin{itemize}[labelindent=0em,leftmargin=1em,topsep=0cm,partopsep=0cm,parsep=0cm,itemsep=0mm]
		\item  Existence of $(\bar{\b f},\bar{\b b})$ and $\left(\bar{\b f}_{\app}, \bar{\b b}_{\app}\right)$: The existence of the solutions follow by the imposed assumptions which include the conditions required in the existence part of Theorem~\ref{thm:certificate}.
		\item Uniqueness of $(\bar{\b f},\bar{\b b})$ and $\left(\bar{\b f}_{\app}, \bar{\b b}_{\app}\right)$: The uniqueness of the solutions follows from the strong convexity of the w-l.s.c.\ $\Lcal$ on the non-empty sets $C$ and $C_{\app}$.
		\item A posteriori bound: We apply the result \citep[Proposition 3.23]{peypouquet15convex} that for any $\mu$-strongly convex proper function $\phi: Z\rightarrow \R\cup\{\infty\}$ over a normed vector space $Z$, attaining its minimum at $z^*$, we have $\phi(z)-\phi(z^*)\ge \frac{\mu}{2} \|z-z^*\|_Z^2$. Here we take $\phi=\Lcal+\chi_C$ and $Z=\Fk\times\R^B$. Consequently, since $\left(\bar{\b f}_{\app},\bar{\b b}_\app\right)\in C$ one derives the claimed bound \eqref{ineq_aposteriori} from
	\begin{align}
		v_{\app}-v_{\text{relax}} & \stackrel{(*)}{\ge} \Lcal\left(\bar{\b f}_{\app},\bar{\b b}_{\app}\right) - \bar{v}\nonumber\\&  = \Lcal\left(\bar{\b f}_{\app}, \bar{\b b}_{\app}\right) - \Lcal\left(\bar{\b f},\bar{\b b}\right) \ge \frac{\mu_{\b f}}{2}\left\|\bar{\b f}_{\app}-\bar{\b f}\right\|^2_K + \frac{\mu_{\b b}}{2} \left\|\b b_{\app}-\bar{\b b}\right\|_2^2, \label{eq:L-diff:lower-bound}
	\end{align}
	 where $(*)$ follows from $v_{\text{relax}}  \le \bar{v}$ (shown in Theorem~\ref{thm:certificate}).
		\end{itemize}	
	
	\subsubsection{Proof of Proposition~\ref{thm:apriori_bound}}\label{sec:proof:thm:apriori_bound}
	Recall that $\left(\bar{\b f},\bar{\b b}\right)$ satisfies \eqref{def_mixded_constraint} and that we assume $\dom(\Lcal(\bar{\b f},\cdot))=\RR^B$. Let $\eta_\infty$ be defined according to \eqref{def:eta_infty}, fix  $\bm{\beta}\in\R^B$ such that $\bm{\Gamma}_i \bm{\beta}>\b 0$ for all $i\in \mathcal{I}$, and define
	\begin{align}
		\tilde{\bm{\beta}} &:=\eta_\infty c_f \bm{\beta}=\eta_\infty \frac{\max_{i\in [I]} \left\|\bar{\b f} -\b f_{0,i}\right\|_K}{\min_{i\in [I],\,p\in P_i}\left(\bm{\Gamma}_i \bm{\beta}\right)_p} \bm{\beta}. \label{eq:tilde-b:def}
	\end{align}
	Applying $\bm{\Gamma}_i$ to \eqref{eq:tilde-b:def} results in the bound (used below)
	\begin{align}
		\bm{\Gamma}_i\tilde{\bm{\beta}} & = \eta_\infty \max_{i\in [I]} \left\|\bar{\b f} -\b f_{0,i}\right\|_K \underbrace{\frac{\bm \Gamma_i \bm \beta}{\min_{i\in [I],\,p\in P_i}\left(\bm{\Gamma}_i \bm{\beta}\right)_p}}_{\ge 1\, \Leftarrow\,\bm{\Gamma}_i \bm{\beta}>\b 0,\ \forall i \in [I]} \ge \eta_\infty\left\|\bar{\b f} -\b f_{0,i}\right\|_K \b 1_{P_i},\, \forall i\in[I]\label{eq:magic} 
	\end{align} 
	with $\b 1_{P_i}\in\R^{P_i}$ being the vector of ones. 
	
	Next we show that $\left(\bar{\b f},\bar{\b b}+\tilde{\bm{\beta}}\right)\in C_{\app}$.
	\begin{itemize}[labelindent=0em,leftmargin=1.2em,topsep=0.2cm,partopsep=0cm,parsep=0cm,itemsep=2mm]
		\item $\left(\bar{\b f},\bar{\b b}+\tilde{\bm{\beta}}\right) \in C^i_{P_i,\text{SOC}}$ for all $i\in\mathcal{I}_{\text{SOC}}$: Let $i\in\mathcal{I}_{\text{SOC}}$. Then for all $\b x \in \Kcons_i$  
		\begin{align}
			\eta_{i,m, P_i} \left\|\bar{\b f}-\b f_{0,i}\right\|_K \b I_{P_i} & \stackrel{(a)}{\preccurlyeq} 	\eta_\infty \left\|\bar{\b f}-\b f_{0,i}\right\|_K \b I_{P_i}
			\stackrel{(b)}{\preccurlyeq} \diag\left(\bm{\Gamma}_i\tilde{\bm{\beta}}  \right)\nonumber\\
			&\preccurlyeq \diag\left(\bm{\Gamma}_i\tilde{\bm{\beta}} \right) +  \underbrace{\Dsdp_i  \left(\bar{\b f} - \b  f_{0,i}\right)(\b x)+ \diag\left(\bm{\Gamma}_i\bar{\b b}- \b b_{0,i} \right)}_{\succcurlyeq \b 0_{P_i \times P_i}\, \text{ for }\forall \b x \in \Kcons_i\,\Leftarrow\, \left(\bar{\b f},\bar{\b b}\right)\in C}\nonumber\\
			&=  \Dsdp_i  \left( \bar{\b f} - \b  f_{0,i}\right)(\b x)+ \diag\left(\bm{\Gamma}_i\left(\bar{\b b} +\tilde{\bm{\beta}} \right)- \b b_{0,i} \right),\label{eq:f,b+beta}
		\end{align}
		where (a) comes from the definition of $\eta_{\infty}$ and 
		(b) follows from \eqref{eq:magic}. Since $\tilde{\b x}_{i,m}\in\Kcons_i$, \eqref{eq:f,b+beta} means that $(\bar{\b f},\bar{\b b}+\tilde{\bm{\beta}}) \in C^i_{P_i,\text{SOC}}$. 
		
		\item $\left(\bar{\b f},\bar{\b b}+\tilde{\bm{\beta}}\right)\in C^i_{1,\Omega}$ for all $i\in \mathcal{I}_{\Omega}$: Let $i\in \mathcal{I}_{\Omega}$. By the definition of $\eta_\infty$, $\Phi_{D_i}(\Kcons_i)\subseteq \bar{\Omega}_i \subseteq \Phi_{D_i}(\Kcons_i)+\BB_K(0,\eta_\infty)$. This inclusion with \eqref{eq:tightened-inclusion} means that for  $\left(\bar{\b f},\bar{\b b}+\tilde{\bm{\beta}}\right)\in C^i_{1,\Omega}$ to hold it is sufficient to prove that $\Phi_{D_i}(\Kcons_i)+\BB_K(0,\eta_\infty)\subseteq H^{+}_{\K}\left(\bar{\b f}-\b f_{0,i}, b_{0,i}-\bm{\Gamma}_i \left(\bar{\b b}+\tilde{\bm{\beta}}\right)\right)$. The latter holds since for any $\b x \in \Kcons_i$ and $\b g\in \BB_K(0,\eta_\infty)$ we have
		\begin{eqnarray*}
			\lefteqn{\langle \bar{\b f} - \b  f_{0,i}, D_i K(\cdot,\b x) +\b g\rangle_K+ \bm{\Gamma}_i(\bar{\b b}+\tilde{\bm{\beta}})-  b_{0,i}}\\
			&&\stackrel{(a)}{=} D_i  \left(\bar{\b f} - \b  f_{0,i}\right)(\b x) + \bm{\Gamma}_i\bar{\b b}-  b_{0,i} +\bm{\Gamma}_i\tilde{\bm{\beta}} + \langle \bar{\b f} - \b  f_{0,i},\b g\rangle_K \\
			&& \stackrel{(b)}{\ge} \underbrace{ D_i  \left(\bar{\b f} - \b  f_{0,i}\right)(\b x) + \bm{\Gamma}_i\bar{\b b}-  b_{0,i}}_{\ge 0 \text{ for } \forall \b x\in \Kcons_i\, \Leftarrow\, \left(\bar{\b f},\bar{\b b}\right)\in C} +\underbrace{\bm{\Gamma}_i\tilde{\bm{\beta}}-\left\|\bar{\b f} - \b  f_{0,i}\right\|_K \underbrace{\|\b g\|_K}_{\le \eta_\infty} }_{\ge 0\, \Leftarrow\, \eqref{eq:magic}} \ge 0.
		\end{eqnarray*}
		In (a) we applied the reproducing formula (Lemma~\ref{lemma:reproducing}), (b) follows from the Cauchy-Schwarz inequality. 
	\end{itemize}
	The proved relation $\left(\bar{\b f},\bar{\b b}+\tilde{\bm{\beta}}\right)\in C_{\app}$ implies that $\left(\bar{\b f},\bar{\b b}+\tilde{\bm{\beta}}\right)$ is admissible for \eqref{opt_cons_SOC} since $\bar{\b f}\in\FK=\hatFK$ since $\bar{\b b}+\tilde{\bm{\beta}}\in\dom(\Lcal(\bar{\b f},\cdot))=\RR^B$. Thus
	\begin{align}
		\Lcal	\left(\bar{\b f}_{\app}, \bar{\b b}_{\app}\right) - \Lcal\left(\bar{\b f},\bar{\b b}\right)  & \stackrel{(a)}{\leq}  \Lcal\left(\bar{\b f},\bar{\b b}+\tilde{\bm{\beta}}\right) - \Lcal\left(\bar{\b f},\bar{\b b}\right)
		\stackrel{(b)}{\leq} L_b \big\|\tilde{\bm{\beta}}\big\|_2 \stackrel{(c)}{\le} L_b \eta_\infty  c_f \big\| \bm{\beta}\big\|_2,\label{eq:L-diff:bound2}
	\end{align}
	where (a) follows from the fact that $\left(\bar{\b f}_{\app}, \bar{\b b}_{\app}\right)$ is an optimal solution of \eqref{opt_cons_SOC},  (b) is implied by the local Lipschitz property of $\Lcal$, and (c) holds by \eqref{eq:tilde-b:def}. This is what we wanted to prove.	
	
	\subsubsection{Proof of Proposition~\ref{prop:reproducing}}\label{sec:proof:prop:reproducing}
	
	\noindent \tb{Finite-dimensional description}:
	Let us consider the finite-dimensional subspace 
	\begin{align*}
		V &:= \Sp\Big( \{\b f_{0,i}\}_{i \in [I]}, \left\{ D_{n,j}^0K(\cdot,\b x_n)\right\}_{n\in [N], j\in J_n}, \{D^i_{p_1,p_2}K(\cdot,\tilde{\b x}_{i,m})\}_{i \in \mathcal{I}_{\text{SOC}},\, p_1,\, p_2\in[P_i],\, m\in [M_i]},\\
		& \hspace{1.7cm}\{\b c_{i,m,j}\}_{i \in \mathcal{I}_{\Omega},\,m\in [M_i],\ j\in[J_{B,i,m}], },\{\b v_{i,m,j}\}_{i \in \mathcal{I}_{\Omega},\,m\in [M_i],\,j\in[J_{H,i,m}]}\Big).
	\end{align*}
	Let $(\bar{\b f}_{\app},\bar{\b b}_{\app})$ be an optimal solution, which we decompose as $\bar{\b f}_{\app}=\b z + \b w$ where $\b z=\text{proj}_V\left(\bar{\b f}_{\app}\right)\in V$ and $\b w \in V^{\perp}$. We show that $\left(\b z,\bar{\b b}_{\app}\right)$ is then also an optimal solution. 
	\begin{itemize}[labelindent=0em,leftmargin=1.2em,topsep=0.2cm,partopsep=0cm,parsep=0cm,itemsep=2mm]
		\item $L\left(\bar{\b f}_{\app},\bar{\b b}_{\app}\right) = L\left(\b z,\bar{\b b}_{\app}\right)$:  By the linearity of the differential operators $D_{n,j}^0$, the reproducing property (Lemma~\ref{lemma:reproducing}) and the orthogonality of $\b w\in V^{\perp}$ and $D_{n,j}^0K(\cdot,\b x_n) \in V$, one gets 
		\begin{align*}
			D_{n,j}^0\left(\bar{\b f}_{\app})(\b x_n\right) & = D_{n,j}^0\left(\b z + \b w\right)(\b x_n) = D_{n,j}^0\left(\b z \right)(\b x_{n})  + \underbrace{D_{n,j}^0\left(\b w\right)(\b x_n)}_{\left< \b w, D_{n,j}^0K(\cdot,\b x_{n})\right>_{\K}=0}.
		\end{align*}
		This implies that the terms appearing in $L$ are the same for $\bar{\b f}_{\app}$ and for $\b z$, and hence $L\left(\bar{\b f}_{\app},\bar{\b b}_{\app}\right) = L\left(\b z,\bar{\b b}_{\app}\right)$.
		\item $R\left(\left\|\b z\right\|_{\K}\right) \le R\left(\left\|\bar{\b f}_{\app} \right\|_{\K}\right)$: This inequality follows from  
		$\left\|\b z\right\|_{\K} \le \left\|\bar{\b f}_{\app} \right\|_{\K}$ by the monotonicity of $R$.	
		\item $\left(\b z,\bar{\b b}_{\app}\right) \in C^i_{P_i,\text{SOC}}$ for all $i \in \mathcal{I}_{\text{SOC}}$: Let $i \in \mathcal{I}_{\text{SOC}}$. Similarly to the previous point, $D^i_{p_1,p_2} (\bar{\b f}_{\app})(\tilde{\b x}_{i,m}) = D^i_{p_1,p_2} (\b z)(\tilde{\b x}_{i,m})$ for all $p_1, p_2 \in [P_i]$ and $m\in [M_i]$, so the r.h.s.\ in the  inequalities in $C^i_{P_i,\text{SOC}}$ are the same for $\left(\bar{\b f}_{\app},\bar{\b b}_{\app}\right)$ and $\left(\b z,\bar{\b b}_{\app}\right)$. Considering the l.h.s.-s, $\left\|\b z - \b f_{0,i}\right\|_{\K} \le \left\|\bar{\b f}_{app} -\b f_{0,i} \right\|_{\K}$ since by the Pythagorean theorem $ \left\| \bar{\b f}_{app}  - \b f_{0,i}\right\|_{\K}^2=\left\|\b z - \b f_{0,i}\right\|_{\K}^2 + \left\|\b w\right\|_{\K}^2$. This shows that $\left(\b z,\bar{\b b}_{\app}\right) \in C^i_{P_i,\text{SOC}}$ for all $ i \in \mathcal{I}_{\text{SOC}}$.
		\item $(\b z,\bar{\b b}_{\app}) \in C^i_{1,\Omega}$ for all $ i \in \mathcal{I}_{\Omega}$: By  \eqref{eq:inclusion:m-fixed} it is sufficient to prove that $\Omega_{i,m}\subseteq H^{+}_{\K}(\b z - \b f_{0,i}, b_{0,i}-\bm{\Gamma}_i\b b)$ for all $i \in \mathcal{I}_{\Omega}$ and $m \in [M_i]$. In the following the $i$ and $m$ indices are assumed to be fixed; in the notations we make them implicit. By Theorem \ref{thm:inclusion}, we have to show the existence of $J_B$ functions $(\b g'_j)_{j\in [J_B]} \subset \Sp\left(\b z-\b f_0, \left\{\b c_j\right\}_{j\in [J_B]},\left\{\b v_j\right\}_{j\in [J_H]}\right)$ and $J_H$ non-negative coefficients $(\xi'_j)_{j\in [J_H]} \in \Rnn^{J_H}$ satisfying \eqref{ineq-eq_thm_DM}. 
		Consider  $(\b g_j)_{j\in [J_B]}$ and $(\xi_j)_{j\in [J_H]}$ for which \eqref{ineq-eq_thm_DM} holds for $\left(\bar{\b f}_{\app},\bar{\b b}_{\app}\right)$. Let us define $\b g'_j:=\text{proj}_V(\b g_j)$ (in other words, $\b g_j = \b g_j'+\b g_j'^\perp$ with $\b g_j' \in V$, $\b g_j'^\perp \in V^\perp$) and $\xi'_j:=\xi_j\in \R_+$.
		With this choice of $(\b g'_j)_{j\in [J_B]}$ and $(\xi'_j)_{j\in [J_H]}$, the pair $\left(\b z,\bar{\b b}_{\app}\right)$ satisfies \eqref{ineq-eq_thm_DM}. 
		Indeed, the inequality in  \eqref{ineq-eq_thm_DM} holds by 
		\begin{align*}
			\left<\b g_j,\b c_j\right>_K & = \left<\b g_j' + \b g_j'^\perp,\b c_j\right>_K \hspace{-0.07cm} = \left<\b g_j',\b c_j\right>_K + \big<\underbrace{\b g_j'^\perp}_{\in V^\perp},\underbrace{\b c_j}_{\in V}\big>_K \hspace{-0.07cm} = \left<\b g_j',\b c_j\right>_K,\, \left\|\b g_j'\right\|_K\le \left\|\b g_j\right\|_K.
		\end{align*}
		The equality in \eqref{ineq-eq_thm_DM} is satisfied since
		\begin{align*}
			\b 0 & = \proj_V(\b 0) = \proj_V
			\left(-\left(\bar{\b f}_{app}-\b f_0\right) + \sum_{j\in [J_B]}\b g_j - \sum_{j \in [J_H]}\xi_j \b v_j\right)\\
			&= -\Big[\underbrace{\proj_V\left(\bar{\b f}_{app}\right)}_{=\b z} - \underbrace{\proj_V\left(\b f_0\right)}_{=\b f_0\, \Leftarrow\, \b f_0 \in V}\Big] + \sum_{j\in [J_B]} \underbrace{\proj_V(\b g_j)}_{=\b g_j'} - \sum_{j \in [J_H]}\underbrace{\xi_j}_{=\xi_j'} \underbrace{\proj_V(\b v_j)}_{=\b v_j\, \Leftarrow\, \b v_j \in V}.
		\end{align*}
		Finally let us fix any $j\in [J_B]$
		and show that $\b g'_j \in \Sp\left(\b z-\b f_0, \left\{\b c_i\right\}_{i\in [J_B]},\left\{\b v_i\right\}_{i\in [J_H]}\right)$. This relation follows from 
		\begin{align*}
			\b g_j &\in \Sp\left(\b f-\b f_0, \left\{\b c_i\right\}_{i\in [J_B]},\left\{\b v_i\right\}_{i\in [J_H]}\right) \Rightarrow \exists a \in \R, (b_i)_{i\in [J_B]} \in \R^{J_B}, (d_i)_{i\in [J_H]} \in \R^{J_H}\text{ s.t.}\\
			\b g_j &= a (\b f-\b f_0) + \sum_{i\in [J_B]} b_i \b c_i
			+ \sum_{i\in [J_H]} d_i \b v_i \Rightarrow\\
			\b g_j' &= \proj_V(\b g_j) = \proj_V\left(a (\b f-\b f_0) + \sum_{i\in [J_B]} b_i \b c_i
			+ \sum_{i\in [J_H]} d_i \b v_i \right)\\
			&= a \Big[\underbrace{\proj_V(\b f)}_{=\b z} - \underbrace{\proj_V(\b f_0)}_{=\b f_0\,\Leftarrow\, \b f_0\in V}\Big] + 
			\sum_{i\in [J_B]} b_i \underbrace{\proj_V(\b c_i)}_{=\b c_i\, \Leftarrow\, \b c_i \in V} + \sum_{i\in [J_H]} d_i \underbrace{\proj_V(\b v_i)}_{=\b v_i\, \Leftarrow\, \b v_i \in V}.
		\end{align*}
	\end{itemize}
	This means that $(\b z,\bar{\b b}_{\app}) \in C_{\app}$ and that it is necessarily optimal.\\
	
	\noindent \tb{w-l.s.c. and coercivity of $\Lcal_S+\chi_{C_{\app}}$:} We use properties of compositions of l.s.c.\ maps. Let $\Ltilde\left(\b f,\b b\right):= L\left(\b b,\left(\left(D^0_{n,j}(\b f)(\b x_n)\right)_{j\in J_n}\right)_{n\in [N]}\right)$. The function $\Ltilde$ is w-l.s.c.\ since it is the composition of the l.s.c.\ $L$ (Assumption~(iv)) with the continuous maps $\b f \mapsto D_{n,j}^0 (\b f)(\b x_{n})$; similarly, $\tilde{R}\left(\b f\right):=R\left(\|\b f\|_K\right)$ is w-l.s.c.\ as $\left\|\cdot\right\|_K$ is w-l.s.c.\ and $R$ is monotone. Indeed, let us take any $(\b f,\b b)\in\dom(\Ltilde)$ and any $(\b f^{(k)},\b b^{(k)})$ weakly converging to $(\b f,\b b)$ then the w-l.s.c.\ properties follow from
		\begin{align}
		L\left(\b{b},\left(\left(D_{n,j}^0 \left(\b f \right)(\b x_{n})\right)_{j\in J_n}\right)_{n\in [N]} \right) & \stackrel{(a)}{\le} \liminf\limits_{k\rightarrow \infty}L\left(\b b^{(k)},\left(\left(D_{n,j}^0 \left(\b f^{(k)}\right)(\b x_{n})\right)_{j\in J_n}\right)_{n\in [N]} \right),\label{eq:L-bound}\\
		R\left(\left\|\b f \right\|_K\right) &\stackrel{(b)}{\le }    R\left(\liminf\limits_{k\rightarrow \infty}\left\|\b f^{(k)} \right\|_K\right) \stackrel{(c)}{\le }  \liminf\limits_{k\rightarrow \infty} R\left(\left\|\b f^{(k)} \right\|_K\right),\label{eq:R-bound}
		\end{align}
		where (a) follows from the fact that in finite-dimensional Euclidean spaces strong and weak convergence coincide, from the lower semi-continuity of $L$ (Assumption~(iv)), and by the fact that $\lim_{k\rightarrow \infty }D_{n,j}^0 \left(\b f^{(k)}\right)(\b x_{n}) = D_{n,j}^0 \left(\b f\right)(\b x_{n})$ ($\forall n\in [N]$ and $j\in J_n$); the latter is implied by the weak convergence of $\left(\b f^{(k)}\right)_{k\in \N}$ to $\b f$ and the reproducing property (Lemma~\ref{lemma:reproducing}). (b) comes from the weak l.s.c.\ property of $\left\|\cdot\right\|_K$ and the monotonicity of $R$. The l.s.c.\ property  of $R$ (Assumption~(iv))  gives (c): by the definition of the $\liminf$, there exists a subsequence $\left(\b f^{(k_n)}\right)_{n\in \N}$ such that $\liminf_{k\rightarrow \infty}\left\|\b f^{(k)}\right\|_K = \lim_{n\rightarrow \infty} \left\|\b f^{(k_n)}\right\|_K$ and $R\left(\lim_{n\rightarrow \infty} \left\|\b f^{(k_n)}\right\|_K\right)\le \liminf_{n\rightarrow \infty} R\left(\left\|\b f^{(k_n)}\right\|_K\right)$ as $R$ is l.s.c.; the reasoning  can be restricted w.l.o.g.\ to the subsequence $\left(\b f^{(k_n)}\right)_{n\in \N}$.
	
	From Lemma~\ref{lemma:closed_constraints}, it follows that $\chi_{C_{\app}}$ is w-l.s.c. Hence $\Lcal_S+\chi_{C_{\app}}=\Ltilde+\tilde{R}+\chi_{C_{app}}$ is w-l.s.c.\ as a sum of w-l.s.c.\ functions. Moreover, since $\Ltilde+\chi_{C_{\app}}$ is lower bounded by Assumption~(iii) and uniformly coercive in $\b b$ by Assumption~(ii), while $\tilde{R}$ is coercive in $\b f$ by Assumption~(i), we obtain that $\Lcal_S+\chi_{C_{\app}}$ is coercive in $(\b f, \b b)$ as a sum of lower-bounded coercive functions in their arguments.
	
	\subsubsection{Proof of Theorem~\ref{thm:soap_bubble}}\label{sec:proof:thm:soap_bubble}
	\tb{Part 1 (limit covering):} The properties we exploit are that $\diam\left(\Omega^{(0)}\right) <\infty$ and that the diameters of the bursting sets decrease by a factor of $\gamma$. Recall that at the $k^{th}$ iteration of Alg.~\ref{alg:soap_Omega} we have
	\begin{align}
		\bm \Phi_D(\Kcons) \subseteq\bar{\Omega}^{(k)} = \cup_{m\in \left[M^{(k)}\right]} \bar{\Omega}_m^{(k)}\subseteq H^{+}_{\K}\left(\b f^{(k)} - \b f_0,b_{0}-\bm{\Gamma}\b b^{(k)}\right). \label{eq:k-th-it:summary}
	\end{align}
	We say that a set $\bar{\Omega}_m^{(j)}$ present at the $j^{th}$ iteration is $k$-persistent if $j\le k$ and $\bar{\Omega}_m^{(j)}$ does not burst at all in Alg.~\ref{alg:soap_Omega}. Let us define 
	\begin{align}
		\bar{\Omega}_{\text{pers}}^{(k)} & \subseteq\bar{\Omega}^{(k)} \subseteq \FK \label{eq:Omega-k-pers}
	\end{align}
	as the union of the $k$-persistent sets. By definition one gets an increasing sequence of sets ($\bar{\Omega}_{\text{pers}}^{(1)}\subseteq \bar{\Omega}_{\text{pers}}^{(2)}\subseteq \bar{\Omega}_{\text{pers}}^{(3)}\subseteq \ldots$), hence we can take the closed limit of these sets and define  $\bar{\Omega}^{(\infty)}_{\text{pers}}:=\overline{\bigcup_{k\in\N}\bar{\Omega}_{\text{pers}}^{(k)}}$. We show that
	\begin{align}
		\lim\limits_{k\rightarrow \infty} \bar{\Omega}^{(k)} = \bm \Phi_D(\Kcons)\cup\bar{\Omega}^{(\infty)}_{\text{pers}}. \label{eq:lim-Omega}
	\end{align}
	Notice that by definition $\bar{\Omega}^{(k)}$ is a closed and bounded set. The set $\bm \Phi_D(\Kcons)$ is  compact (thus closed and bounded) as $\Phi_D$ is continuous and $\Kcons$ is compact. The set $\bar{\Omega}^{(\infty)}_{\text{pers}}$ is closed by definition; it is also bounded as $\bar{\Omega}^{(\infty)}_{\text{pers}} \subseteq \bm \Phi_D(\Kcons)+ \BB_{\K}\left(\b 0,\diam\left(\Omega^{(0)}\right)\right)$.
	Hence the terms in \eqref{eq:lim-Omega} are elements of the complete \citep{price40completeness} metric space of closed, bounded, non-empty sets of $\FK$ equipped with the Hausdorff distance
	\begin{align*}
		d_{\text{H}}(S_1,S_2) & = \inf\left\{\epsilon>0 \,:\, S_1 \subseteq S_2+ \BB_{\K}(\b 0,\epsilon) \text{ and } S_2 \subseteq S_1+ \BB_{\K}(\b 0,\epsilon)\right\}
	\end{align*}
	where '+' denotes the Minkowski sum. The limit in \eqref{eq:lim-Omega} is meant in this $d_{\text{H}}$ sense.
	
	Indeed \eqref{eq:lim-Omega} can be proved as follows. 
	Let $\Cp$ denote the  covering elements of the $k^{th}$ iteration  that are not $k$-persistent; in other words, each of these sets $\bar{\Omega}^{(k)}_m$ will burst after  $N^{(k)}_m\in \N$ iterations. Let  $A^{(k)}\subseteq \R_+$ be the finite set of the diameters of the elements in $\Cp$ and $\alpha^{(k)}:=\max\left(A^{(k)}\right)$. Since at each iteration, the diameters can only decrease, $\left(\alpha^{(k)}\right)_{k\in\N}$ is a non-negative decreasing sequence which thus converges to some $\alpha\in\R_+$. We show that $\alpha=0$ by contradiction. Assume that $\alpha>0$, and take $k$ such that $0\le \alpha^{(k)}- \alpha < (1-\gamma)\alpha$ which is possible since $\alpha>0$ and $\gamma \in (0,1)$. As $\alpha \le \alpha^{(k)}$, this choice of $k$ implies that $\alpha^{(k)}- \alpha< (1-\gamma)\alpha^{(k)}$, in other words that 
	$\gamma \alpha^{(k)} < \alpha$.
	By taking $N^{(k)}:=\max_{m} N^{(k)}_m$, we get that
	\begin{align*}
		\alpha^{\left(k+N^{(k)}\right)}\le \gamma \alpha^{(k)} < \alpha.
	\end{align*}
	However, the obtained relation $\alpha^{\left(k+N^{(k)}\right)}<\alpha$  contradicts the fact that $\left(\alpha^{(k)}\right)_{k\in\N}$ converges decreasingly to $\alpha$; this contradiction establishes that $\alpha=0$. 
	
	We have that 
	\begin{align}
		\bar{\Omega}_{\text{pers}}^{(k)} \cup \bm \Phi_D(\Kcons) \stackrel{(a)}{\subseteq}\bar{\Omega}^{(k)} \stackrel{(b)}{\subseteq} \bar{\Omega}_{\text{pers}}^{(k)}\cup \left(\bm \Phi_D(\Kcons)+ \BB_{\K}\left(\b 0,\alpha^{(k)}\right)\right), \label{eq:incl-Omega}
	\end{align} 
	The inclusion (a) holds since $\bm \Phi_D(\Kcons) \subseteq\bar{\Omega}^{(k)}$ by \eqref{eq:k-th-it:summary} and 
	$\bar{\Omega}_{\text{pers}}^{(k)} \subseteq\bar{\Omega}^{(k)}$ by \eqref{eq:Omega-k-pers}, while (b) holds given that at each iteration $k$, $\bar{\Omega}_m^{(k)}\cap \bm  \Phi_D(\Kcons)\neq\emptyset$ for any $m$  (recall that superfluous covering elements were discarded in Alg.~\ref{alg:covering}).
	This means by the previously proved relation  $\lim\limits_{k\rightarrow \infty} \alpha^{(k)}=0$ that 
	\begin{align*}
		\lim\limits_{k\rightarrow \infty}\bar{\Omega}_{\text{pers}}^{(k)} \cup \bm \Phi_D(\Kcons) &= \lim\limits_{k\rightarrow \infty}\bar{\Omega}^{(k)} =  \lim\limits_{k\rightarrow \infty} \bar{\Omega}_{\text{pers}}^{(k)}\cup \left(\bm \Phi_D(\Kcons)+ \BB_{\K}\left(\b 0,\alpha^{(k)}\right)\right) = \bm \Phi_D(\Kcons)\cup\bar{\Omega}^{(\infty)}_{\text{pers}}   
	\end{align*}
	in Hausdorff distance sense; this establishes
	\eqref{eq:lim-Omega}.
	
	Let $\Theta^{(k)}:=\bar{\Omega}^{(k)}\setminus \bar{\Omega}_{\text{pers}}^{(k)}$. Since the constraints associated to $\bar{\Omega}_{\text{pers}}^{(k)}$ are never active by definition, they can be removed from the problem:
	\begin{align}\label{eq:Omega-Theta}
		\left(\b f^{(k)},\b b^{(k)}\right)\in \argmin_{\substack{\b f\,\in\,\hatFK,\, \b b\,\in\,R^B\\ \bar{\Omega}^{(k)}\subseteq H^{+}_{\K}(\b f - \b f_0,b_{0}-\bm{\Gamma}\b b)}}\Lcal(\b f,\b b)=\argmin_{\substack{\b f\,\in\,\hatFK,\, \b b\,\in\,R^B\\ \Theta^{(k)}\subseteq H^{+}_{\K}(\b f - \b f_0,b_{0}-\bm{\Gamma}\b b)}}\Lcal(\b f,\b b).
	\end{align}
	However by \eqref{eq:incl-Omega} and by using the fact that $(A\cup B)\backslash B \subseteq A$ for any sets $A,B$, we have that $\bar{\Theta}^{(\infty)}:=  \overline{\lim\limits_{k\rightarrow \infty} \Theta^{(k)}} \subseteq \bm \Phi_D(\Kcons)$, where the limit is again meant in Hausdorff distance sense. Hence, considering the limit constraint sets in \eqref{eq:Omega-Theta}, any 
	\begin{align}\label{eq:Omega-Theta-temp}
		\left(\b f^{(\infty)},\b b^{(\infty)}\right)\in \argmin_{\substack{\b f\,\in\,\hatFK,\, \b b\,\in\,\R^B\\ \bar{\Omega}^{(\infty)}\subseteq H^{+}_{\K}(\b f - \b f_0,b_{0}-\bm{\Gamma}\b b)}}\Lcal(\b f,\b b)=\argmin_{\substack{\b f\,\in\,\hatFK,\, \b b\,\in\,\R^B\\ \bar{\Theta}^{(\infty)}\subseteq H^{+}_{\K}(\b f - \b f_0,b_{0}-\bm{\Gamma}\b b)}}\Lcal(\b f,\b b)
	\end{align}
	is the solution of both a tightening ($\bar{\Omega}^{(\infty)}\supseteq \Phi_D(\Kcons)$)   and a relaxation ($\bar{\Theta}^{(\infty)}\subseteq \Phi_D(\Kcons)$)  of the original problem; hence
	\begin{align}
		\left(\b f^{(\infty)},\b b^{(\infty)}\right) &\in
		\argmin_{\substack{\b f\,\in\,\hatFK,\, \b b\,\in\,\R^B\\ \bm \Phi_D(\Kcons)\subseteq H^{+}_{\K}(\b f - \b f_0,b_{0}-\bm{\Gamma}\b b)}}\Lcal(\b f,\b b)  
		\label{eq:f-b-inf}.
	\end{align}
	This establishes the first statement of Theorem \ref{thm:soap_bubble}.
	
	\noindent\tb{Part 2 (convergence of $\left(\b f^{(k)},\b     b^{(k)}\right)_{k\in\N}$):}  Suppose that Assumptions~(i)-(iv) hold. 
	\begin{itemize}[labelindent=0em,leftmargin=1em,topsep=0cm,partopsep=0cm,parsep=0cm,itemsep=2mm]
		\item Existence of $\left(\b f^{(k)},\b     b^{(k)}\right)_{k\in \N}$: First we prove the existence of the iterates $\left(\b f^{(k)},\b     b^{(k)}\right)$ by induction over $k$. For $k=0$, the  existence of $(\b f^{(0)},\b b^{(0)})$ is guaranteed by  Assumptions~(i)-(ii) and Theorem~\ref{thm:certificate}. Suppose we reached the $k^{th}$ step, and let $d_k:=d_H\left(\bar{\Omega}^{(0)}, \bar{\Omega}^{(k)}\right)$. Let us recall that $\left(\hat{\b f}, \hat{\b b}\right)$ is an admissible pair  for $\Psc\left(\bar{\Omega}^{(0)}\right)$ (see Assumption (ii)), and let us define $\hat{\b b}_k:=\hat{\b b}+\frac{d_k \left\|\hat{\b f}- \b  f_{0}\right\|_K}{\|\bm{\Gamma}\|_2^2} \bm{\Gamma}^\top \in \R^B$ which exists since $\bm{\Gamma}\neq \b 0$ by Assumption~(iv). With this choice, we show that
		\begin{align}
			\bar{\Omega}^{(k)} \subseteq H^{+}_{\K}\left(\hat{\b f} - \b f_0,b_{0}-\bm{\Gamma}\hat{\b b}_k\right). 
			\label{eq:Omega-k:containing}
		\end{align}
		Indeed, by the definition of the Hausdorff distance for any $\b g\in \bar{\Omega}^{(k)}$ there exists some $\b u\in \BB_{\K}(\b 0,1)$ and $\b g_0\in \bar{\Omega}^{(0)}$ such that $\b g = \b g_0 + d_k \b u$. This implies  \eqref{eq:Omega-k:containing} as
		\begin{eqnarray}
			\lefteqn{\left< \hat{\b f} - \b  f_{0}, \b g\right>_K+ \bm{\Gamma}\hat{\b b}_k-  b_{0} = 
				\left< \hat{\b f} - \b  f_{0}, \b g_0 + d_k \b u\right>_K+ \bm{\Gamma}\left( \hat{\b b}+\frac{d_k \left\|\hat{\b f}- \b  f_{0}\right\|_K}{\|\bm{\Gamma}\|_2^2} \bm{\Gamma}^\top \right)-  b_{0}} \nonumber\\ &&=\underbrace{\left< \hat{\b f} - \b  f_{0}, \b g_0\right>_K+ \bm{\Gamma}\hat{\b b}-  b_{0}}_{\ge 0 \text{ by Assumption (ii)}} + \hspace{-0.1cm} \underbrace{d_k \left< \hat{\b f} - \b  f_{0}, \b u\right>_K + d_k\left\|\hat{\b f}- \b  f_{0}\right\|_K}_{\ge 0\text{ by }\b u\in \BB_K(\b 0,1)\text{ and the Cauchy-Schwartz inequality}} \hspace{-0.7cm}\ge 0.\label{eq:omega-cover_dk}
		\end{eqnarray}
		\eqref{eq:Omega-k:containing} means that $\left(\hat{\b f}, \hat{\b b}_k\right)$ is admissible for $\Psc\left(\bar{\Omega}^{(k)}\right)$ as $\hat{\b b}_k\in \dom\left(\Lcal\left(\hat{\b f},\cdot\right)\right)=\RR^B$ by Assumption~(iii). The existence of $\left(\b f^{(k)},\b b^{(k)}\right)$ follows from the proved admissibility of $\left(\hat{\b f}, \hat{\b b}_k\right)$ and since the conditions of Theorem~\ref{thm:certificate} hold.
		\item Boundedness of $\left(\b f^{(k)},\b b^{(k)}\right)_{k\in\N}$:
		Let us define the bound $d_{\text{max}}:=\sup_{k\in\N} d_k<\infty$ with $d_k = d_H\left(\bar{\Omega}^{(0)},\bar{\Omega}^{(k)}\right)$;  $d_{\text{max}}$ exists since $\left(\bar{\Omega}^{(k)}\right)_{k\in\N}$ converges as it was proved in \eqref{eq:lim-Omega}. Let $\hat{\b b}_{\text{max}}:=\hat{\b b}+\frac{d_{\text{max}} \left\|\hat{\b f}- \b  f_{0}\right\|_K}{\|\bm{\Gamma}\|_2^2} \bm{\Gamma}^\top$. Then $\left(\hat{\b f}, \hat{\b b}_{\text{max}}\right)$ is admissible for $\Psc\left(\bar{\Omega}^{(k)}\right)$ for all $k\in\N$ by a  computation analogous to \eqref{eq:Omega-k:containing}-\eqref{eq:omega-cover_dk} and by using Assumption~(iii). This admissibility means that $\Lcal\left(\b f^{(k)},\b b^{(k)}\right) \le \Lcal\left(\hat{\b f}, \hat{\b b}_{\text{max}}\right)$ , in other words $\left\{\left(\b f^{(k)},\b b^{(k)}\right)\right\}_{k\in\N} \subseteq \Lcal^{-1}\left(\left(-\infty, \Lcal\left(\hat{\b f}, \hat{\b b}_{\text{max}}\right)\right]\right)=:S$. The set $S$ is closed and bounded as Assumption~(i) states the coercivity of $\Lcal$ over the Hilbert space $\FK\times \R^B$ equipped with the sum of the inner products. By the boundedness of $\left(\b f^{(k)},\b b^{(k)}\right)_{k\in\N}$, it has a weakly converging subsequence (w.l.o.g. it is the sequence itself) to some $\left(\bar{\b f}_{\app},\bar{\b b}_{\app}\right)\in \hatFK\times \R^B$.
		\item $\left(\bar{\b f}_{\app},\bar{\b b}_{\app}\right)$ is admissible for $\Psc\left(\bar{\Omega}^{(\infty)}\right)$:  Next we show that $\left(\bar{\b f}_{\app},\bar{\b b}_{\app}\right)$ is admissible for $\Psc\left(\bar{\Omega}^{(\infty)}\right)$. Indeed,  let $\epsilon>0$. Then for any 
		$\b g\in\bar{\Omega}^{(\infty)}$, one can find $k\in \N$,  $\left(\b f^{(k)},\b b^{(k)}\right)$ and $\b g_k\in\bar{\Omega}^{(k)}$ such that 
		\begin{align}
			\left|\left< \b f^{(k)} - \b  f_{0}, \b g_k-\b g\right>_K\right| + \left|\left< \b f^{(k)} - \bar{\b f}_{\app}, \b g\right>_K\right|+ \left|\bm{\Gamma}\left(\b b^{(k)}-\bar{\b b}_{\app}\right)\right|\le \epsilon \label{eq:3terms}
		\end{align}
		using the boundedness of $\left(\b f^{(k)}\right)_{k\in \N}$ and the convergence of $\left(\bar{\Omega}^{(k)}\right)_{k\in\N}$ to  $\bar{\Omega}^{(\infty)}$ in Hausdorff distance (in the first term), and the weak convergence of $\left(\b f^{(k)},\b b^{(k)}\right)_{k\in\N}$ to $\left(\bar{\b f}_{\app},\bar{\b b}_{\app}\right)$ (in the 2nd and the 3rd terms). Notice that 
		\begin{eqnarray}
			\lefteqn{b_{0}-\bm{\Gamma}\bar{\b b}_{\app}- \bm{\Gamma}\left(\b b^{(k)}-\bar{\b b}_{\app}\right) = b_{0}-\bm{\Gamma}\b b^{(k)} \stackrel{(a)}{\le}\left< \b f^{(k)} - \b  f_{0}, \b g_k\right>_K} \nonumber\\
			&&= \left< \b f^{(k)} - \b  f_{0}, \b g_k-\b g\right>_K + \left< \b f^{(k)} - \bar{\b f}_{\app}, \b g\right>_K+\left< \bar{\b f}_{\app} - \b  f_{0}, \b g\right>_K, \label{eq:admi}
		\end{eqnarray}
		where (a) holds since $\b g_k\in\bar{\Omega}^{(k)}$. Rearranging \eqref{eq:admi} leads to
		\begin{align*}
			\left< \bar{\b f}_{\app} - \b  f_{0}, \b g\right>_K & \ge b_{0}-\bm{\Gamma}\bar{\b b}_{\app} - \bm{\Gamma}\left(\b b^{(k)}-\bar{\b b}_{\app}\right) -
			\left< \b f^{(k)} - \b  f_{0}, \b g_k-\b g\right>_K - \left< \b f^{(k)} - \bar{\b f}_{\app}, \b g\right>_K\\
			& \stackrel{\eqref{eq:3terms}}{\ge} b_{0}-\bm{\Gamma}\bar{\b b}_{\app}  - \epsilon. 
		\end{align*}
		Taking the limit $\epsilon\rightarrow 0$, we get that $\left(\bar{\b f}_{\app},\bar{\b b}_{\app}\right)$ is admissible for $\Psc \left(\bar{\Omega}^{(\infty)}\right)$.
		
		\item $\left(\bar{\b f}_{\app},\bar{\b b}_{\app}\right)$ is an optimal solution of $\Psc\left(\bar{\Omega}^{(\infty)}\right)$: Fix $\xi>0$. Since $\left(\b f^{(k)},\b b^{(k)}\right)_{k\in\N}$ weakly converges to $\left(\bar{\b f}_{\app},\bar{\b b}_{\app}\right)$, and, by Assumption~(i), $\Lcal$ is weakly l.s.c.\ so there exists some $N_0 \in \N$ such that
		\begin{align}
			\Lcal\left(\b f^{(k)},\b b^{(k)}\right)\ge\Lcal\left(\bar{\b f}_{\app},\bar{\b b}_{\app}\right) -\xi \text{ for all $k\ge N_0$.}
			\label{eq:k-N_0}
		\end{align}
		Consider an arbitrary pair $\left(\b f, \b b\right)$ admissible for $\Psc\left(\bar{\Omega}^{(\infty)}\right)$. Let us define $c_{\b f}:=\frac{\|\b f- \b  f_{0}\|_K}{\|\bm{\Gamma}\|_2^2} \bm{\Gamma}^\top$ and $\epsilon_k:=d_H\left(\bar{\Omega}^{(\infty)}, \bar{\Omega}^{(k)}\right)$. A computation similar to \eqref{eq:omega-cover_dk}  combined with Assumption~(iii) implies that $(\b f, \b b + \epsilon_k c_{\b f})$ is admissible for $\Psc(\bar{\Omega}^{(k)})$ for all $k\in\N$, and that 
		\begin{align}
			\Lcal(\b f, \b b + \epsilon_kc_{\b f})\ge \Lcal\left(\b f^{(k)},\b b^{(k)}\right) \stackrel{(a)}{\ge}\Lcal\left(\bar{\b f}_{\app},\bar{\b b}_{\app}\right) -\xi,
			\label{eq:-psi}
		\end{align}
		where (a) holds by \eqref{eq:k-N_0} for $k\ge N_0$. This inequality shows that 
		\begin{align}
			\Lcal\left(\b f, \b b\right)\ge\Lcal\left(\bar{\b f}_{\app},\bar{\b b}_{\app}\right) -\xi
			\label{eq:xi-now}
		\end{align}
		by taking in \eqref{eq:-psi} the limit $k\rightarrow \infty$ (implying $\epsilon_k \rightarrow 0$). Taking the limit of \eqref{eq:xi-now} as $\xi\rightarrow 0$ shows that
		$\left(\bar{\b f}_{\app},\bar{\b b}_{\app}\right)$ is a solution of $\Psc\left(\bar{\Omega}^{(\infty)}\right)$. Assuming that $\hatFK=\FK$, this means that $\left(\bar{\b f}_{\app},\bar{\b b}_{\app}\right)$  also solves the original problem \eqref{opt-cons} 
		\begin{align*}
			\left(\bar{\b f}_{\app},\bar{\b b}_{\app}\right) & \in \argmin_{\substack{\b f\,\in\,\FK,\, \b b\,\in\,\R^B,\\\,(\b f, \b b)\,\in \, C
			}} \Lcal(\b f,\b b) \ni \left(\bar{\b f},\bar{\b b}\right)
		\end{align*}
		by applying the same argument used to derive \eqref{eq:f-b-inf}. Consequently if $\left(\bar{\b f},\bar{\b b}\right)$ is unique, then $\left(\bar{\b f}_{\app},\bar{\b b}_{\app}\right)=\left(\bar{\b f},\bar{\b b}\right)$. Hence every weakly converging subsequence of $\left(\b f^{(k)},\b b^{(k)}\right)_{k\in\N}$ converges to $\left(\bar{\b f},\bar{\b b}\right)$, so the whole sequence $\left(\b f^{(k)},\b b^{(k)}\right)_{k\in\N}$ weakly converges to $\left(\bar{\b f},\bar{\b b}\right)$.
	\end{itemize}

	\bibliography{./BIB/aubin22a}

\begin{thebibliography}{68}
\providecommand{\natexlab}[1]{#1}
\providecommand{\url}[1]{\texttt{#1}}
\expandafter\ifx\csname urlstyle\endcsname\relax
  \providecommand{\doi}[1]{doi: #1}\else
  \providecommand{\doi}{doi: \begingroup \urlstyle{rm}\Url}\fi

\bibitem[Agrell(2019)]{agrell19gaussian}
Christian Agrell.
\newblock Gaussian processes with linear operator inequality constraints.
\newblock \emph{Journal of Machine Learning Research}, 20:\penalty0 1--36,
  2019.

\bibitem[A\"it-Sahalia and Duarte(2003)]{aitsahalia03nonparametric}
Yacine A\"it-Sahalia and Jefferson Duarte.
\newblock Nonparametric option pricing under shape restrictions.
\newblock \emph{Journal of Econometrics}, 116\penalty0 (1-2):\penalty0 9--47,
  2003.

\bibitem[Allon et~al.(2007)Allon, Beenstock, Hackman, Passy, and
  Shapiro]{allon07nonparametric}
Gad Allon, Michael Beenstock, Steven Hackman, Ury Passy, and Alexander Shapiro.
\newblock Nonparametric estimation of concave production technologies by
  entropic methods.
\newblock \emph{Journal of Applied Econometrics}, 22\penalty0 (4):\penalty0
  795--816, 2007.

\bibitem[{\'A}lvarez et~al.(2012){\'A}lvarez, Rosasco, and
  Lawrence]{alvarez12kernels}
Mauricio {\'A}lvarez, Lorenzo Rosasco, and Neil Lawrence.
\newblock Kernels for vector-valued functions: a review.
\newblock \emph{Foundations and Trends in Machine Learning}, 4\penalty0
  (3):\penalty0 195--266, 2012.

\bibitem[Aronszajn(1950)]{aronszajn50theory}
Nachman Aronszajn.
\newblock Theory of reproducing kernels.
\newblock \emph{Transactions of the American Mathematical Society},
  68:\penalty0 337--404, 1950.

\bibitem[Attouch et~al.(2014)Attouch, Buttazzo, and
  Michaille]{attouch14variational}
Hedy Attouch, Giuseppe Buttazzo, and G{\'{e}}rard Michaille.
\newblock \emph{Variational Analysis in {S}obolev and {BV} Spaces}.
\newblock Society for Industrial and Applied Mathematics, 2014.

\bibitem[Aubin-Frankowski(2021)]{aubin2020hard_control}
Pierre-Cyril Aubin-Frankowski.
\newblock Linearly constrained linear quadratic regulator from the viewpoint of
  kernel methods.
\newblock \emph{{SIAM} Journal on Control and Optimization}, 59\penalty0
  (4):\penalty0 2693--2716, 2021.

\bibitem[Aubin-Frankowski and Szab{\'o}(2020)]{aubin20hard}
Pierre-Cyril Aubin-Frankowski and Zolt{\'a}n Szab{\'o}.
\newblock Hard shape-constrained kernel machines.
\newblock In \emph{Advances in Neural Information Processing Systems
  (NeurIPS)}, pages 384--395, 2020.

\bibitem[Aubin-Frankowski et~al.(2020)Aubin-Frankowski, Petit, and
  Szab{\'o}]{aubin20kernel}
Pierre-Cyril Aubin-Frankowski, Nicolas Petit, and Zolt{\'a}n Szab{\'o}.
\newblock Kernel regression for vehicle trajectory reconstruction under speed
  and inter-vehicular distance constraints.
\newblock In \emph{IFAC World Congress (IFAC WC)}, pages 15084--15089, 2020.

\bibitem[Berlinet and Thomas-Agnan(2004)]{berlinet04reproducing}
Alain Berlinet and Christine Thomas-Agnan.
\newblock \emph{Reproducing Kernel Hilbert Spaces in Probability and
  Statistics}.
\newblock Kluwer, 2004.

\bibitem[Blundell et~al.(2012)Blundell, Horowitz, and
  Parey]{blundell12measuring}
Richard Blundell, Joel~L. Horowitz, and Matthias Parey.
\newblock Measuring the price responsiveness of gasoline demand: economic shape
  restrictions and nonparametric demand estimation.
\newblock \emph{Quantitative Economics}, 3:\penalty0 29--51, 2012.

\bibitem[Bouche et~al.(2021)Bouche, Clausel, Roueff, and d'Alch{\'e}
  Buc]{bouche21nonlinear}
Dimitri Bouche, Marianne Clausel, Francois Roueff, and Florence d'Alch{\'e}
  Buc.
\newblock Nonlinear functional output regression: a dictionary approach.
\newblock In \emph{International Conference on Artificial Intelligence and
  Statistics (AISTATS)}, pages 235--243, 2021.

\bibitem[Brault et~al.(2019)Brault, Lambert, Szab{\'o}, Sangnier, and
  d'Alch{\'e} Buc]{brault19infinite}
Romain Brault, Alex Lambert, Zolt{\'a}n Szab{\'o}, Maxime Sangnier, and
  Florence d'Alch{\'e} Buc.
\newblock Infinite-task learning with {RKHS}s.
\newblock In \emph{International Conference on Artificial Intelligence and
  Statistics (AISTATS)}, pages 1294--1302, 2019.

\bibitem[Brouard et~al.(2011)Brouard, d'Alch{\'e} Buc, and
  Szafranski]{brouard11semisupervised}
C{\'e}line Brouard, Florence d'Alch{\'e} Buc, and Marie Szafranski.
\newblock Semi-supervised penalized output kernel regression for link
  prediction.
\newblock In \emph{International Conference on Machine Learning (ICML)}, pages
  593--600, 2011.

\bibitem[Brunk(1955)]{brunk55maximum}
Hugh~D. Brunk.
\newblock Maximum likelihood estimates of monotone parameters.
\newblock \emph{Annals of Mathematical Statistics}, 26\penalty0 (4):\penalty0
  607--616, 1955.

\bibitem[Carmeli et~al.(2010)Carmeli, Vito, Toigo, and
  Umanit{\'a}]{carmeli10vector}
Claudio Carmeli, Ernesto~De Vito, Alessandro Toigo, and Veronica Umanit{\'a}.
\newblock Vector valued reproducing kernel {H}ilbert spaces and universality.
\newblock \emph{Analysis and Applications}, 8:\penalty0 19--61, 2010.

\bibitem[Chen and Samworth(2016)]{chen16generalized}
Yining Chen and Richard~J. Samworth.
\newblock Generalized additive and index models with shape constraints.
\newblock \emph{Journal of the Royal Statistical Society -- Statistical
  Methodology, Series {B}}, 78\penalty0 (4):\penalty0 729--754, 2016.

\bibitem[Chetverikov et~al.(2018)Chetverikov, Santos, and
  Shaikh]{chetverikov18econometrics}
Denis Chetverikov, Andres Santos, and Azeem~M. Shaikh.
\newblock The econometrics of shape restrictions.
\newblock \emph{Annual Review of Economics}, 10\penalty0 (1):\penalty0 31--63,
  2018.

\bibitem[Curmei and Hall(2021)]{curmei21shape}
Mihaela Curmei and Georgina Hall.
\newblock Shape-constrained regression using sum of squares polynomials.
\newblock Technical report, 2021.
\newblock (\url{https://arxiv.org/abs/2004.03853}).

\bibitem[Deng and Zhang(2020)]{deng20isotonic}
Hang Deng and Cun-Hui Zhang.
\newblock Isotonic regression in multi-dimensional spaces and graphs.
\newblock \emph{Annals of Statistics}, 48\penalty0 (6):\penalty0 3672--3698,
  2020.

\bibitem[Dubovitskii and Milyutin(1965)]{dubovitskii65extremum}
A.~Ya. Dubovitskii and A.~A. Milyutin.
\newblock Extremum problems in the presence of restrictions.
\newblock \emph{{USSR} Computational Mathematics and Mathematical Physics},
  5\penalty0 (3):\penalty0 1--80, 1965.

\bibitem[Freyberger and Reeves(2018)]{freyberger18inference}
Joachim Freyberger and Brandon Reeves.
\newblock Inference under shape restrictions.
\newblock Technical report, University of Wisconsin-Madison, 2018.
\newblock
  (\url{https://www.ssc.wisc.edu/~jfreyberger/Shape_Inference_Freyberger_Reeves.pdf}).

\bibitem[Guntuboyina and Sen(2018)]{guntuboyina18nonparametric}
Adityanand Guntuboyina and Bodhisattva Sen.
\newblock Nonparametric shape-restricted regression.
\newblock \emph{Statistical Science}, 33\penalty0 (4):\penalty0 568--594, 2018.

\bibitem[Hall(2018)]{hall18thesis}
Georgina Hall.
\newblock Optimization over nonnegative and convex polynomials with and without
  semidefinite programming.
\newblock {PhD Thesis}, Princeton University, 2018.

\bibitem[Han and Wellner(2016)]{han16multivariate}
Qiyang Han and Jon~A. Wellner.
\newblock Multivariate convex regression: global risk bounds and adaptation.
\newblock Technical report, 2016.
\newblock (\url{https://arxiv.org/abs/1601.06844}).

\bibitem[Han et~al.(2019)Han, Wang, Chatterjee, and Samworth]{han19isotonic}
Qiyang Han, Tengyao Wang, Sabyasachi Chatterjee, and Richard~J. Samworth.
\newblock Isotonic regression in general dimensions.
\newblock \emph{Annals of Statistics}, 47\penalty0 (5):\penalty0 2440--2471,
  2019.

\bibitem[Hartl et~al.(1995)Hartl, Sethi, and Vickson]{hartl95survey}
Richard~F. Hartl, Suresh~P. Sethi, and Raymond~G. Vickson.
\newblock A survey of the maximum principles for optimal control problems with
  state constraints.
\newblock \emph{SIAM Review}, 37\penalty0 (2):\penalty0 181--218, 1995.

\bibitem[Hu et~al.(2005)Hu, Kapoor, Zhang, Hamilton, and Coombes]{hu05analysis}
Jianhua Hu, Mini Kapoor, Wei Zhang, Stanley~R. Hamilton, and Kevin~R. Coombes.
\newblock Analysis of dose-response effects on gene expression data with
  comparison of two microarray platforms.
\newblock \emph{Bioinformatics}, 21\penalty0 (17):\penalty0 3524--3529, 2005.

\bibitem[Huusari and Kadri(2021)]{huusari21entangled}
Riikka Huusari and Hachem Kadri.
\newblock Entangled kernels -- beyond separability.
\newblock \emph{Journal of Machine Learning Research}, 22:\penalty0 1--40,
  2021.

\bibitem[Johnson and Jiang(2018)]{johnson18shape}
Andrew~L. Johnson and Daniel~R. Jiang.
\newblock Shape constraints in economics and operations research.
\newblock \emph{Statistical Science}, 33\penalty0 (4):\penalty0 527--546, 2018.

\bibitem[Kadri et~al.(2016)Kadri, Duflos, Preux, Canu, Rakotomamonjy, and
  Audiffren]{kadri16operator}
Hachem Kadri, Emmanuel Duflos, Philippe Preux, St{\'e}phane Canu, Alain
  Rakotomamonjy, and Julien Audiffren.
\newblock Operator-valued kernels for learning from functional response data.
\newblock \emph{Journal of Machine Learning Research}, 17:\penalty0 1--54,
  2016.

\bibitem[Koenker(2005)]{koenker05quantile}
Roger Koenker.
\newblock \emph{Quantile Regression}.
\newblock Cambridge University Press, 2005.

\bibitem[Koppel et~al.(2019)Koppel, Zhang, Zhu, and
  Ba\c{s}ar]{koppel19projected}
Alec Koppel, Kaiqing Zhang, Hao Zhu, and Tamer Ba\c{s}ar.
\newblock Projected stochastic primal-dual method for constrained online
  learning with kernels.
\newblock \emph{IEEE Transactions on Signal Processing}, 67\penalty0
  (10):\penalty0 2528--2542, 2019.

\bibitem[Kur et~al.(2020)Kur, Dagan, and Rakhlin]{kur20optimality}
Gil Kur, Yuval Dagan, and Alexander Rakhlin.
\newblock Optimality of maximum likelihood for log-concave density estimation
  and bounded convex regression.
\newblock Technical report, 2020.
\newblock (\url{https://arxiv.org/abs/1903.05315}).

\bibitem[Lim(2020)]{lim20limiting}
Eunji Lim.
\newblock The limiting behavior of isotonic and convex regression estimators
  when the model is misspecified.
\newblock \emph{Electronic Journal of Statistics}, 14:\penalty0 2053--2097,
  2020.

\bibitem[Lofberg(2004)]{lofberg04yalmip}
Johan Lofberg.
\newblock {YALMIP}: A toolbox for modeling and optimization in {MATLAB}.
\newblock In \emph{IEEE International Conference on Robotics and Automation},
  pages 284--289, 2004.

\bibitem[Luss et~al.(2012)Luss, Rossett, and Shahar]{luss12efficient}
Ronny Luss, Saharon Rossett, and Moni Shahar.
\newblock Efficient regularized isotonic regression with application to
  gene-gene interaction search.
\newblock \emph{Annals of Applied Statistics}, 6\penalty0 (1):\penalty0
  253--283, 2012.

\bibitem[Marteau-Ferey et~al.(2020)Marteau-Ferey, Bach, and
  Rudi]{marteauferey20nonparametric}
Ulysse Marteau-Ferey, Francis Bach, and Alessandro Rudi.
\newblock Non-parametric models for non-negative functions.
\newblock In \emph{Advances in Neural Information Processing Systems
  (NeurIPS)}, pages 12816--12826, 2020.

\bibitem[Mazumder et~al.(2019)Mazumder, Choudhury, Iyengar, and
  Sen]{mazumder19computational}
Rahul Mazumder, Arkopal Choudhury, Garud Iyengar, and Bodhisattva Sen.
\newblock A computational framework for multivariate convex regression and its
  variants.
\newblock \emph{Journal of the American Statistical Association}, 114\penalty0
  (525):\penalty0 318--331, 2019.

\bibitem[Meyer(2018)]{meyer18framework}
Mary~C. Meyer.
\newblock A framework for estimation and inference in generalized additive
  models with shape and order restrictions.
\newblock \emph{Statistical Science}, 33\penalty0 (4):\penalty0 595--614, 2018.

\bibitem[Micchelli and Pontil(2005)]{micchelli05learning}
Charles Micchelli and Massimiliano Pontil.
\newblock On learning vector-valued functions.
\newblock \emph{Neural Computation}, 17:\penalty0 177--204, 2005.

\bibitem[Micchelli et~al.(2006)Micchelli, Xu, and Zhang]{micchelli06universal}
Charles Micchelli, Yuesheng Xu, and Haizhang Zhang.
\newblock Universal kernels.
\newblock \emph{Journal of Machine Learning Research}, 7:\penalty0 2651--2667,
  2006.

\bibitem[Micheli and Glaun{\'e}s(2014)]{micheli14matrix}
Mario Micheli and Joan~A. Glaun{\'e}s.
\newblock Matrix-valued kernels for shape deformation analysis.
\newblock \emph{Geometry, Imaging and Computing}, 1\penalty0 (1):\penalty0
  57--139, 2014.

\bibitem[Muzellec et~al.(2022)Muzellec, Bach, and Rudi]{muzellec22learning}
Boris Muzellec, Francis Bach, and Alessandro Rudi.
\newblock Learning {PSD}-valued functions using kernel sums-of-squares.
\newblock Technical report, 2022.
\newblock (\url{https://arxiv.org/abs/2111.11306}).

\bibitem[Papp and Alizadeh(2014)]{papp14shape}
D{\'a}vid Papp and Farid Alizadeh.
\newblock Shape-constrained estimation using nonnegative splines.
\newblock \emph{Journal of Computational and Graphical Statistics}, 23\penalty0
  (1):\penalty0 211--231, 2014.

\bibitem[Peypouquet(2015)]{peypouquet15convex}
Juan Peypouquet.
\newblock \emph{Convex optimization in normed spaces}.
\newblock Springer Cham, 2015.

\bibitem[Price(1940)]{price40completeness}
G.~Baley Price.
\newblock On the completeness of a certain metric space with an application to
  {B}laschke's selection theorem.
\newblock \emph{Bulletin of the American Mathematical Society}, 46\penalty0
  (4):\penalty0 278--280, 1940.

\bibitem[Pya and Wood(2015)]{pya15shape}
Natalya Pya and Simon~N. Wood.
\newblock Shape constrained additive models.
\newblock \emph{Statistics and Computing}, 25:\penalty0 543--559, 2015.

\bibitem[Royset and Wets(2015)]{royset15fusion}
Johannes~O. Royset and Roger J-B Wets.
\newblock Fusion of hard and soft information in nonparametric density
  estimation.
\newblock \emph{European Journal of Operational Research}, 247\penalty0
  (2):\penalty0 532--547, 2015.

\bibitem[Rudi et~al.(2020)Rudi, Marteau-Ferey, and Bach]{rudi20finding}
Alessandro Rudi, Ulysse Marteau-Ferey, and Francis Bach.
\newblock Finding global minima via kernel approximations.
\newblock Technical report, 2020.
\newblock (\url{https://arxiv.org/abs/2012.11978}).

\bibitem[Saitoh and Sawano(2016)]{saitoh16theory}
Saburou Saitoh and Yoshihiro Sawano.
\newblock \emph{Theory of Reproducing Kernels and Applications}.
\newblock Springer Singapore, 2016.

\bibitem[Sangnier et~al.(2016)Sangnier, Fercoq, and d'Alch{\'e}
  Buc]{sangnier16joint}
Maxime Sangnier, Olivier Fercoq, and Florence d'Alch{\'e} Buc.
\newblock Joint quantile regression in vector-valued {RKHS}s.
\newblock \emph{Advances in Neural Information Processing Systems (NIPS)},
  pages 3693--3701, 2016.

\bibitem[Sch{\"o}lkopf et~al.(2001)Sch{\"o}lkopf, Herbrich, and
  Smola]{scholkopf01generalized}
Bernhard Sch{\"o}lkopf, Ralf Herbrich, and Alex~J. Smola.
\newblock A generalized representer theorem.
\newblock In \emph{Conference on Learning Theory (COLT)}, pages 416--426, 2001.

\bibitem[Simchi-Levi et~al.(2014)Simchi-Levi, Chen, and
  Bramel]{simchilevi14logic}
David Simchi-Levi, Xin Chen, and Julien Bramel.
\newblock \emph{The Logic of Logistics: Theory, Algorithms, and Applications
  for Logistics Management}.
\newblock Springer, 2014.

\bibitem[Simon-Gabriel and Sch{\"o}lkopf(2018)]{simon-gabriel18kernel}
Carl-Johann Simon-Gabriel and Bernhard Sch{\"o}lkopf.
\newblock Kernel distribution embeddings: Universal kernels, characteristic
  kernels and kernel metrics on distributions.
\newblock \emph{Journal of Machine Learning Research}, 19\penalty0
  (44):\penalty0 1--29, 2018.

\bibitem[Sriperumbudur et~al.(2011)Sriperumbudur, Fukumizu, and
  Lanckriet]{sriperumbudur11universality}
Bharath Sriperumbudur, Kenji Fukumizu, and Gert Lanckriet.
\newblock Universality, characteristic kernels and {RKHS} embedding of
  measures.
\newblock \emph{Journal of Machine Learning Research}, 12:\penalty0 2389--2410,
  2011.

\bibitem[Steinwart(2001)]{steinwart01influence}
Ingo Steinwart.
\newblock On the influence of the kernel on the consistency of support vector
  machines.
\newblock \emph{Journal of Machine Learning Research}, 6\penalty0 (3):\penalty0
  67--93, 2001.

\bibitem[Steinwart and Christmann(2008)]{steinwart08support}
Ingo Steinwart and Andreas Christmann.
\newblock \emph{Support Vector Machines}.
\newblock Springer, 2008.

\bibitem[Takeuchi et~al.(2006)Takeuchi, Le, Sears, and
  Smola]{takeuchi06nonparametric}
Ichiro Takeuchi, Quoc Le, Timothy Sears, and Alexander Smola.
\newblock Nonparametric quantile estimation.
\newblock \emph{Journal of Machine Learning Research}, 7:\penalty0 1231--1264,
  2006.

\bibitem[Topkis(1998)]{topkis98supermodularity}
Donald~M. Topkis.
\newblock \emph{Supermodularity and complementarity}.
\newblock Princeton University Press, 1998.

\bibitem[Turlach(2005)]{turlach05shape}
Berwin~A. Turlach.
\newblock Shape constrained smoothing using smoothing splines.
\newblock \emph{Computational Statistics}, 20:\penalty0 81--104, 2005.

\bibitem[Varian(1984)]{varian84nonparametric}
Hal~R. Varian.
\newblock The nonparametric approach to production analysis.
\newblock \emph{Econometrica}, 52\penalty0 (3):\penalty0 579--597, 1984.

\bibitem[Vito et~al.(2013)Vito, Umanit{\'a}, and Villa]{devito13extension}
Ernesto~De Vito, Veronica Umanit{\'a}, and Silvia Villa.
\newblock An extension of {M}ercer theorem to matrix-valued measurable kernels.
\newblock \emph{Applied and Computational Harmonic Analysis}, 34\penalty0
  (3):\penalty0 339--351, 2013.

\bibitem[Wahba(1990)]{wahba90spline}
Grace Wahba.
\newblock \emph{Spline Models for Observational Data}.
\newblock SIAM, CBMS-NSF Regional Conference Series in Applied Mathematics,
  1990.

\bibitem[Wang(2011)]{wang11splines}
Yuedong Wang.
\newblock \emph{Smoothing Splines -- Methods and Applications}.
\newblock CRC Press, 2011.

\bibitem[Wu and Sickles(2018)]{wu18semiparametric}
Ximing Wu and Robin Sickles.
\newblock Semiparametric estimation under shape constraints.
\newblock \emph{Econometrics and Statistics}, 6:\penalty0 74--89, 2018.

\bibitem[Yang et~al.(2019)Yang, Wang, Kiyavash, and He]{yang19learning}
Yingxiang Yang, Haoxiang Wang, Negar Kiyavash, and Niao He.
\newblock Learning positive functions with pseudo mirror descent.
\newblock In \emph{Advances in Neural Information Processing Systems
  (NeurIPS)}, pages 14144--14154, 2019.

\bibitem[Zhou(2008)]{zhou08derivative}
Ding-Xuan Zhou.
\newblock Derivative reproducing properties for kernel methods in learning
  theory.
\newblock \emph{Journal of Computational and Applied Mathematics},
  220:\penalty0 456--463, 2008.

\end{thebibliography}

\end{document}